\documentclass{article}

\PassOptionsToPackage{numbers,sort&compress}{natbib}
 \usepackage[main, final]{neurips_2026}

\usepackage[T1]{fontenc}    
\usepackage{url}            
\usepackage{booktabs}       
\usepackage{amsfonts}       
\usepackage{nicefrac}       
\usepackage{microtype}      
\usepackage{xcolor}         
\usepackage{amsthm}
\usepackage{amssymb}
\usepackage{graphicx}
\newtheorem{theorem}{Theorem}
\newtheorem{corollary}{Corollary}
\newtheorem{proposition}{Proposition}

\usepackage{amsmath}
\usepackage{algorithm}
\usepackage{algpseudocode}
\usepackage{float}

\newtheorem{assumption}{Assumption}
\usepackage{multirow}
\usepackage{enumitem}
\usepackage{wrapfig}
\usepackage{array}
\usepackage[table]{xcolor}
\usepackage{colortbl}
\usepackage{soul}
\usepackage{booktabs}
\newcommand{\best}[1]{\textcolor{red}{#1}}
\newcommand{\second}[1]{\underline{#1}}
\definecolor{metablue}{RGB}{232,244,252}
\usepackage{calc}
\usepackage{tcolorbox}
\tcbuselibrary{breakable}
\definecolor{myblue}{RGB}{0,112,192}
\usepackage{tabularx}
\usepackage{ragged2e}
\usepackage{hyperref}       
\hypersetup{linkcolor=black}

\title{MetaKE: Meta-Learning for Knowledge Editing Toward a Better Accuracy-Editability Trade-off}

\author{
Shuxin Liu \quad Di Gao\quad Ou Wu\thanks{Corresponding author.} \\
\texttt{liushuxin25@mails.ucas.ac.cn} \quad
\texttt{\{gaodi,wuou\}@ucas.ac.cn} \\
Hangzhou Institute for Advanced Study\\ University of Chinese Academy of Sciences\\
  Hangzhou, China
}

\begin{document}
\maketitle
\begin{abstract}
Existing locate-then-edit Knowledge Editing (KE) methods typically decompose editing into two stages: upstream target representation optimization and downstream constrained parameter optimization. The optimization across the two stages is disconnected: upstream applies uniform regularization without observing downstream realization of the planned residual, hindering a refined accuracy–editability trade-off. Since this realization is request-specific and depends on downstream constraints, uniform regularization can over-shrink high-association requests, causing
insufficient editing, while it can under-regularize low-association requests,
producing over-large planned residuals that reduce downstream editability. To bridge this disconnect, we propose \textbf{MetaKE} (\textbf{Meta}-learning for \textbf{K}nowledge \textbf{E}diting), a new framework that unifies upstream and downstream stages into a bi-level optimization problem. The inner level optimizes parameter updates for the target representation, while the outer level optimizes representation using feedback from downstream constraints, achieving a better semantic accuracy-editability trade-off.  To avoid costly multi-layer backpropagation, we introduce a Structural Gradient Proxy to approximate and propagate this feedback. Extensive experiments show that MetaKE outperforms strong baselines, offering a new perspective on KE. Code at: https://github.com/zhutizi23-coder/MetaKE.
\end{abstract}

\section{Introduction}
Large Language Models (LLMs) have become a core substrate for NLP systems \citep{yao2023editing,brown2020language,rae2021scaling}, yet their parametric memory inevitably contains factual errors and outdated knowledge \citep{zhang2024comprehensive,ji2023survey,ferrara2023should,mitchell2022memory,liska2022streamingqa,anderson1972simple,lazaridou2021mind}. KE aims to update specific knowledge in LLMs while preserving general capabilities \citep{de2021editing,hartvigsen2023aging}. Recent studies suggest that the \emph{editing} process in locate-then-edit methods (e.g., MEMIT \citep{mengmass}, AlphaEdit \citep{fangalphaedit}) consists of two optimization stages \citep{gupta2024unified,gupta-etal-2025-lifelong,baghel2025resolving,wang2025revealing,zhao2025fedleke}: an upstream stage that optimizes a target representation and induces a planned residual, and a downstream stage that realizes this residual through constrained parameter optimization. In the upstream stage, methods such as
PMET \citep{li2024pmet} and STEAM \citep{jeong-etal-2025-steam} apply uniform KL
regularization across different target knowledge to improve semantic accuracy
and reduce interference with preserved knowledge. In the downstream stage, MEMIT uses covariance constraints, while AlphaEdit introduces projection constraints to protect preserved knowledge.

Despite these advances, the optimization across the two stages remains disconnected,
which hinders a refined accuracy--editability trade-off. The upstream controls
the planned residual through the target representation, whereas the downstream stage realizes only a request-specific fraction of it,
governed by the request's association with the downstream constraints, i.e.,
how strongly the request interacts with those constraints. The shared upstream regularizer,
however, is association-blind: it scales the planned residual without observing
its request-specific realization. Since the realized update depends jointly on
this residual and the request's association with the downstream constraints, the
same regularization strength may fail to suit different requests:
high-association requests may be over-shrunk, causing insufficient editing,
while low-association requests may yield over-large planned residuals that are
difficult to realize under downstream constraints, thereby reducing downstream
editability. A request-specific regularization strength can rescale the planned
residual, but such a scalar adjustment changes the residual magnitude without
revealing to the upstream stage how much of that residual will be realized after
downstream optimization. Intuitively, if the upstream stage could receive
feedback on downstream constraints while optimizing the target representation,
it could achieve a better semantic accuracy--editability trade-off.

To bridge this disconnect, we propose \textbf{MetaKE} (\textbf{Meta}-learning
for \textbf{K}nowledge \textbf{E}diting), a new framework that unifies upstream
and downstream stages into a bi-level optimization problem. The inner level
optimizes the parameter update for a given target representation, while the
outer level optimizes the target representation using feedback from the
downstream constraints; the target representation sits at the outer level
because it specifies what the parameter update should achieve, while the update
is its constrained realization. To avoid expensive multi-layer backpropagation,
we introduce a \emph{Structural Gradient Proxy}, which efficiently approximates
and propagates feedback from downstream constraints without full differentiation
through the multi-layer process. Based on this mechanism, MetaKE first refines
the target representation in the bi-level phase using downstream-aware feedback,
and then performs the multi-layer parameter update based on the refined target.

Our contributions are summarized as follows:
\begin{itemize}[leftmargin=*,nosep]
    \item We analyze the disconnect between the two optimization stages and show
    that, under a local quadratic surrogate, downstream realization is
    request-specific while the shared upstream regularizer is association-blind, thereby hindering a refined accuracy--editability trade-off.
    \item We propose \textbf{MetaKE}, a new framework that unifies upstream and
    downstream stages into a bi-level optimization problem, enabling the upstream
    to perceive the downstream constraints and achieve a better semantic
    accuracy--editability trade-off.
    \item We introduce a \emph{Structural Gradient Proxy} that provides a
    tractable approximation of downstream constraints feedback, avoiding costly
    full differentiation through multi-layer updates.
    \item Experiments on GPT-2 XL (1.5B), GPT-J (6B), and LLaMA3 (8B) show that
    MetaKE achieves SOTA performance and consistently improves the
    accuracy-editability trade-off.

\end{itemize}

\section{Related Work}
\label{sec:related_work}
\textbf{Knowledge Editing with Constraints.}
The locate-then-edit paradigm is central for KE, where factual associations are
localized and modified through targeted updates
\citep{meng2022locating,dai2022knowledge,mengmass}. Knowledge Neurons
attributes factual knowledge to feed-forward neurons \citep{dai2022knowledge},
while ROME edits factual recall via rank-one updates
\citep{meng2022locating}. MEMIT extends this to multi-layer editing
with covariance-based preservation constraints \citep{mengmass}. EMMET unifies
ROME and MEMIT under a preservation--memorization objective and derives a
closed-form constrained update \citep{gupta2024unified}. AlphaEdit projects
perturbations onto the null space of preserved knowledge to reduce disruption
to preserved key-value associations \citep{fangalphaedit}, while LyapLock
studies preservation under sequential editing \citep{wang-etal-2025-lyaplock}. These
methods mainly strengthen downstream updates, constraints, or edit trajectories after
target optimization. In contrast, MetaKE studies how downstream constraints
provide feedback to upstream target optimization.

\textbf{Meta-Learning and Target-Centric Optimization for Knowledge Editing.}
Another line of work improves KE by learning update mappings or optimizing
targets. KE and MEND train hypernetworks to transform gradients into parameter
updates \citep{de2021editing,mitchellfast}, while MALMEN scales this paradigm
to batched edits via least-squares approximations \citep{tan2024massive}.
InstructEdit improves editor generalization with task-specific instructions,
and RLEdit formulates lifelong hypernetwork editing as reinforcement learning
\citep{zhang2024instructedit,li2025reinforced}. Target-centric methods such as
AdaEdit, AnyEdit, REP, and DAFNet refine editing targets, robust keys, or
sequence-level representations
\citep{li2025adaedit,jianganyedit,yan2025keys,zhang2024dafnet}. These methods
show the importance of learned update mappings and target optimization, but
their signals are mainly semantic, gradient-based, instructional, or
robustness-oriented. MetaKE instead couples upstream target refinement with
downstream constrained realization, refining targets using downstream feedback.

\section{Preliminaries}
\label{sec:background}

\subsection{Knowledge Editing as Linear Associative Memory}
KE typically treats Transformer FFNs as linear associative memories
\citep{geva2021transformer,kohonen1972correlation,anderson1972simple}. For an
edited layer \(l\), the FFN output matrix
\(\boldsymbol{W}^{l}_{\mathrm{out}}\in\mathbb{R}^{d_1\times d_0}\) maps a key
\(\boldsymbol{k}_l\in\mathbb{R}^{d_0}\), which encodes the subject-relation
pair \((s,r)\), to a value \(\boldsymbol{m}_l\in\mathbb{R}^{d_1}\), which
contributes to predicting the object \(o\):
\begin{equation}
\boldsymbol{m}_l \approx \boldsymbol{W}^{l}_{\mathrm{out}}\boldsymbol{k}_l,
\qquad
\boldsymbol{k}_l
=
\sigma\!(\boldsymbol{W}^{l}_{\mathrm{in}}\gamma(\cdot)),
\label{eq:bg:assoc}
\end{equation}
where \(\boldsymbol{W}^{l}_{\mathrm{in}}\) is the FFN input matrix,
\(\gamma(\cdot)\) denotes layer normalization, and \(\sigma(\cdot)\) is the
FFN activation function. In this paper, we use
\(\boldsymbol{W}_l:=\boldsymbol{W}^{l}_{\mathrm{out}}\). Editing seeks a
parameter update such that the post-edit memory output aligns with the
desired object \(o^*\).

\subsection{The Editing Process in Locate-then-Edit as Two Optimization Stages}
\label{sec:bg:l2e}

In this paper, \textbf{semantic accuracy} measures how well a target
representation \(\boldsymbol{v}^*\) encodes the desired knowledge. \textbf{Editability} denotes how well
\(\boldsymbol{v}^*\) can be translated into parameter updates under
downstream constraints. Recent analyses suggest that the editing process in locate-then-edit methods
can be viewed as two optimization stages
\citep{gupta2024unified,gupta-etal-2025-lifelong,baghel2025resolving,wang2025revealing,zhao2025fedleke}.

\textbf{Stage I: Upstream Target Representation Optimization.} MEMIT-style variants optimize \(\boldsymbol{v}^*\) for
the desired knowledge through a likelihood objective
with KL regularization
\citep{li2024pmet,park-etal-2025-context,wei2025setke,fei2026scaling}:
\begin{equation}
\small
\boldsymbol{v}^*
=
\boldsymbol{m}_L+
\operatorname{arg\,min}_{\boldsymbol{\delta}}
(
-\log
\mathbb{P}_{f_{L_{\mathrm{out}}}(\boldsymbol{m}_L+\boldsymbol{\delta})}
[o^*\mid s,r]
+
\omega
D_{\mathrm{KL}}
(
\mathbb{P}_{f}[\cdot\mid p']
\Vert
\mathbb{P}_{f_{L_{\mathrm{out}}}(\boldsymbol{m}_L+\boldsymbol{\delta})}
[\cdot\mid p']
)
),
\label{eq:bg:vstar}
\end{equation}
where \(\boldsymbol{m}_L\) is the original representation
at the last edited layer \(L\), \(f_{L_{\mathrm{out}}}(\cdot)\) denotes the remaining forward, and \(p'\) is the prompt template. The likelihood
term promotes semantic accuracy, while the KL term limits distributional shift.
The upstream stage optimizes the semantic target under a uniform regularization
strength without accounting for how editable it will be under downstream constraints.

\textbf{Stage II: Downstream Constrained Parameter Optimization.}
Given the target representation \(\boldsymbol{v}^*\), the downstream stage
optimizes a constrained parameter update. Let \(\boldsymbol{k}_1\) be the key of the to-be-updated
knowledge. The prescribed residual is
\(\boldsymbol{r}=\boldsymbol{v}^*-\boldsymbol{W}\boldsymbol{k}_1\). For AlphaEdit, the preserved keys \(\boldsymbol{K}_0\) are protected by a
null-space projection matrix \(\boldsymbol{P}\) satisfying
\(\boldsymbol{P}\boldsymbol{K}_0=\boldsymbol{0}\), while
\(\boldsymbol{K}_p\) denotes the keys of previously edited knowledge. The
objective can be written as
\begin{equation}
\operatorname{arg\,min}_{\widetilde{\boldsymbol{\Delta}}}
\left(
\|
(\boldsymbol{W}+\widetilde{\boldsymbol{\Delta}}\boldsymbol{P})
\boldsymbol{k}_1
-
\boldsymbol{v}^*
\|_2^2
+
\|
\widetilde{\boldsymbol{\Delta}}\boldsymbol{P}
\|_F^2
+
\|
\widetilde{\boldsymbol{\Delta}}\boldsymbol{P}\boldsymbol{K}_p
\|_F^2
\right).
\label{eq:bg:alpha_downstream}
\end{equation}
The applied update is
\begin{equation}
\boldsymbol{\Delta}^{*}_{\mathrm{AlphaEdit}}
=
\boldsymbol{r}\boldsymbol{k}_1^{T}\boldsymbol{P}
\left(
\boldsymbol{K}_p\boldsymbol{K}_p^{T}\boldsymbol{P}
+
\boldsymbol{k}_1\boldsymbol{k}_1^{T}\boldsymbol{P}
+
\boldsymbol{I}
\right)^{-1}.
\label{eq:bg:alpha_solution}
\end{equation}

\begin{figure}[t]
\centering
\includegraphics[width=1\linewidth]{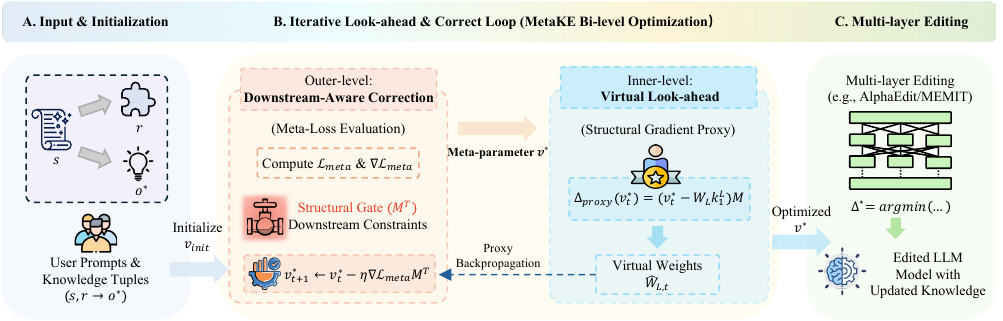}
\caption{The architecture of the proposed method MetaKE.}
\label{fig:metake_architecture}
\end{figure}

\section{Methodology}
\label{sec:method}
We first analyze the disconnect between the two optimization stages, and then describe our proposed method MetaKE. Fig.~\ref{fig:metake_architecture} illustrates the overall architecture.

\subsection{The Disconnect Between the Two Optimization Stages}
\label{sec:analysis}

Stage~I optimizes \(\boldsymbol{v}^{*}\) under a uniform regularization
strength (e.g., the KL weight \(\omega\)), while Stage~II realizes this
target through a constrained update induced by the downstream
constraints. For the
current edit, the key \(\boldsymbol{k}_1\) is determined by the input
subject--relation pair through the forward pass and is not an upstream
variable; the quantity Stage~I effectively controls is the planned
residual
\(\boldsymbol{r}=\boldsymbol{v}^{*}-\boldsymbol{W}\boldsymbol{k}_1\).
However, what Stage~II writes into the model is not \(\boldsymbol{r}\)
itself.

\begin{proposition}[Realized-Residual Attenuation]
\label{prop:target_attenuation}
Consider a current edit with key \(\boldsymbol{k}_1\) and planned residual
\(\boldsymbol{r}\). Under the single-request local quadratic surrogate of
Stage~II with fixed keys and constraints, the realized residual satisfies
\(\boldsymbol{\delta}^{\mathrm{real}}=\beta\boldsymbol{r}\) with
\(\beta=\rho/(1+\rho)\in[0,1)\). We refer to \(\rho\) as the
request-specific \emph{association}, which quantifies how strongly
\(\boldsymbol{k}_1\) interacts with the downstream constraints. For MEMIT
covariance constraints,
\(\rho=\boldsymbol{k}_1^{T}(\boldsymbol{C}_0+\boldsymbol{K}_p
\boldsymbol{K}_p^{T})^{-1}\boldsymbol{k}_1\), where \(\boldsymbol{C}_0= \boldsymbol{K}_0\boldsymbol{K}_0^T\); for AlphaEdit projected
updates,
\(\rho=(\boldsymbol{P}\boldsymbol{k}_1)^{T}(\boldsymbol{I}+
\boldsymbol{P}\boldsymbol{K}_p\boldsymbol{K}_p^{T}\boldsymbol{P})^{-1}
\boldsymbol{P}\boldsymbol{k}_1\). For batch editing, the scalar attenuation
is replaced by a request-interaction matrix. Proof in
Appendix~\ref{app:spectral_proofs}.
\end{proposition}

Proposition~\ref{prop:target_attenuation} reveals that
Stage~II writes \(\beta\boldsymbol{r}\), not \(\boldsymbol{r}\) itself, with
\(\beta\) request-specific because \(\rho\) depends on the request key
and downstream constraints. To assess this request-specificity, we measure
\(\rho_i\) for each edit request. As shown in Fig.~\ref{rho}, the ratio
between the 95th and 5th percentiles of \(\rho_i\) reaches 48.9\(\times\) on
ZsRE and remains around 4--5\(\times\) on CounterFact. Since
\(\beta_i=\rho_i/(1+\rho_i)\), this dispersion induces request-specific
realized and unrealized components, \(\beta_i\boldsymbol{r}_i\) and
\((1-\beta_i)\boldsymbol{r}_i\). In high-\(\rho_i\) regimes, the variation
manifests more clearly in \(1-\beta_i=1/(1+\rho_i)\), which governs the
editability-side upper bound \(\tau/(1-\beta_i)\) in
Theorem~\ref{thm:regularization_disconnect}. Hence upstream choice of \(\boldsymbol{r}_i\) should account for both
realized and unrealized downstream components. The shared KL weight
\(\omega\), however, controls the planned residual scale: locally,
\(\|\boldsymbol{r}_i(\omega)\|\approx a_i/\omega\), where \(a_i\) depends on
the local upstream objective. The effective edit signal is
\(\beta_i\|\boldsymbol{r}_i(\omega)\|\), whereas the unrealized part is
\((1-\beta_i)\|\boldsymbol{r}_i(\omega)\|\). Thus, \(\omega\) is
association-blind: it scales the planned residual without observing this
downstream decomposition.

\begin{theorem}[Association-Dependent Editing Requirement]
\label{thm:regularization_disconnect}
Under Proposition~\ref{prop:target_attenuation}, suppose request \(i\)
requires \(\|\boldsymbol{\delta}^{\mathrm{real}}_i\|\ge m\) for semantic
success and
\(\|\boldsymbol{r}_i-\boldsymbol{\delta}^{\mathrm{real}}_i\|\le\tau\) for
downstream editability. Then any feasible planned residual satisfies
\(\frac{m}{\beta_i}\le\|\boldsymbol{r}_i\|\le\frac{\tau}{1-\beta_i}\), and
the feasibility interval depends on the downstream association through
\(\beta_i=\rho_i/(1+\rho_i)\). In particular, if
\(\beta_i<m/(m+\tau)\), no choice of \(\|\boldsymbol{r}_i\|\) meets both
requirements within this local model. Proof in
Appendix~\ref{app:static_trap_proof}.
\end{theorem}

Theorem~\ref{thm:regularization_disconnect} shows that the admissible scale of
\(\boldsymbol{r}_i\) is governed by the residual decomposition induced by
\(\rho_i\): semantic success imposes the lower bound \(m/\beta_i\), whereas
editability imposes the upper bound \(\tau/(1-\beta_i)\). A shared
\(\omega\), however, imposes a shrinkage rule that does not observe this
downstream-dependent decomposition. For high-association requests,
\(\boldsymbol{r}_i\) is largely realized, so excessive shrinkage may suppress
otherwise feasible edits. For lower-association requests, a larger unrealized
component can make an over-large planned residual difficult to realize under
downstream constraints. These regimes impose different requirements on the
residual scale, so a shared \(\omega\) can be miscalibrated across requests. Appendix~\ref{app:bilevel_proofs} further shows
that, in the bi-level formulation detailed in Sec.~\ref{sec:method:setup}, upstream target optimization should account
for downstream realization rather than being analyzed as an isolated optimization objective.

\begin{figure}[t]
	\centering
	    \includegraphics[width=\linewidth]{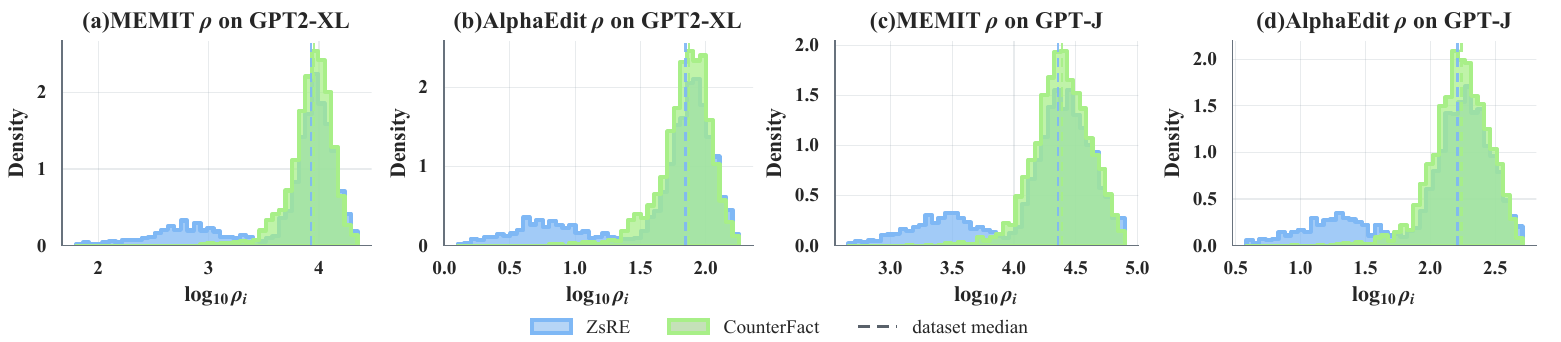}
          
	\caption{Distribution of request-specific downstream association
\(\rho_i\) over \(500\) sampled edit requests. Dashed lines mark medians.}
	\label{rho}
\end{figure}

\subsection{Bi-level Optimization for Downstream-Aware Editing}
\label{sec:method:setup}

Sec.~\ref{sec:analysis} concludes that a better accuracy--editability
trade-off requires \(\boldsymbol{v}^{*}\) to be optimized with access to the
downstream constraints,
so that \(\boldsymbol{r}_i\) can anticipate, per request, how much of
itself will survive the downstream constraints. We grant this access by
recasting KE as a \textbf{Bi-level Optimization (BLO)} problem with
\(\boldsymbol{v}^{*}\) as a learnable \emph{meta-parameter}. Let
\(\mathcal{W}:=\{\boldsymbol{W}_l\}_{l\in\mathcal{L}}\) denote the
key-layer weights and
\(\boldsymbol{\Delta}\mathcal{W}^{*}(\boldsymbol{v}^{*}):=
\{\boldsymbol{\Delta}_l^{*}(\boldsymbol{v}^{*})\}_{l\in\mathcal{L}}\) the
multi-layer downstream updates. The BLO objective is
\begin{equation}
\label{eq:blo_overall}
\begin{aligned}
\boldsymbol{v}^{*}_{\mathrm{meta}}
=
\operatorname{arg\,min}_{\boldsymbol{v}^*}
\quad
&
\mathcal{L}_{\mathrm{meta}}
\big(
f(\cdot;\mathcal{W}+\boldsymbol{\Delta}\mathcal{W}^{*}(\boldsymbol{v}^*))
\big)
\\
\mathrm{s.t.}
\quad
&
\boldsymbol{\Delta}\mathcal{W}^{*}(\boldsymbol{v}^*)
=
\operatorname{arg\,min}_{\widetilde{\boldsymbol{\Delta}}\mathcal{W}}
\;
\mathcal{J}_{\mathrm{inner}}
(
\widetilde{\boldsymbol{\Delta}}\mathcal{W};
\boldsymbol{v}^*
).
\end{aligned}
\end{equation}
where
\(\widehat{\mathcal{W}}=\mathcal{W}+\boldsymbol{\Delta}\mathcal{W}^{*}
(\boldsymbol{v}^{*})\) denotes the virtual post-edit weights. The nesting
embeds the downstream stage as a constraint of the upstream optimization,
so that \(\boldsymbol{v}^{*}\) is both semantically accurate and editable
under the downstream constraints.

\paragraph{Inner Level.}
Given \(\boldsymbol{v}^{*}\), the inner level reproduces the standard
downstream constrained update by allocating a layer-specific target
\(\boldsymbol{t}_l(\boldsymbol{v}^{*})\) detailed in
Appendix~\ref{app:layer_target_allocation} and minimizing
\begin{equation}
\label{eq:lower_level}
\mathcal{J}_{\mathrm{inner}}
=\sum\nolimits_{l\in\mathcal{L}}\!\Big[
\|(\boldsymbol{W}_l+\widetilde{\boldsymbol{\Delta}}_l\boldsymbol{P}_l)
\boldsymbol{k}_1^l-\boldsymbol{t}_l(\boldsymbol{v}^{*})\|_2^2
+\|\widetilde{\boldsymbol{\Delta}}_l\boldsymbol{P}_l\|_F^2
+\|\widetilde{\boldsymbol{\Delta}}_l\boldsymbol{P}_l\boldsymbol{K}_p^l\|_F^2
\!\Big],
\end{equation}
where \(\boldsymbol{P}_l\) is the projection induced by
\(\boldsymbol{K}_0^l\) and \(\boldsymbol{K}_p^l\) collects previously
edited keys. The minimizer
\(\boldsymbol{\Delta}_l^{*}(\boldsymbol{v}^{*})\) is an \emph{implicit function} of
\(\boldsymbol{v}^{*}\); the inner level introduces no new objective beyond
the existing downstream stage but exposes its realization to the upstream.

\paragraph{Outer level.}
The outer level refines \(\boldsymbol{v}^{*}\) on the virtual post-edit
model:
\begin{equation}
\label{eq:meta_loss}
\mathcal{L}_{\mathrm{meta}}(\boldsymbol{v}^{*})
=
\mathcal{L}_{\mathrm{edit}}(x_e,y_e;\widehat{\mathcal{W}})
+
\mathcal{L}_{\mathrm{loc}}
(x_{\mathrm{loc}};\widehat{\mathcal{W}},\mathcal{W})
+
\|\boldsymbol{v}^{*}-\boldsymbol{v}_{\mathrm{init}}\|_2^2 ,
\end{equation}
The edit loss encourages the post-edit model to predict the target output,
the locality loss penalizes unintended drift on unrelated prompts, and the
meta-regularization keeps the refined target close to the initial semantic plan
\(\boldsymbol{v}_{\mathrm{init}}\). All three terms are used with unit
coefficients in our experiments.

\subsection{Structural Gradient Proxy}
\label{sec:method:proxy}

Direct differentiation through the implicit map
\(\boldsymbol{\Delta}\mathcal{W}^{*}(\boldsymbol{v}^{*})\) is
computationally prohibitive, since each outer step requires either nested
optimization or a layer-coupled linear solve. To make Eq.~\eqref{eq:blo_overall}
tractable, we introduce a \textbf{Structural Gradient Proxy} that
approximates the downstream-aware feedback through a single closed-form
surrogate. The proxy rests on a local \emph{Structural Consistency
Hypothesis}: across the edited layers \(\mathcal{L}\), the projection
\(\boldsymbol{P}_l\) and the previous-edit Gram
\(\boldsymbol{K}_p^l(\boldsymbol{K}_p^l)^{T}\) are aligned with their
final-layer counterparts up to a bounded drift. Under this hypothesis, the
dominant downstream signal is captured by the closed form at the
final edited layer \(L\):
\begin{equation}
\label{eq:proxy_definition}
\boldsymbol{\Delta}_{\mathrm{proxy}}(\boldsymbol{v}^{*})
=(\boldsymbol{v}^{*}-\boldsymbol{W}_L\boldsymbol{k}_1^L)\,\boldsymbol{M},
\;\;
\boldsymbol{M}
=(\boldsymbol{k}_1^L)^{T}\boldsymbol{P}_L
\!\big(
\boldsymbol{K}_p^L(\boldsymbol{K}_p^L)^{T}\boldsymbol{P}_L
+\boldsymbol{k}_1^L(\boldsymbol{k}_1^L)^{T}\boldsymbol{P}_L
+\boldsymbol{I}\big)^{-1}.
\end{equation}
Here \(\boldsymbol{\Delta}_{\mathrm{proxy}}\) is already the
update applied to the virtual final layer, and \(\boldsymbol{M}\) is
request-specific through \(\boldsymbol{k}_1^L\). Substituting
\(\boldsymbol{\Delta}_{\mathrm{proxy}}\) into the chain rule yields a
tractable hypergradient:
\begin{equation}
\label{eq:proxy_gradient}
\nabla_{\boldsymbol{v}^{*}}\mathcal{L}_{\mathrm{meta}}
=
\frac{\partial\mathcal{L}_{\mathrm{meta}}}{\partial f}
\cdot
\frac{\partial f}{\partial\boldsymbol{\Delta}_{\mathrm{proxy}}}
\cdot
\frac{\partial\boldsymbol{\Delta}_{\mathrm{proxy}}}{\partial\boldsymbol{v}^{*}}
+
\left.
\frac{\partial\mathcal{L}_{\mathrm{meta}}}{\partial\boldsymbol{v}^{*}}
\right|_{f}
=
\nabla_{\boldsymbol{\Delta}_{\mathrm{proxy}}}
\mathcal{L}_{\mathrm{meta}}
\cdot
\boldsymbol{M}^{T}
+
2(\boldsymbol{v}^{*}-\boldsymbol{v}_{\mathrm{init}}).
\end{equation}
For readability, Eqs.~\eqref{eq:proxy_definition}--\eqref{eq:proxy_gradient}
use single-request notation. In implementation, current-batch keys and target
residuals are column-stacked and updated in batched matrix form. The Task
Signal denotes the downstream-mediated gradient induced by the semantic-accuracy
and locality terms. The
Structural Gate filters this signal through \(\boldsymbol{P}_L\) and
\(\boldsymbol{K}_p^L(\boldsymbol{K}_p^L)^{T}\), propagating approximate
downstream constraints back to \(\boldsymbol{v}^{*}\). We instantiate the gate in the main experiments with the
AlphaEdit-style projected update, whose local response admits the
closed form in Eq.~\eqref{eq:proxy_definition}.
The same bi-level interface also admits operator-specific gates
for native MEMIT and ridge-stabilized EMMET; their derivations
and paired experiments are reported in
Appendix~\ref{app:cross_operator}.

\begin{algorithm}[htp]
\caption{MetaKE}
\label{alg:MetaKE}
\small
\begin{algorithmic}[1]
\Require LR \(\eta\), iterations \(T\), data \((x_e,y_e,x_{\mathrm{loc}})\),
initial target \(\boldsymbol{v}_{\mathrm{init}}\)
\Ensure Edited weights \(\widehat{\boldsymbol{W}}\)
\State Precompute Structural Gate \(\boldsymbol{M}\)
(Eq.~\ref{eq:proxy_definition}) and initialize
\(\boldsymbol{v}^{*}\leftarrow\boldsymbol{v}_{\mathrm{init}}\)
\For{\(t=0\) \textbf{to} \(T-1\)}
\State \(\widehat{\boldsymbol{W}}_{L,t}\leftarrow
\boldsymbol{W}_L+(\boldsymbol{v}^{*}-\boldsymbol{W}_L\boldsymbol{k}_1^L)
\boldsymbol{M}\)
\Comment{Phase 1: Virtual Look-ahead}
\State \(\boldsymbol{g}_t\leftarrow
\nabla_{\boldsymbol{v}^{*}}\mathcal{L}_{\mathrm{meta}}
(\widehat{\boldsymbol{W}}_{L,t})\)
\Comment{Phase 2: hypergradient via Eq.~\ref{eq:proxy_gradient}}
\State \(\boldsymbol{v}^{*}\leftarrow\boldsymbol{v}^{*}-\eta\,\boldsymbol{g}_t\)
\Comment{Refine target with downstream-aware feedback}
\EndFor
\State \(\widehat{\boldsymbol{W}}\leftarrow
\mathrm{Standard~multi\mbox{-}layer~editing}
(\boldsymbol{W},\boldsymbol{v}_T^{*};\mathcal{L})\)
\Comment{Final downstream update (e.g., AlphaEdit)}
\State \textbf{return} \(\widehat{\boldsymbol{W}}\)
\end{algorithmic}
\end{algorithm}

\subsection{Optimization Algorithm: MetaKE}
\label{sec:method:algo}

MetaKE refines the target \(\boldsymbol{v}^*\) via an iterative
\emph{Look-ahead and Correct} loop.

\paragraph{Phase 1: Virtual Look-ahead.}
At iteration \(t\), MetaKE simulates the edit effect by forming virtual weights
for the final edited layer as
\(\widehat{\boldsymbol{W}}_{L,t}=\boldsymbol{W}_L+
\boldsymbol{\Delta}_{\mathrm{proxy}}(\boldsymbol{v}_t^*)\). This step assesses
how the current target would be realized under the AlphaEdit-style projection
and previous-edit preservation structure, without differentiating through the
full multi-layer editing process.

\paragraph{Phase 2: Downstream-Aware Feedback Correction.}
MetaKE evaluates \(\mathcal{L}_{\mathrm{meta}}\) under the virtual weights and
backpropagates the signal through the structural gate \(\boldsymbol{M}^{T}\).
The resulting update refines \(\boldsymbol{v}^*\) using an approximate
downstream-aware feedback signal. After \(T\) iterations, the refined target is
passed to the standard multi-layer editing procedure to perform the final
parameter update across the edited layers \(\mathcal{L}\). Under
standard local regularity assumptions, the convergence analysis of the
proxy-driven iteration is deferred to Appendix~\ref{app:convergence}.

\section{Theoretical Justification}
\label{sec:theory}

Sec.~\ref{sec:analysis} identifies that the trade-off bottleneck stems from
optimizing \(\boldsymbol{v}^{*}\) without access to the downstream
constraints that determine each \(\beta_i\). MetaKE grants this access
through the implicit map \(\boldsymbol{\Delta}_l^{*}(\boldsymbol{v}^{*})\):
in Eq.~\eqref{eq:proxy_gradient}, the Task Signal carries the requirements
of Theorem~\ref{thm:regularization_disconnect}, while the Structural Gate
\(\boldsymbol{M}^{T}\) exposes the request-specific association at layer
\(L\) through \(\boldsymbol{P}_L\) and
\(\boldsymbol{K}_p^L(\boldsymbol{K}_p^L)^{T}\). The resulting feedback is no
longer association-blind, and the residual scale is adapted per request.

\begin{figure}[t]
	\centering
	    \includegraphics[width=\linewidth]{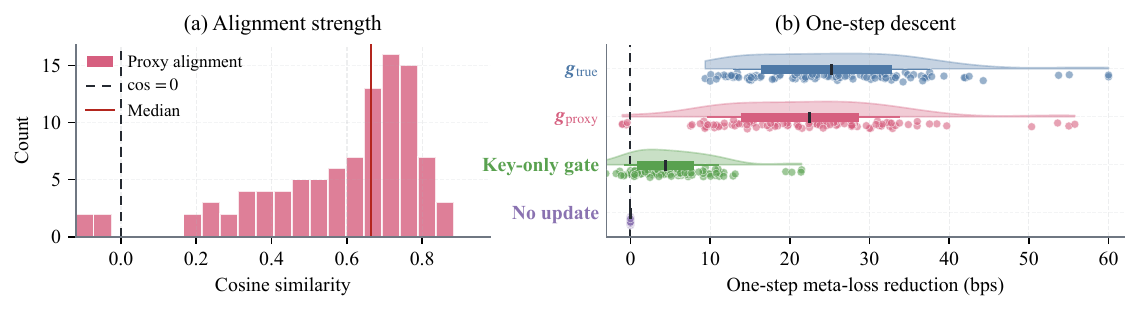}    \vspace{-0.2in}
\caption{Fidelity of the Structural Gradient Proxy.
(a) Cosine similarity between \(\boldsymbol{g}_{\mathrm{proxy}}\) and
\(\boldsymbol{g}_{\mathrm{true}}\). \(\boldsymbol{g}_{\mathrm{proxy}}\) aligns with
\(\boldsymbol{g}_{\mathrm{true}}\) in \(96\%\) of \(N=100\) edits, with
\(72\%\) above \(\cos>0.5\).
(b) One-step meta-loss reduction. The proxy achieves \(22.2\) bps reduction,
recovering \(85\%\) of the \(\boldsymbol{g}_{\mathrm{true}}\) descent effect.}
    \label{fig:proxy_validation}
    \vspace{-0.1in}
\end{figure}

\textbf{Proxy fidelity.}
We next formalize when the structural proxy preserves the orientation of the
full multi-layer hypergradient. Let \(\boldsymbol{g}_{\mathrm{true}}\) denote
the exact full hypergradient of the outer objective, including the direct
anchoring derivative, and let \(\boldsymbol{g}_{\mathrm{proxy}}\) be the
corresponding full proxy hypergradient induced by the frozen structural gate.
The following result shows that the proxy remains a descent direction when the
non-final-layer contribution and final-layer proxy gap are jointly controlled.

\begin{theorem}[Conditional Fidelity of the Structural Gradient Proxy]
\label{thm:fidelity}
Let \(F(\boldsymbol{v}^{*})\) be the exact outer objective, with exact
hypergradient \(\boldsymbol{g}_{\mathrm{true}}\) and structural proxy gradient
\(\boldsymbol{g}_{\mathrm{proxy}}\). Under local fidelity conditions, if
\(\xi+\frac{K(1+\xi)}{\alpha_0}\varepsilon_{\mathrm{AE}}
<\chi(\boldsymbol{v}^{*})\), where
\(\chi(\boldsymbol{v}^{*})=
\|\boldsymbol{g}_{\mathrm{proxy}}\|_2/
\|\boldsymbol{h}_{\mathrm{proxy}}\|_2\) and
\(\boldsymbol{h}_{\mathrm{proxy}}\) is the downstream-mediated component of
\(\boldsymbol{g}_{\mathrm{proxy}}\), then
\(\langle\boldsymbol{g}_{\mathrm{true}},
\boldsymbol{g}_{\mathrm{proxy}}\rangle>0\). Thus
\(-\boldsymbol{g}_{\mathrm{proxy}}\) is a descent direction for \(F\); if
\(F\) is locally smooth, small-enough steps along it decrease \(F\).
Proof in Appendix~\ref{app:proxy_fidelity}.
\end{theorem}

The condition couples the two error sources: stronger non-final-layer leakage
leaves less room for final-layer proxy error, with the admissible range
decreasing as \((1-\xi)/(1+\xi)\), which gives a quantitative form of the
Structural Consistency Hypothesis in Sec.~\ref{sec:method:proxy}.
Fig.~\ref{fig:proxy_validation} evaluates the descent-direction implication of Theorem~\ref{thm:fidelity}, across the tested edits,
\(\boldsymbol{g}_{\mathrm{proxy}}\) is positively aligned with the true
multi-layer hypergradient \(\boldsymbol{g}_{\mathrm{true}}\) and yields a
positive one-step reduction of \(\mathcal{L}_{\mathrm{meta}}\)
(Appendix~\ref{app:proxy_validation_details}).

\section{Experiments}
\label{sec:experiments}
In this section, we aim to answer three questions:
(1) Does MetaKE improve editing performance over strong baselines?
(2) Does MetaKE mitigate the disconnect between the two optimization stages, and which components are responsible for the gains?
(3) Does MetaKE remain stable under long-horizon sequential editing while better preserving general capabilities and internal representation geometry?

\begin{table*}[t]
\caption{\label{main-results}
  Comparison of MetaKE with editing methods on the knowledge editing task. The best results are highlighted in \textcolor{red}{red}, while the second-best results are \underline{underlined}. 
  }
  \label{table 1}
\renewcommand{\arraystretch}{1.3}
  \centering
  \resizebox{\textwidth}{!}{
    \begin{tabular}{cc|ccccc|ccc}
    \toprule
      \multirow{2}{*}{\textbf{Model}} & \multirow{2}{*}{\textbf{Method}} & \multicolumn{5}{c}{\textbf{CounterFact}} & \multicolumn{3}{c}{\textbf{ZsRE}} \\
      \cmidrule(lr){3-7} \cmidrule(lr){8-10}
       & & \textbf{Eff.}$\uparrow$ & \textbf{Gen.}$\uparrow$ & \textbf{Spe.}$\uparrow$ & \textbf{Flu.}$\uparrow$ & \textbf{Consis.}$\uparrow$ & \textbf{Eff.}$\uparrow$ & \textbf{Gen.}$\uparrow$ & \textbf{Spe.}$\uparrow$ \\
      \midrule
      \multirow{11}{*}{\rotatebox{90}{LLaMA3~(8B)}} 
       & MEND & $63.18_{\pm 0.34}$ & $61.26_{\pm 0.32}$ & $45.42_{\pm 0.41}$ & $372.08_{\pm 0.75}$ & $4.19_{\pm 0.06}$ & $0.88_{\pm 0.05}$ & $1.06_{\pm 0.05}$ & $0.51_{\pm 0.03}$ \\
       & MALMEN & $38.82_{\pm 0.34}$ & $49.01_{\pm 0.41}$ & $56.64_{\pm 0.37}$ & $524.48_{\pm 0.45}$ & $19.83_{\pm 0.13}$ & $39.65_{\pm 0.21}$ & $40.56_{\pm 0.28}$ & $30.47_{\pm 0.24}$ \\
       & ROME & ${66.18}_{\pm 0.47}$ & ${63.47}_{\pm 0.41}$ & ${48.61}_{\pm 0.25}$ & ${458.75}_{\pm 0.23}$ & ${2.05}_{\pm 0.06}$ & ${1.28}_{\pm 0.04}$ & ${1.32}_{\pm 0.03}$ & ${0.64}_{\pm 0.03}$ \\
       & MEMIT & $65.52_{\pm 0.35}$ & $63.94_{\pm 0.38}$ & $51.26_{\pm 0.41}$ & $452.55_{\pm 1.24}$ & $5.94_{\pm 0.15}$ & $36.15_{\pm 0.31}$ & $31.42_{\pm 0.28}$ & $18.57_{\pm 0.25}$ \\
       & PRUNE & $66.82_{\pm 0.42}$ & $64.64_{\pm 0.37}$ & $49.88_{\pm 0.25}$ & $440.67_{\pm 1.12}$ & $6.12_{\pm 0.09}$ & $0.52_{\pm 0.05}$ & $0.58_{\pm 0.07}$ & $1.34_{\pm 0.04}$ \\
       & RECT & $64.53_{\pm 0.35}$ & $62.35_{\pm 0.44}$ & $60.78_{\pm 0.31}$ & $520.94_{\pm 0.36}$ & $19.21_{\pm 0.15}$ & $87.41_{\pm 0.22}$ & $82.35_{\pm 0.24}$ & $32.18_{\pm 0.22}$ \\
       & PMET & ${99.40}_{\pm 0.21}$ & ${93.24}_{\pm 0.79}$ & $\second{{77.42}_{\pm 0.80}}$ & $\second{{625.61}_{\pm 0.52}}$ & ${32.25}_{\pm 0.35}$ & ${94.93}_{\pm 0.27}$ & ${90.47}_{\pm 0.54}$ & ${32.61}_{\pm 0.66}$ \\
       & AlphaEdit & $99.23_{\pm 0.15}$ & $86.21_{\pm 0.18}$ & $77.16_{\pm 0.21}$ & $624.95_{\pm 0.24}$ & $30.25_{\pm 0.12}$ & $94.27_{\pm 0.15}$ & $88.35_{\pm 0.12}$ & $\second{32.65_{\pm 0.14}}$ \\
       & $\text{AlphaEdit}_{\text{BLUE}}$ & $\second{99.92_{\pm 0.05}}$ & $\second{97.55_{\pm 0.46}}$ & $77.12_{\pm 0.86}$ & $624.92_{\pm 0.24}$ & $\second{33.84_{\pm 0.42}}$ & $\second{95.82_{\pm 0.46}}$ & $\second{91.85_{\pm 0.42}}$ & $32.10_{\pm 0.68}$ \\
       & SPHERE  & $99.75_{\pm 0.22}$&$88.38_{\pm 0.32}$ &$74.84_{\pm 0.16}$ &$616.53_{\pm 0.34}$  &$28.67_{\pm 0.24}$  &  $94.25_{\pm 0.32}$&  $88.35_{\pm 0.38}$&$32.02_{\pm 0.12}$\\
       \rowcolor{metablue}
       & MetaKE & $\best{99.97}_{\pm 0.02}$ & $\best{97.95}_{\pm 0.21}$ & $\best{78.75}_{\pm 0.35}$ & $\best{626.54}_{\pm 0.42}$ & $\best{34.93}_{\pm 0.35}$ & $\best{96.84}_{\pm 0.14}$ & $\best{92.45}_{\pm 0.16}$ & $\best{33.12}_{\pm 0.21}$ \\
      \midrule \midrule
      \multirow{11}{*}{\rotatebox{90}{GPT-J~(6B)}} 
       & MEND & $46.15_{\pm 0.50}$ & $46.22_{\pm 0.51}$ & $53.9_{\pm 0.48}$ & $242.24_{\pm 0.39}$ & $3.94_{\pm 0.04}$ & $0.71_{\pm 0.04}$ & $0.71_{\pm 0.04}$ & $0.52_{\pm 0.03}$ \\
       & MALMEN & $63.37_{\pm0.29}$ & $59.87_{\pm0.36}$ & $41.40_{\pm0.44}$ & $229.56_{\pm0.27}$ & $20.35_{\pm0.08}$ & $33.3_{\pm 0.41}$ & $31.68_{\pm 0.43}$ & $3.81_{\pm 0.04}$ \\
       & ROME & $52.32_{\pm 0.24}$ & $50.67_{\pm 0.32}$ & $51.05_{\pm 0.65}$ & $581.57_{\pm 0.22}$ & $1.54_{\pm 0.04}$ & $55.62_{\pm 0.24}$ & $52.71_{\pm 0.33}$ & $13.34_{\pm 0.26}$ \\
       & MEMIT & $98.12_{\pm 0.18}$ & $95.71_{\pm 0.15}$ & $63.74_{\pm 0.28}$ & $594.15_{\pm 0.64}$ & $39.33_{\pm 0.24}$ & $96.85_{\pm 0.18}$ & $93.52_{\pm 0.25}$ & $\best{31.42}_{\pm 0.29}$ \\
       & PRUNE & $84.43_{\pm 0.29}$ & $87.54_{\pm 0.26}$ & $53.35_{\pm 0.34}$ & $460.27_{\pm 0.65}$ & $23.34_{\pm 0.21}$ & $25.64_{\pm 0.16}$ & $24.08_{\pm 0.22}$ & $20.45_{\pm 0.13}$ \\
       & RECT & $98.83_{\pm 0.10}$ & $86.25_{\pm 0.26}$ & $72.54_{\pm 0.22}$ & $616.68_{\pm 0.17}$ & $41.43_{\pm 0.19}$ & $96.18_{\pm 0.24}$ & $91.75_{\pm 0.21}$ & $28.25_{\pm 0.26}$ \\
       & PMET & ${99.75}_{\pm 0.16}$ & ${94.78}_{\pm 0.53}$ & ${75.18}_{\pm 0.78}$ & ${618.59}_{\pm 0.49}$ & ${41.40}_{\pm 0.32}$ & ${99.36}_{\pm 0.19}$ & ${96.23}_{\pm 0.51}$ & ${29.53}_{\pm 0.80}$ \\
       & AlphaEdit & $99.75_{\pm 0.14}$ & $96.47_{\pm 0.21}$ & $\second{76.14_{\pm 0.26}}$ & $618.41_{\pm 0.23}$ & $42.36_{\pm 0.15}$ & $99.56_{\pm 0.15}$ & $95.84_{\pm 0.26}$ & $28.31_{\pm 0.22}$ \\
       & $\text{AlphaEdit}_{\text{BLUE}}$ & $99.72_{\pm 0.17}$ & $\second{97.75_{\pm 0.53}}$ & $76.12_{\pm 0.86}$ & $\best{621.54}_{\pm 0.75}$ & $41.54_{\pm 0.35}$ & $99.45_{\pm 0.25}$ & $93.64_{\pm 0.48}$ & $28.41_{\pm 0.72}$ \\
       & SPHERE  & $\second{99.80_{\pm 0.22}}$&$95.40_{\pm 0.35}$ &$76.12_{\pm 0.74}$ &$611.21_{\pm 0.54}$ &$\second{43.20_{\pm 0.32}}$& $\second{99.69_{\pm 0.08}}$& $\second{97.10_{\pm 0.15}}$& $28.09_{\pm 0.19}$\\
       \rowcolor{metablue}
       & MetaKE & $\best{99.88}_{\pm 0.10}$ & $\best{98.68}_{\pm 0.24}$ & $\best{76.47}_{\pm 0.38}$ & $\second{620.98_{\pm 0.35}}$ & $\best{43.34}_{\pm 0.26}$ & $\best{99.82}_{\pm 0.12}$ & $\best{97.37}_{\pm 0.15}$ & $\second{29.73_{\pm 0.20}}$ \\
      \midrule \midrule
      \multirow{11}{*}{\rotatebox{90}{GPT2-XL~(1.5B)}} 
       & MEND & $49.70_{\pm 0.49}$ & $49.62_{\pm 0.46}$ & $50.42_{\pm 0.52}$ & $513.08_{\pm 0.25}$ & $26.22_{\pm 0.08}$ & $0.24_{\pm 0.05}$ & $0.24_{\pm 0.05}$ & $2.07_{\pm 0.03}$ \\
       & MALMEN & $52.60_{\pm0.33}$ & $50.99_{\pm0.38}$ & $49.76_{\pm0.41}$ & $536.67_{\pm0.34}$ & $12.84_{\pm0.09}$ & $32.99_{\pm 0.38}$ & $33.00_{\pm 0.03}$ & $2.62_{\pm 0.03}$ \\
       & ROME & $51.73_{\pm 0.38}$ & $49.08_{\pm 0.21}$ & $50.87_{\pm 0.24}$ & $527.96_{\pm 0.89}$ & $1.06_{\pm 0.05}$ & $41.25_{\pm 0.45}$ & $37.62_{\pm 0.38}$ & $12.35_{\pm 0.22}$ \\
       & MEMIT & $94.83_{\pm 0.15}$ & $85.62_{\pm 0.31}$ & $60.21_{\pm 0.33}$ & $472.14_{\pm 0.45}$ & $22.08_{\pm 0.15}$ & $79.58_{\pm 0.25}$ & $72.45_{\pm 0.31}$ & $26.12_{\pm 0.18}$ \\
       & PRUNE & $75.73_{\pm 0.28}$ & $73.75_{\pm 0.21}$ & $52.27_{\pm 0.16}$ & $528.87_{\pm 0.19}$ & $13.84_{\pm 0.21}$ & $28.42_{\pm 0.26}$ & $25.35_{\pm 0.24}$ & $14.52_{\pm 0.18}$ \\
       & RECT & $92.43_{\pm 0.13}$ & $81.35_{\pm 0.27}$ & $65.06_{\pm 0.31}$ & $483.17_{\pm 0.85}$ & $20.46_{\pm 0.09}$ & $83.45_{\pm 0.38}$ & $75.28_{\pm 0.35}$ & $24.93_{\pm 0.24}$ \\
       & PMET & ${96.55}_{\pm 0.58}$ & ${89.95}_{\pm 0.72}$ & ${64.41}_{\pm 0.66}$ & ${560.97}_{\pm 1.74}$ & ${33.90}_{\pm 0.49}$ & ${94.51}_{\pm 0.38}$ & ${87.79}_{\pm 0.62}$ & ${26.47}_{\pm 0.68}$ \\
       & AlphaEdit & $\second{99.42_{\pm 0.25}}$ & $\second{94.35_{\pm 0.24}}$ & $66.05_{\pm 0.22}$ & $592.17_{\pm 0.28}$ & $38.85_{\pm 0.16}$ & $93.52_{\pm 0.28}$ & $83.65_{\pm 0.36}$ & $25.85_{\pm 0.22}$ \\
       & $\text{AlphaEdit}_{\text{BLUE}}$ & $99.13_{\pm 0.18}$ & $92.04_{\pm 0.52}$ & $69.14_{\pm 0.69}$ & $612.51_{\pm 0.68}$ & $\second{39.63_{\pm 0.35}}$ & $\second{97.42_{\pm 0.45}}$ & $\second{88.75_{\pm 0.82}}$ & $26.58_{\pm 0.75}$\\
       & SPHERE  & $99.30_{\pm 0.19}$& $89.10_{\pm 0.36}$& $\second{72.35_{\pm 0.28}}$& $\second{616.41_{\pm 0.86}}$& $39.43_{\pm 0.25}$& $91.41_{\pm 0.16}$& $81.92_{\pm 0.24}$ &$\second{26.74_{\pm 0.21}}$ \\
       \rowcolor{metablue}
       & MetaKE & $\best{99.70}_{\pm 0.18}$ & $\best{94.88}_{\pm 0.31}$ & $\best{72.64}_{\pm 0.26}$ & $\best{617.57}_{\pm 0.25}$ & $\best{41.05}_{\pm 0.28}$ & $\best{98.12}_{\pm 0.18}$ & $\best{91.26}_{\pm 0.32}$ & $\best{27.23}_{\pm 0.25}$ \\
      \bottomrule
    \end{tabular}
   }
\end{table*}

\subsection{Experimental Setup}
\paragraph{Datasets \& Evaluation Metrics.}
We evaluate our method on two benchmark datasets: ZsRE \citep{levy-etal-2017-zero} and CounterFact \citep{meng2022locating}. We report five key metrics:
Efficacy (success rate),
Generalization (paraphrase robustness),
Specificity (locality preservation),
Fluency (generation entropy),
and Consistency (reference mapping) to comprehensively assess the performance of MetaKE.

\paragraph{Baselines \& Implementation Details.}
We compare MetaKE against KE baselines: ROME \citep{meng2022locating}, MEMIT \citep{mengmass}, PRUNE \citep{ma2025perturbationrestrained}, RECT \citep{gu2024model}, PMET \citep{li2024pmet}, AlphaEdit \citep{fangalphaedit}, BLUE \citep{li2025rethinking}, and SPHERE \citep{liu2026energyregularized}. We also include meta-learning baselines MEND~\citep{mitchellfast} and MALMEN~\citep{tan2024massive}. Experiments are conducted on GPT2-XL (1.5B) \citep{radford2019language}, GPT-J (6B) \citep{gpt-j}, and
LLaMA3 (8B) \citep{llama3modelcard}.
Unless otherwise noted, we use hyperparameters: shared
settings $T=15$, $\eta = 5\times10^{-3}$ for three models.
The main results are averaged over three random seeds, and experiments are on a single NVIDIA A100 GPU.

\subsection{Main Results}
\label{sec:main_results}
We assess MetaKE under the same sequential editing protocol established by \citep{fangalphaedit}, randomly sampling 2,000 editing instances with a batch size of 100.
As shown in Table~\ref{table 1}, MetaKE achieves the
best results on most reported metrics across CounterFact
and ZsRE for all three evaluated models.
On ZsRE with GPT2-XL, MetaKE improves Generalization by 9.10\% relative to AlphaEdit, indicating stronger robustness to paraphrased edit queries while maintaining competitive preservation performance. MetaKE attains notable gains on LLaMA3 over AlphaEdit\textsubscript{BLUE}, improving Consistency by 3.22\% and Specificity by 2.11\% while preserving a near-perfect Efficacy of 99.97\%. Overall, by aligning planned targets with the model's feasible geometry, MetaKE achieves stronger generalization while better preserving specificity, yielding a better trade-off between semantic accuracy and editability.

\subsection{Analysis of the Two-Stage Optimization Disconnect}
\label{sec:disconnect_validation}

To assess the two-stage optimization disconnect in Sec.~\ref{sec:analysis}, we analyze association-dependent residual realization and compare high- and low-\(\rho\) efficacy curves under upstream KL sweeps. Details are provided in Appendix~\ref{app:disconnect_protocol}.

\setlength{\columnsep}{0pt}
\begin{wrapfigure}[11]{r}{0.49\columnwidth}
\vspace{-2.7em} 
\centering
\includegraphics[width=0.49\columnwidth]{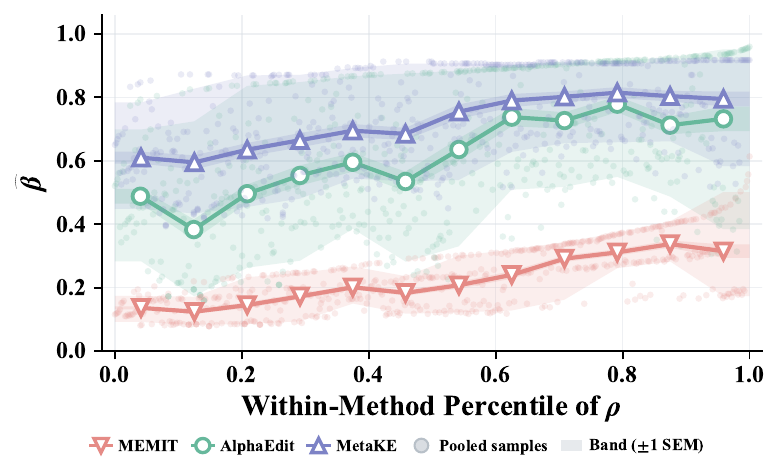}
\begin{minipage}{0.42\columnwidth}
\caption{
Empirical realization coefficient \(\widehat{\beta}\) across association \(\rho\).
}
\label{fig:validation}
\end{minipage}
\end{wrapfigure}

\paragraph{Realized-residual attenuation.}
For each method, we sample 500 edit requests, compute the request-specific association \(\rho_i\), and measure empirical realized coefficient
\(\widehat{\beta}_i\), which is the realized fraction of the planned residual after editing. Fig.~\ref{fig:validation} shows that MEMIT exhibits strong attenuation, especially in
low-\(\rho\) regions; AlphaEdit improves the retained signal but leaves a clear realization gap. MetaKE achieves higher \(\widehat{\beta}\) across association percentiles, suggesting that downstream-aware target optimization mitigates residual attenuation under downstream constraints.

\begin{figure}[htp]
	\centering
    
	\includegraphics[width=\linewidth]{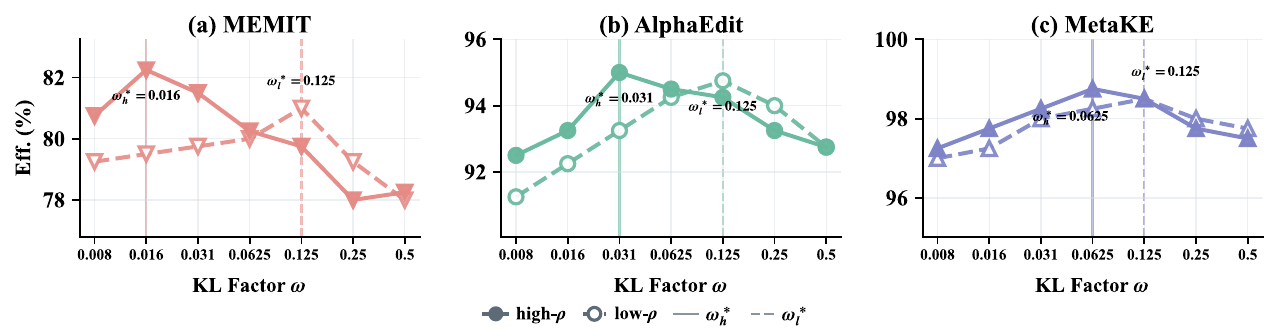}
	\caption{Efficacy of high- and low-\(\rho\) requests under different KL strengths.}
	\label{fig:theorem2}
          \vspace{-0.1in}
\end{figure}

\paragraph{Association-dependent requirements on KL strength.}
We further test whether requests with different downstream associations require different upstream KL strengths. For each method, we split 2,000 requests into high-\(\rho\) and low-\(\rho\) groups of 400 requests each, and sweep \(\omega\in\{0.008,0.016,0.031,0.0625,0.125,0.25,0.5\}\). As shown in Fig.~\ref{fig:theorem2}, MEMIT attains its best high-\(\rho\) Efficacy at a much smaller KL value than its best low-\(\rho\) Efficacy, indicating different favorable planning strengths. AlphaEdit reduces this discrepancy but still shows distinct group-wise optima. In contrast, MetaKE makes the two groups' favorable KL regions substantially overlap and maintains higher Efficacy across most KL values. Together with the residual-realization results above, this indicates that since the feasible residual scale depends on \(\beta_i=\rho_i/(1+\rho_i)\), different association regimes favor different effective KL ranges; downstream-aware target optimization alleviates this limitation of association-blind planning, leading to a more favorable accuracy--editability trade-off. 
This comparison evaluates end-to-end sensitivity to the initial
KL weight.
The refinement-matched control in
Appendix~\ref{app:disconnect_protocol} shows that generic
refinement contributes to the gains, while the Structural Gate
provides additional improvements and reduces KL sensitivity.

\subsection{Ablation Study}
\label{sec:ablation}

\setlength{\columnsep}{4pt} 

We ablate MetaKE on LLaMA3 with ZsRE to identify its gain sources. We compare four variants: (i) w/o Bi-level Optimization, which directly uses
\(\boldsymbol{v}_{\mathrm{init}}\) for final editing and removes iterative
target refinement; (ii) w/o Structural Gate, which removes downstream
structural filtering by replacing \(\boldsymbol{M}\) with a
dimension-compatible key-only gate obtained by dropping \(\boldsymbol{P}_L\)
and \(\boldsymbol{K}_p^L(\boldsymbol{K}_p^L)^T\) from
Eq.~\eqref{eq:proxy_definition}; (iii) Stop-Grad Feedback, which keeps the
virtual look-ahead forward pass but detaches
\(\boldsymbol{\Delta}_{\mathrm{proxy}}(\boldsymbol{v}^{*})\) before computing
the edit and locality losses; and (iv) MetaKE.

\begin{wraptable}{r}{0.57\textwidth} 
\centering
\vspace{-0.2in}
\small
\renewcommand{\arraystretch}{1.1}
\setlength{\tabcolsep}{1.5mm}
\caption{\label{tab:ablation} Ablation study on MetaKE.}
\begin{tabular}{l|ccc}
\toprule
\textbf{Variant} & \textbf{Eff.}$\uparrow$ & \textbf{Gen.}$\uparrow$ & \textbf{Spe.}$\uparrow$ \\
\midrule
MetaKE                  & \textbf{96.84} & \textbf{92.45} & \textbf{33.12} \\
\quad - w/o Bi-level Optimization        & 94.27 & 88.35 & 32.65 \\
\quad - w/o Structural Gate     & 95.88 & 90.75 & 32.14 \\
\quad - Stop-Grad Feedback         & 94.31 & 88.24 & 32.48 \\
\bottomrule
\end{tabular}
\end{wraptable}

As shown in Table~\ref{tab:ablation}, removing bi-level optimization reduces
Efficacy and Generalization, indicating that target refinement benefits from
downstream constraints feedback. Stop-Grad Feedback performs
close
to w/o Bi-level Optimization, showing that
virtual look-ahead alone is insufficient without differentiable downstream-aware
feedback. Using the key-only gate also degrades all metrics, with a particularly notable relative drop in Specificity,
suggesting that structural filtering through downstream constraints helps align the outer update with
downstream constraints. Overall, MetaKE's gains come from target refinement
driven by differentiable downstream-realization feedback and filtered by the
Structural Gate. We further analyze the runtime and memory cost of this feedback
path in Appendix~\ref{app:runtime} and Appendix~\ref{app:memory}; the results
show that MetaKE improves the accuracy--editability trade-off with manageable
overhead, while remaining within the same cost scale as existing locate-then-edit
methods.

\begin{figure}[t]
	\centering
	\includegraphics[width=\linewidth]{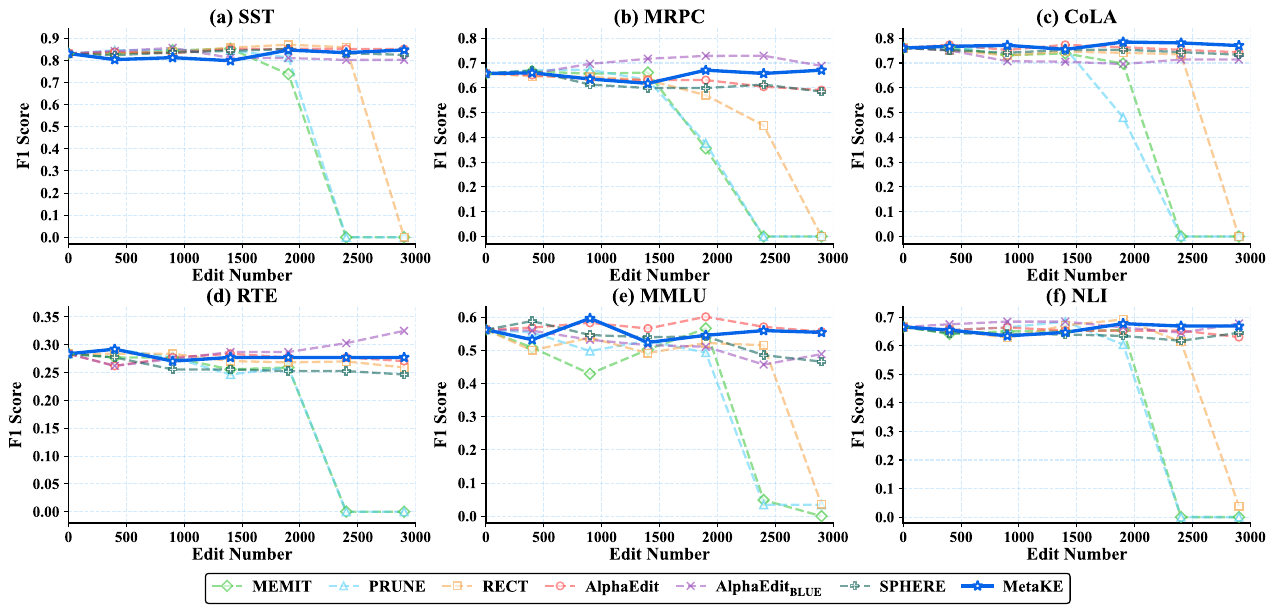}
\vspace{-0.1in}
	\caption{General-capability retention under long-horizon sequential editing on LLaMA3 (8B).}
 \label{fig:sequential}
\end{figure}

\subsection{Sustainability of General Capabilities under Sequential Editing}
\label{sec:exp:sustain}

To evaluate stability under sequential editing, we measure the post-edited model's general capability on six tasks from the GLUE
benchmark \citep{wang2018glue}, including SST \citep{socher2013recursive},
MRPC \citep{dolan2005automatically}, CoLA \citep{warstadt2019neural},
RTE \citep{bentivogli2009fifth}, NLI \citep{williams2018broad}, and MMLU \citep{hendrycks2021measuring}. We perform 3,000 sequential edit requests on LLaMA3, applying batch size 100 and evaluating the general capability score every 500 edits. As shown in Fig.~\ref{fig:sequential}, MEMIT, PRUNE, and RECT all undergo sharp degradation, and in multiple tasks collapse to near-zero performance in the late editing stage. AlphaEdit, AlphaEdit\textsubscript{BLUE} and SPHERE remain considerably more stable. In contrast, MetaKE maintains nearly flat performance trajectories, and remains competitive with or better than the strongest baselines on most benchmarks.

\subsection{Hidden Representations Analysis} 
The two-stage optimization disconnect can also manifest as shifts in internal representation geometry after editing. To examine this effect, we compare the hidden-state distributions of pre- and post-edit models. We extract hidden states for 1,000 randomly sampled factual queries and project them to a 2D plane using t-SNE~\citep{maaten2008visualizing}. As shown in Fig.~\ref{fig:hidden}, AlphaEdit exhibits a visible shift between pre- and post-edit representations across different backbones. In contrast, MetaKE produces a smaller distributional shift, suggesting that downstream-aware target optimization helps preserve the model's internal representation geometry while updating the target knowledge.

\begin{figure}[t]
	\centering
	\includegraphics[width=\linewidth]{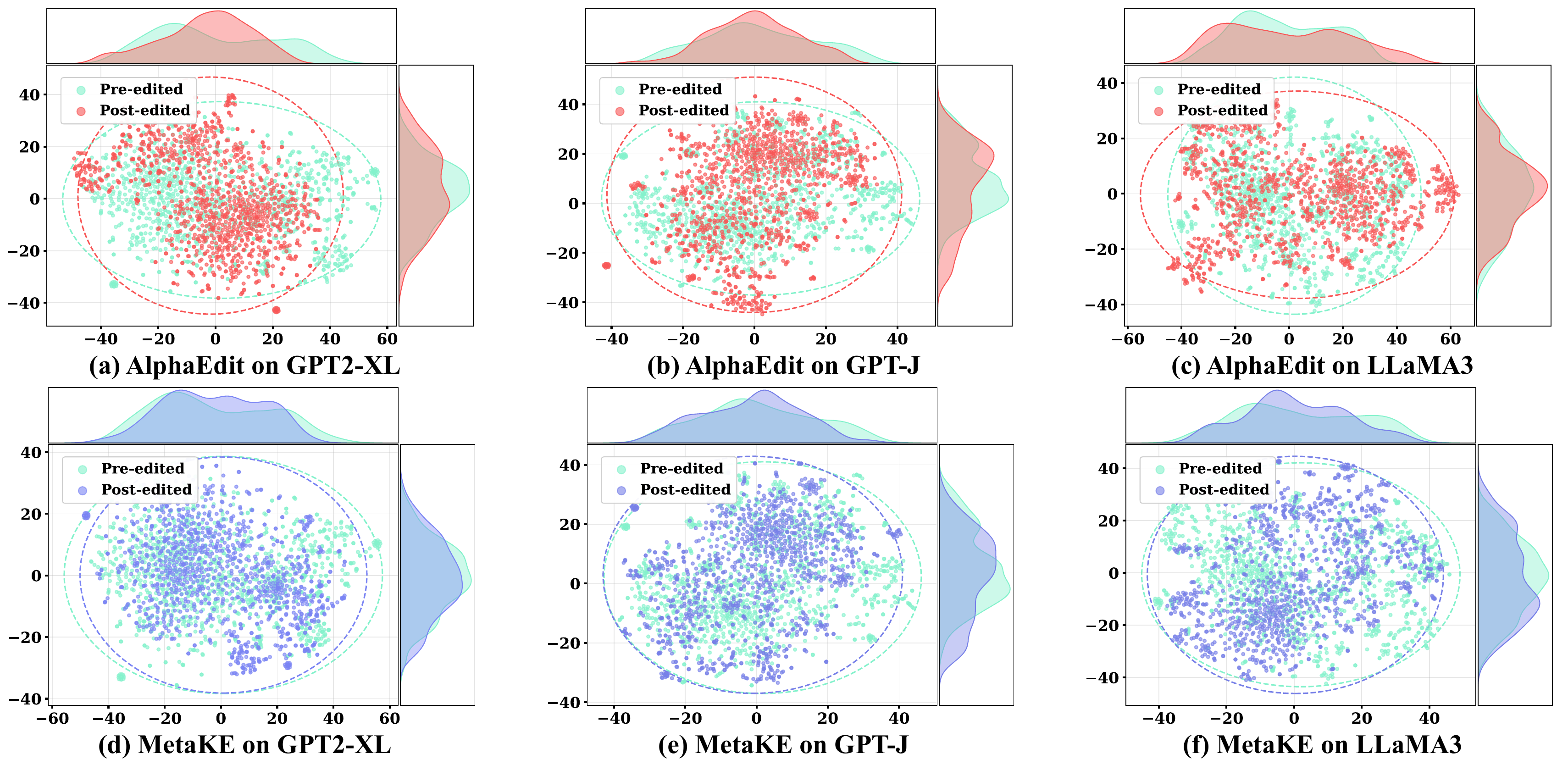}
 \vspace{-0.1in}
	\caption{The distribution of hidden representations of pre-edited and post-edited LLMs.}
 \label{fig:hidden}
\end{figure}

\section{Conclusion}
\label{sec:conclusion}

We identify a two-stage optimization disconnect in locate-then-edit KE, where upstream target optimization plans a residual without observing request-specific downstream constraints. MetaKE addresses this by refining the target representation with downstream-aware feedback, made tractable by a Structural Gradient Proxy. Experiments show that MetaKE mitigates residual attenuation, reduces association-blind KL limitations, and improves the accuracy--editability trade-off. Future work will extend this principle to more diverse and general nonlinear editing settings.

\section*{Acknowledgments and Disclosures}
This work was supported by the National Natural Science Foundation of China under Grant Nos. 42550187 and 62476191, with computing resources provided by our research group. The authors declare no competing interests.

\newpage
{
\small
\bibliographystyle{IEEEtran}
\bibliography{custom}

@inproceedings{yao2023editing,
  title={Editing large language models: Problems, methods, and opportunities},
  author={Yao, Yunzhi and Wang, Peng and Tian, Bozhong and Cheng, Siyuan and Li, Zhoubo and Deng, Shumin and Chen, Huajun and Zhang, Ningyu},
  booktitle={Proceedings of the 2023 Conference on Empirical Methods in Natural Language Processing},
  pages={10222--10240},
  year={2023}
}

@article{zhang2024comprehensive,
  title={A comprehensive study of knowledge editing for large language models},
  author={Zhang, Ningyu and Yao, Yunzhi and Tian, Bozhong and Wang, Peng and Deng, Shumin and Wang, Mengru and Xi, Zekun and Mao, Shengyu and Zhang, Jintian and Ni, Yuansheng and others},
  journal={arXiv preprint arXiv:2401.01286},
  year={2024}
}

@article{meng2022locating,
  title={Locating and editing factual associations in gpt},
  author={Meng, Kevin and Bau, David and Andonian, Alex and Belinkov, Yonatan},
  journal={Advances in neural information processing systems},
  volume={35},
  pages={17359--17372},
  year={2022}
}

@inproceedings{mengmass,
  title={Mass-Editing Memory in a Transformer},
  author={Meng, Kevin and Sharma, Arnab Sen and Andonian, Alex J and Belinkov, Yonatan and Bau, David},
  booktitle={The Eleventh International Conference on Learning Representations},
    year={2023}
}

@inproceedings{mitchellfast,
  title={Fast Model Editing at Scale},
  author={Mitchell, Eric and Lin, Charles and Bosselut, Antoine and Finn, Chelsea and Manning, Christopher D},
  booktitle={International Conference on Learning Representations},
  year={2022}
}

@inproceedings{fangalphaedit,
  title={AlphaEdit: Null-Space Constrained Knowledge Editing for Language Models},
  author={Fang, Junfeng and Jiang, Houcheng and Wang, Kun and Ma, Yunshan and Shi, Jie and Wang, Xiang and He, Xiangnan and Chua, Tat-Seng},
  booktitle={The Thirteenth International Conference on Learning Representations},
    year={2025}
}

@inproceedings{
tan2024massive,
title={Massive Editing for Large Language Models via Meta Learning},
author={Chenmien Tan and Ge Zhang and Jie Fu},
booktitle={The Twelfth International Conference on Learning Representations},
year={2024},
url={https://openreview.net/forum?id=L6L1CJQ2PE}
}

@inproceedings{geva2021transformer,
  title={Transformer feed-forward layers are key-value memories},
  author={Geva, Mor and Schuster, Roei and Berant, Jonathan and Levy, Omer},
  booktitle={Proceedings of the 2021 Conference on Empirical Methods in Natural Language Processing},
  pages={5484--5495},
  year={2021}
}

@inproceedings{dai2022knowledge,
  title={Knowledge neurons in pretrained transformers},
  author={Dai, Damai and Dong, Li and Hao, Yaru and Sui, Zhifang and Chang, Baobao and Wei, Furu},
  booktitle={Proceedings of the 60th Annual Meeting of the Association for Computational Linguistics (Volume 1: Long Papers)},
  pages={8493--8502},
  year={2022}
}

@inproceedings{li2024pmet,
  title={PMET: Precise model editing in a transformer},
  author={Li, Xiaopeng and Li, Shasha and Song, Shezheng and Yang, Jing and Ma, Jun and Yu, Jie},
  booktitle={Proceedings of the AAAI Conference on Artificial Intelligence},
  volume={38},
  number={17},
  pages={18564--18572},
  year={2024}
}

@inproceedings{gupta2024unified,
  title={A Unified Framework for Model Editing},
  author={Gupta, Akshat and Sajnani, Dev and Anumanchipalli, Gopala},
  booktitle={Findings of the Association for Computational Linguistics: EMNLP 2024},
  pages={15403--15418},
  year={2024}
}

@inproceedings{de2021editing,
  title={Editing Factual Knowledge in Language Models},
  author={De Cao, Nicola and Aziz, Wilker and Titov, Ivan},
  booktitle={Proceedings of the 2021 Conference on Empirical Methods in Natural Language Processing},
  pages={6491--6506},
  year={2021}
}

@inproceedings{li2025adaedit,
  title={AdaEdit: Advancing Continuous Knowledge Editing For Large Language Models},
  author={Li, Qi and Chu, Xiaowen},
  booktitle={Proceedings of the 63rd Annual Meeting of the Association for Computational Linguistics (Volume 1: Long Papers)},
  pages={4127--4149},
  year={2025}
}

@inproceedings{jianganyedit,
  title={AnyEdit: Edit Any Knowledge Encoded in Language Models},
  author={Jiang, Houcheng and Fang, Junfeng and Zhang, Ningyu and Wan, Mingyang and Ma, Guojun and Wang, Xiang and He, Xiangnan and Chua, Tat-Seng},
  booktitle={Forty-second International Conference on Machine Learning},
  year = {2025}
}

@inproceedings{yan2025keys,
  title={Keys to robust edits: From theoretical insights to practical advances},
  author={Yan, Jianhao and Wang, Futing and Luo, Yun and Li, Yafu and Zhang, Yue},
  booktitle={Proceedings of the 63rd Annual Meeting of the Association for Computational Linguistics (Volume 1: Long Papers)},
  pages={22545--22560},
  year={2025}
}

@inproceedings{gu2024model,
  title={Model editing harms general abilities of large language models: Regularization to the rescue},
  author={Gu, Jia-Chen and Xu, Hao-Xiang and Ma, Jun-Yu and Lu, Pan and Ling, Zhen-Hua and Chang, Kai-Wei and Peng, Nanyun},
  booktitle={Proceedings of the 2024 Conference on Empirical Methods in Natural Language Processing},
  pages={16801--16819},
  year={2024}
}

@article{brown2020language,
  title={Language models are few-shot learners},
  author={Brown, Tom and Mann, Benjamin and Ryder, Nick and Subbiah, Melanie and Kaplan, Jared D and Dhariwal, Prafulla and Neelakantan, Arvind and Shyam, Pranav and Sastry, Girish and Askell, Amanda and others},
  journal={Advances in neural information processing systems},
  volume={33},
  pages={1877--1901},
  year={2020}
}

@article{rae2021scaling,
  title={Scaling language models: Methods, analysis \& insights from training gopher},
  author={Rae, Jack W and Borgeaud, Sebastian and Cai, Trevor and Millican, Katie and Hoffmann, Jordan and Song, Francis and Aslanides, John and Henderson, Sarah and Ring, Roman and Young, Susannah and others},
  journal={arXiv preprint arXiv:2112.11446},
  year={2021}
}

@article{ji2023survey,
  title={Survey of hallucination in natural language generation},
  author={Ji, Ziwei and Lee, Nayeon and Frieske, Rita and Yu, Tiezheng and Su, Dan and Xu, Yan and Ishii, Etsuko and Bang, Ye Jin and Madotto, Andrea and Fung, Pascale},
  journal={ACM computing surveys},
  volume={55},
  number={12},
  pages={1--38},
  year={2023},
  publisher={ACM New York, NY}
}

@article{ferrara2023should,
  title={Should ChatGPT be Biased? Challenges and Risks of Bias in Large Language Models},
  author={Ferrara, Emilio},
  journal={Challenges and Risks of Bias in Large Language Models (October 26, 2023)},
  year={2023}
}

@article{hartvigsen2023aging,
  title={Aging with grace: Lifelong model editing with discrete key-value adaptors},
  author={Hartvigsen, Tom and Sankaranarayanan, Swami and Palangi, Hamid and Kim, Yoon and Ghassemi, Marzyeh},
  journal={Advances in Neural Information Processing Systems},
  volume={36},
  pages={47934--47959},
  year={2023}
}

@article{kohonen1972correlation,
  title={Correlation matrix memories},
  author={Kohonen, Teuvo},
  journal={IEEE transactions on computers},
  volume={100},
  number={4},
  pages={353--359},
  year={1972},
  publisher={IEEE}
}

@article{anderson1972simple,
  title={A simple neural network generating an interactive memory},
  author={Anderson, James A},
  journal={Mathematical biosciences},
  volume={14},
  number={3-4},
  pages={197--220},
  year={1972},
  publisher={Elsevier}
}

@inproceedings{levy-etal-2017-zero,
    title = "Zero-Shot Relation Extraction via Reading Comprehension",
    author = "Levy, Omer  and
      Seo, Minjoon  and
      Choi, Eunsol  and
      Zettlemoyer, Luke",
    editor = "Levy, Roger  and
      Specia, Lucia",
    booktitle = "Proceedings of the 21st Conference on Computational Natural Language Learning ({C}o{NLL} 2017)",
    month = aug,
    year = "2017",
    address = "Vancouver, Canada",
    publisher = "Association for Computational Linguistics",
    url = "https://aclanthology.org/K17-1034/",
    doi = "10.18653/v1/K17-1034",
    pages = "333--342"
}

@article{radford2019language,
    title   = {Language models are unsupervised multitask learners},
    author  = {Radford, Alec and Wu, Jeffrey and Child, Rewon and Luan, David and Amodei, Dario and Sutskever, Ilya and others},
    journal = {OpenAI blog},
    volume  = {1},
    number  = {8},
    pages   = {9},
    year    = {2019},
}

@misc{llama3modelcard,
  author = {{Meta}},
  title = {{Llama 3}},
  year = {2024},
  url = {https://llama.meta.com/llama3/},
  urldate = {2024-05-21}, 
  note = {Large language model release}
}

@inproceedings{
ma2025perturbationrestrained,
title={Perturbation-Restrained Sequential Model Editing},
author={Jun-Yu Ma and Hong Wang and Hao-Xiang Xu and Zhen-Hua Ling and Jia-Chen Gu},
booktitle={The Thirteenth International Conference on Learning Representations},
year={2025},
url={https://openreview.net/forum?id=bfI8cp8qmk}
}

@inproceedings{
li2025rethinking,
title={Rethinking Residual Distribution in Locate-then-Edit Model Editing},
author={Xiaopeng Li and Shangwen Wang and Shasha Li and Shezheng Song and Bin Ji and Ma Jun and Jie Yu},
booktitle={The Thirty-ninth Annual Conference on Neural Information Processing Systems},
year={2025},
url={https://openreview.net/forum?id=P9gY05BDkW}
}

@misc{gpt-j,
    author       = {Wang, Ben and Komatsuzaki, Aran},
    title        = {{GPT-J-6B: A 6 billion parameter autoregressive language model}},
    year         = {2021},
    month        = {May},
}

@inproceedings{wang2018glue,
  title={GLUE: A multi-task benchmark and analysis platform for natural language understanding},
  author={Wang, Alex and Singh, Amanpreet and Michael, Julian and Hill, Felix and Levy, Omer and Bowman, Samuel},
  booktitle={Proceedings of the 2018 EMNLP workshop BlackboxNLP: Analyzing and interpreting neural networks for NLP},
  pages={353--355},
  year={2018}
}

@inproceedings{socher2013recursive,
  title={Recursive deep models for semantic compositionality over a sentiment treebank},
  author={Socher, Richard and Perelygin, Alex and Wu, Jean and Chuang, Jason and Manning, Christopher D and Ng, Andrew Y and Potts, Christopher},
  booktitle={Proceedings of the 2013 conference on empirical methods in natural language processing},
  pages={1631--1642},
  year={2013}
}

@inproceedings{dolan2005automatically,
  title={Automatically constructing a corpus of sentential paraphrases},
  author={Dolan, William B and Brockett, Chris},
  booktitle={Proceedings of the third international workshop on paraphrasing (IWP2005)},
  year={2005}
}

@article{warstadt2019neural,
  title={Neural network acceptability judgments},
  author={Warstadt, Alex and Singh, Amanpreet and Bowman, Samuel R},
  journal={Transactions of the Association for Computational Linguistics},
  volume={7},
  pages={625--641},
  year={2019}
}

@article{bentivogli2009fifth,
  title={The Fifth PASCAL Recognizing Textual Entailment Challenge.},
  author={Bentivogli, Luisa and Clark, Peter and Dagan, Ido and Giampiccolo, Danilo},
  journal={TAC},
  volume={7},
  number={8},
  pages={1},
  year={2009}
}

@inproceedings{williams2018broad,
  title={A broad-coverage challenge corpus for sentence understanding through inference},
  author={Williams, Adina and Nangia, Nikita and Bowman, Samuel},
  booktitle={Proceedings of the 2018 conference of the North American chapter of the association for computational linguistics: human language technologies, volume 1 (long papers)},
  pages={1112--1122},
  year={2018}
}

@inproceedings{
hendrycks2021measuring,
title={Measuring Massive Multitask Language Understanding},
author={Dan Hendrycks and Collin Burns and Steven Basart and Andy Zou and Mantas Mazeika and Dawn Song and Jacob Steinhardt},
booktitle={International Conference on Learning Representations},
year={2021},
url={https://openreview.net/forum?id=d7KBjmI3GmQ}
}

@inproceedings{park-etal-2025-context,
    title = "Context-Robust Knowledge Editing for Language Models",
    author = "Park, Haewon  and
      Choi, Gyubin  and
      Kim, Minjun  and
      Jo, Yohan",
    editor = "Che, Wanxiang  and
      Nabende, Joyce  and
      Shutova, Ekaterina  and
      Pilehvar, Mohammad Taher",
    booktitle = "Findings of the Association for Computational Linguistics: ACL 2025",
    month = jul,
    year = "2025",
    address = "Vienna, Austria",
    publisher = "Association for Computational Linguistics",
    url = "https://aclanthology.org/2025.findings-acl.540/",
    doi = "10.18653/v1/2025.findings-acl.540",
    pages = "10360--10385",
    ISBN = "979-8-89176-256-5"
}

@article{maaten2008visualizing,
  title={Visualizing data using t-SNE},
  author={Maaten, Laurens van der and Hinton, Geoffrey},
  journal={Journal of machine learning research},
  volume={9},
  pages={2579--2605},
  year={2008},
month = nov
}

@inproceedings{
liu2026energyregularized,
title={Energy-Regularized Sequential Model Editing on Hyperspheres},
author={Qingyuan Liu and Jia-Chen Gu and Yunzhi Yao and Hong Wang and Nanyun Peng},
booktitle={The Fourteenth International Conference on Learning Representations},
year={2026},
url={https://openreview.net/forum?id=CHsdtzCip6}
}

@article{li2023textbooks,
  title={Textbooks are all you need ii: phi-1.5 technical report},
  author={Li, Yuanzhi and Bubeck, S{\'e}bastien and Eldan, Ronen and Del Giorno, Allie and Gunasekar, Suriya and Lee, Yin Tat},
  journal={arXiv preprint arXiv:2309.05463},
  year={2023}
}

@misc{qwen3technicalreport,
      title={Qwen3 Technical Report}, 
      author={Qwen Team},
      year={2025},
      eprint={2505.09388},
      archivePrefix={arXiv},
      primaryClass={cs.CL},
      url={https://arxiv.org/abs/2505.09388}, 
}

@inproceedings{mitchell2022memory,
  title={Memory-based model editing at scale},
  author={Mitchell, Eric and Lin, Charles and Bosselut, Antoine and Manning, Christopher D and Finn, Chelsea},
  booktitle={International Conference on Machine Learning},
  pages={15817--15831},
  year={2022},
  organization={PMLR}
}

@inproceedings{liska2022streamingqa,
  title={StreamingQA: A benchmark for adaptation to new knowledge over time in question answering models},
  author={Liska, Adam and Kocisky, Tomas and Gribovskaya, Elena and Terzi, Tayfun and Sezener, Eren and Agrawal, Devang and D’Autume, Cyprien De Masson and Scholtes, Tim and Zaheer, Manzil and Young, Susannah and others},
  booktitle={International Conference on Machine Learning},
  pages={13604--13622},
  year={2022},
  organization={PMLR}
}

@article{lazaridou2021mind,
  title={Mind the gap: Assessing temporal generalization in neural language models},
  author={Lazaridou, Angeliki and Kuncoro, Adhi and Gribovskaya, Elena and Agrawal, Devang and Liska, Adam and Terzi, Tayfun and Gimenez, Mai and de Masson d'Autume, Cyprien and Kocisky, Tomas and Ruder, Sebastian and others},
  journal={Advances in Neural Information Processing Systems},
  volume={34},
  pages={29348--29363},
  year={2021}
}

@inproceedings{jeong-etal-2025-steam,
    title = "{STEAM}: A Semantic-Level Knowledge Editing Framework for Large Language Models",
    author = "Jeong, Geunyeong  and
      Sun, Juoh  and
      Lee, Seonghee  and
      Kim, Harksoo",
    editor = "Christodoulopoulos, Christos  and
      Chakraborty, Tanmoy  and
      Rose, Carolyn  and
      Peng, Violet",
    booktitle = "Findings of the Association for Computational Linguistics: EMNLP 2025",
    month = nov,
    year = "2025",
    address = "Suzhou, China",
    publisher = "Association for Computational Linguistics",
    url = "https://aclanthology.org/2025.findings-emnlp.585/",
    doi = "10.18653/v1/2025.findings-emnlp.585",
    pages = "11008--11023",
    ISBN = "979-8-89176-335-7"
}

@inproceedings{wei2025setke,
  title={SetKE: knowledge editing for knowledge elements overlap},
  author={Wei, Yifan and Yu, Xiaoyan and Song, Ran and Peng, Hao and Li, Angsheng},
  booktitle={Proceedings of the Thirty-Fourth International Joint Conference on Artificial Intelligence},
  pages={8295--8303},
  year={2025}
}

@inproceedings{
fei2026scaling,
title={Scaling Knowledge Editing in {LLM}s to 100,000 Facts with Neural {KV} Database},
author={Weizhi Fei and Hao Shi and Jing Xu and Jingchen Peng and Jiazheng Li and Jingzhao Zhang and Bo Bai and Wei Han and Zhenyuan Chen and Xueyan Niu},
booktitle={The Fourteenth International Conference on Learning Representations},
year={2026},
url={https://openreview.net/forum?id=Z0CX62CSJQ}
}

@inproceedings{
wang2025revealing,
title={Revealing and Mitigating Over-Attention in Knowledge Editing},
author={Pinzheng Wang and Zecheng Tang and Keyan Zhou and Juntao Li and Qiaoming Zhu and Min Zhang},
booktitle={The Thirteenth International Conference on Learning Representations},
year={2025},
url={https://openreview.net/forum?id=4l3AH8Bhmt}
}

@inproceedings{gupta-etal-2025-lifelong,
    title = "Lifelong Knowledge Editing requires Better Regularization",
    author = "Gupta, Akshat  and
      Prateepamornkul, Phudish  and
      Lu, Maochuan  and
      Alaa, Ahmed  and
      Hartvigsen, Thomas  and
      Anumanchipalli, Gopala",
    editor = "Christodoulopoulos, Christos  and
      Chakraborty, Tanmoy  and
      Rose, Carolyn  and
      Peng, Violet",
    booktitle = "Findings of the Association for Computational Linguistics: EMNLP 2025",
    month = nov,
    year = "2025",
    address = "Suzhou, China",
    publisher = "Association for Computational Linguistics",
    url = "https://aclanthology.org/2025.findings-emnlp.1234/",
    doi = "10.18653/v1/2025.findings-emnlp.1234",
    pages = "22653--22675",
    ISBN = "979-8-89176-335-7"
}

@inproceedings{baghel2025resolving,
  title={Resolving UnderEdit \& OverEdit with Iterative \& Neighbor-Assisted Model Editing},
  author={Baghel, Bhiman Kumar and Jordan, Emma and Shi, Zheyuan Ryan and Li, Xiang Lorraine},
  booktitle={Findings of the Association for Computational Linguistics: EMNLP 2025},
  pages={14786--14808},
  year={2025}
}

@inproceedings{zhao2025fedleke,
  title={FedLEKE: Federated Locate-then-Edit Knowledge Editing for Multi-Client Collaboration},
  author={Zhao, Zongkai and Xu, Guozeng and Li, Xiuhua and Wei, Kaiwen and Zhong, Jiang},
  booktitle={Findings of the Association for Computational Linguistics: ACL 2025},
  pages={14247--14258},
  year={2025}
}

@inproceedings{wang-etal-2025-lyaplock,
    title = "{L}yap{L}ock: Bounded Knowledge Preservation in Sequential Large Language Model Editing",
    author = "Wang, Peng  and
      Zhou, Biyu  and
      Tang, Xuehai  and
      Han, Jizhong  and
      Hu, Songlin",
    editor = "Christodoulopoulos, Christos  and
      Chakraborty, Tanmoy  and
      Rose, Carolyn  and
      Peng, Violet",
    booktitle = "Proceedings of the 2025 Conference on Empirical Methods in Natural Language Processing",
    month = nov,
    year = "2025",
    address = "Suzhou, China",
    publisher = "Association for Computational Linguistics",
    url = "https://aclanthology.org/2025.emnlp-main.327/",
    doi = "10.18653/v1/2025.emnlp-main.327",
    pages = "6434--6459",
    ISBN = "979-8-89176-332-6"
}

@inproceedings{zhang2024instructedit,
  title={InstructEdit: instruction-based knowledge editing for large language models},
  author={Zhang, Ningyu and Tian, Bozhong and Cheng, Siyuan and Liang, Xiaozhuan and Hu, Yi and Xue, Kouying and Gou, Yanjie and Chen, Xi and Chen, Huajun},
  booktitle={Proceedings of the Thirty-Third International Joint Conference on Artificial Intelligence},
  pages={6633--6641},
  year={2024}
}

@inproceedings{li2025reinforced,
  title={Reinforced Lifelong Editing for Language Models},
  author={Li, Zherui and Jiang, Houcheng and Chen, Hao and Bi, Baolong and Zhou, Zhenhong and Sun, Fei and Fang, Junfeng and Wang, Xiang},
  booktitle={International Conference on Machine Learning},
  pages={34920--34942},
  year={2025},
  organization={PMLR}
}

@inproceedings{zhang2024dafnet,
  title={Dafnet: Dynamic auxiliary fusion for sequential model editing in large language models},
  author={Zhang, Taolin and Chen, Qizhou and Li, Dongyang and Wang, Chengyu and He, Xiaofeng and Huang, Longtao and Huang, Jun and others},
  booktitle={Findings of the Association for Computational Linguistics: ACL 2024},
  pages={1588--1602},
  year={2024}
}
}


\clearpage
\appendix

\section{Theoretical Analysis}
\label{app:theory}
\subsection{Proof of Proposition~\ref{prop:target_attenuation}}
\label{app:spectral_proofs}
Let
\(\boldsymbol{K}_1\) denote the keys of the current to-be-updated knowledge,
\(\boldsymbol{K}_0\) the keys of preserved knowledge, and
\(\boldsymbol{K}_p\) the keys of previously updated knowledge. For a
current edit, let \(\boldsymbol{k}_1\in\boldsymbol{K}_1\) be the current edit
key and let \(\boldsymbol{r}\) be the planned residual induced by Stage~I. \(\boldsymbol{r}=\boldsymbol{v}^{*}-\boldsymbol{W}\boldsymbol{k}_1\),
or, under multi-layer residual allocation,
\(\boldsymbol{r}\) denotes the effective residual assigned to the current
edited layer.

\paragraph{Scope of the surrogate.}
The quadratic surrogate fixes current keys, planned residual, and
constraint statistics, and analyzes the local downstream realization step
rather than the full nonlinear Transformer. Nonlinearities affect the result
through locally evaluated keys and higher-order remainders, so
Proposition~\ref{prop:target_attenuation} should be read as a first-order
local characterization. When hidden states and keys vary smoothly with the
update, the deviation from the local surrogate can be written as a nonlinear
residual \(\boldsymbol{e}_{\mathrm{nl}}\), with
\(\|\boldsymbol{e}_{\mathrm{nl}}\|=O(\|\boldsymbol{\Delta}\|_F^2)\) in a
sufficiently small neighborhood. Thus, \(\beta=\rho/(1+\rho)\)
captures the leading-order constrained realization, while activation
nonlinearities such as SwiGLU are absorbed into the locally evaluated keys
and the higher-order residual term.

\paragraph{MEMIT covariance constraints.}
For a current edit, consider the MEMIT local quadratic surrogate
\begin{equation}
\boldsymbol{\Delta}^{*}
=
\operatorname*{arg\,min}_{\boldsymbol{\Delta}}
\;
\|
\boldsymbol{\Delta}\boldsymbol{k}_1-\boldsymbol{r}
\|_2^2
+
\operatorname{tr}
(
\boldsymbol{\Delta}
\boldsymbol{C}_0
\boldsymbol{\Delta}^{T}
)
+
\|
\boldsymbol{\Delta}\boldsymbol{K}_p
\|_F^2 ,
\end{equation}
where \(\boldsymbol{C}_0\) is the covariance surrogate of preserved knowledge.
Since
\begin{equation}
\|
\boldsymbol{\Delta}\boldsymbol{K}_p
\|_F^2
=
\operatorname{tr}
(
\boldsymbol{\Delta}
\boldsymbol{K}_p\boldsymbol{K}_p^{T}
\boldsymbol{\Delta}^{T}
),
\end{equation}
define
\begin{equation}
\boldsymbol{A}_{\mathrm{M}}
:=
\boldsymbol{C}_0
+
\boldsymbol{K}_p\boldsymbol{K}_p^{T}.
\end{equation}
Assume that \(\boldsymbol{A}_{\mathrm{M}}\) is nonsingular on the considered
subspace. The objective becomes
\begin{equation}
\boldsymbol{\Delta}^{*}
=
\operatorname*{arg\,min}_{\boldsymbol{\Delta}}
\;
\|
\boldsymbol{\Delta}\boldsymbol{k}_1-\boldsymbol{r}
\|_2^2
+
\operatorname{tr}
(
\boldsymbol{\Delta}
\boldsymbol{A}_{\mathrm{M}}
\boldsymbol{\Delta}^{T}
).
\end{equation}
The first-order condition is
\begin{equation}
(
\boldsymbol{\Delta}\boldsymbol{k}_1-\boldsymbol{r}
)
\boldsymbol{k}_1^{T}
+
\boldsymbol{\Delta}
\boldsymbol{A}_{\mathrm{M}}
=
\boldsymbol{0}.
\end{equation}
Therefore,
\begin{equation}
\boldsymbol{\Delta}^{*}
=
\boldsymbol{r}\boldsymbol{k}_1^{T}
\left(
\boldsymbol{A}_{\mathrm{M}}
+
\boldsymbol{k}_1\boldsymbol{k}_1^{T}
\right)^{-1}.
\end{equation}
The realized residual is
\begin{equation}
\boldsymbol{\delta}^{\mathrm{real}}
=
\boldsymbol{\Delta}^{*}\boldsymbol{k}_1
=
\boldsymbol{r}\,
\boldsymbol{k}_1^{T}
\left(
\boldsymbol{A}_{\mathrm{M}}
+
\boldsymbol{k}_1\boldsymbol{k}_1^{T}
\right)^{-1}
\boldsymbol{k}_1 .
\end{equation}
Define
\begin{equation}
\rho_{\mathrm{M}}
=
\boldsymbol{k}_1^{T}
\boldsymbol{A}_{\mathrm{M}}^{-1}
\boldsymbol{k}_1 .
\end{equation}
By the Sherman--Morrison identity,
\begin{equation}
\boldsymbol{k}_1^{T}
\left(
\boldsymbol{A}_{\mathrm{M}}
+
\boldsymbol{k}_1\boldsymbol{k}_1^{T}
\right)^{-1}
\boldsymbol{k}_1
=
\frac{
\rho_{\mathrm{M}}
}{
1+
\rho_{\mathrm{M}}
}.
\end{equation}
Thus,
\begin{equation}
\boldsymbol{\delta}^{\mathrm{real}}
=
\frac{
\rho_{\mathrm{M}}
}{
1+
\rho_{\mathrm{M}}
}
\boldsymbol{r}.
\end{equation}

\paragraph{AlphaEdit projected constraints.}
For AlphaEdit, preserved knowledge is handled by the null-space projector
\(\boldsymbol{P}\), which satisfies
\(\boldsymbol{P}^{2}=\boldsymbol{P}\) and
\(\boldsymbol{P}^{T}=\boldsymbol{P}\). Let
\(\boldsymbol{B}\) denote the AlphaEdit update applied to the edited weights,
i.e.,
\[
\boldsymbol{B}:=\boldsymbol{\Delta}_{\mathrm{AE}} .
\]
Following the projected form of AlphaEdit, this applied update lies in the
right-projected subspace, so
\[
\boldsymbol{B}=\boldsymbol{B}\boldsymbol{P}.
\]
This property implies
\(\boldsymbol{B}\boldsymbol{K}_0=\boldsymbol{0}\), and therefore the preserved
keys are unchanged by the applied perturbation.

Define the projected current-edit key as
\[
\boldsymbol{q}_1:=\boldsymbol{P}\boldsymbol{k}_1 .
\]
Since \(\boldsymbol{B}=\boldsymbol{B}\boldsymbol{P}\), we have
\[
\boldsymbol{B}\boldsymbol{k}_1
=
\boldsymbol{B}\boldsymbol{P}\boldsymbol{k}_1
=
\boldsymbol{B}\boldsymbol{q}_1 .
\]
The single-request local objective can therefore be written over the applied
\begin{equation}
\boldsymbol{B}^{*}
=
\operatorname*{arg\,min}_{\boldsymbol{B}=\boldsymbol{B}\boldsymbol{P}}
\;
\|
\boldsymbol{B}\boldsymbol{q}_1-\boldsymbol{r}
\|_2^2
+
\|
\boldsymbol{B}
\|_F^2
+
\|
\boldsymbol{B}\boldsymbol{K}_p
\|_F^2 .
\end{equation}
\begin{equation}
\|
\boldsymbol{B}\boldsymbol{K}_p
\|_F^2
=
\operatorname{tr}
\left(
\boldsymbol{B}
\boldsymbol{P}\boldsymbol{K}_p\boldsymbol{K}_p^{T}\boldsymbol{P}
\boldsymbol{B}^{T}
\right).
\end{equation}
Define
\begin{equation}
\boldsymbol{A}_{\mathrm{AE}}
:=
\boldsymbol{I}
+
\boldsymbol{P}\boldsymbol{K}_p\boldsymbol{K}_p^{T}\boldsymbol{P}.
\end{equation}
Then
\begin{equation}
\boldsymbol{B}^{*}
=
\operatorname*{arg\,min}_{\boldsymbol{B}=\boldsymbol{B}\boldsymbol{P}}
\;
\|
\boldsymbol{B}\boldsymbol{q}_1-\boldsymbol{r}
\|_2^2
+
\operatorname{tr}
(
\boldsymbol{B}
\boldsymbol{A}_{\mathrm{AE}}
\boldsymbol{B}^{T}
).
\end{equation}
The first-order condition in the projected subspace is
\begin{equation}
(
\boldsymbol{B}\boldsymbol{q}_1-\boldsymbol{r}
)
\boldsymbol{q}_1^{T}
+
\boldsymbol{B}
\boldsymbol{A}_{\mathrm{AE}}
=
\boldsymbol{0}.
\end{equation}
Hence,
\begin{equation}
\boldsymbol{B}^{*}
=
\boldsymbol{r}\boldsymbol{q}_1^{T}
\left(
\boldsymbol{A}_{\mathrm{AE}}
+
\boldsymbol{q}_1\boldsymbol{q}_1^{T}
\right)^{-1}.
\end{equation}
The realized residual on the original current edit key is
\begin{equation}
\boldsymbol{\delta}^{\mathrm{real}}
=
\boldsymbol{B}^{*}\boldsymbol{k}_1
=
\boldsymbol{B}^{*}\boldsymbol{P}\boldsymbol{k}_1
=
\boldsymbol{B}^{*}\boldsymbol{q}_1 .
\end{equation}
Therefore,
\begin{equation}
\boldsymbol{\delta}^{\mathrm{real}}
=
\boldsymbol{r}\,
\boldsymbol{q}_1^{T}
\left(
\boldsymbol{A}_{\mathrm{AE}}
+
\boldsymbol{q}_1\boldsymbol{q}_1^{T}
\right)^{-1}
\boldsymbol{q}_1 .
\end{equation}
Define
\begin{equation}
\rho_{\mathrm{AE}}
=
\boldsymbol{q}_1^{T}
\boldsymbol{A}_{\mathrm{AE}}^{-1}
\boldsymbol{q}_1
=
(\boldsymbol{P}\boldsymbol{k}_1)^{T}
\left(
\boldsymbol{I}
+
\boldsymbol{P}\boldsymbol{K}_p\boldsymbol{K}_p^{T}\boldsymbol{P}
\right)^{-1}
\boldsymbol{P}\boldsymbol{k}_1 .
\end{equation}
By the Sherman--Morrison identity,
\begin{equation}
\boldsymbol{q}_1^{T}
\left(
\boldsymbol{A}_{\mathrm{AE}}
+
\boldsymbol{q}_1\boldsymbol{q}_1^{T}
\right)^{-1}
\boldsymbol{q}_1
=
\frac{
\rho_{\mathrm{AE}}
}{
1+
\rho_{\mathrm{AE}}
}.
\end{equation}
Thus,
\begin{equation}
\boldsymbol{\delta}^{\mathrm{real}}
=
\frac{
\rho_{\mathrm{AE}}
}{
1+
\rho_{\mathrm{AE}}
}
\boldsymbol{r}.
\end{equation}

\paragraph{Remaining unrealized residual.}
Combining MEMIT and AlphaEdit derivations, both single-request
updates satisfy
\begin{equation}
\boldsymbol{\delta}^{\mathrm{real}}
=
\beta\boldsymbol{r},
\qquad
\beta=\frac{\rho}{1+\rho}\in[0,1).
\end{equation}
Consequently, the unrealized residual is
\begin{equation}
\boldsymbol{r}^{\mathrm{rem}}
:=
\boldsymbol{r}
-
\boldsymbol{\delta}^{\mathrm{real}}
=
(1-\beta)\boldsymbol{r}
=
\frac{1}{1+\rho}\boldsymbol{r}.
\end{equation}

\paragraph{Batch-edit extension.}
For MEMIT-style constraints, let
\(\boldsymbol{K}_1=[\boldsymbol{k}_{1,1},\ldots,\boldsymbol{k}_{1,b}]\) and
\(\boldsymbol{R}\) collect the planned residuals of the current batch. The
batch objective is
\begin{equation}
\boldsymbol{\Delta}^{*}
=
\operatorname*{arg\,min}_{\boldsymbol{\Delta}}
\;
\|
\boldsymbol{\Delta}\boldsymbol{K}_1-\boldsymbol{R}
\|_F^2
+
\operatorname{tr}
(
\boldsymbol{\Delta}
\boldsymbol{A}_{\mathrm{M}}
\boldsymbol{\Delta}^{T}
).
\end{equation}
The solution is
\begin{equation}
\boldsymbol{\Delta}^{*}
=
\boldsymbol{R}
\boldsymbol{K}_1^{T}
\left(
\boldsymbol{A}_{\mathrm{M}}
+
\boldsymbol{K}_1\boldsymbol{K}_1^{T}
\right)^{-1}.
\end{equation}
Thus the realized batch residual is
\begin{equation}
\boldsymbol{\Delta}^{*}\boldsymbol{K}_1
=
\boldsymbol{R}
\boldsymbol{T}_{\mathrm{M}},
\end{equation}
where
\begin{equation}
\boldsymbol{T}_{\mathrm{M}}
=
\boldsymbol{K}_1^{T}
\left(
\boldsymbol{A}_{\mathrm{M}}
+
\boldsymbol{K}_1\boldsymbol{K}_1^{T}
\right)^{-1}
\boldsymbol{K}_1 .
\end{equation}
Let
\begin{equation}
\boldsymbol{S}_{\mathrm{M}}
=
\boldsymbol{K}_1^{T}
\boldsymbol{A}_{\mathrm{M}}^{-1}
\boldsymbol{K}_1 .
\end{equation}
By the Woodbury identity,
\begin{equation}
\boldsymbol{T}_{\mathrm{M}}
=
\boldsymbol{S}_{\mathrm{M}}
(
\boldsymbol{I}
+
\boldsymbol{S}_{\mathrm{M}}
)^{-1}.
\end{equation}

For AlphaEdit-style constraints, define
\(\boldsymbol{Q}_1:=\boldsymbol{P}\boldsymbol{K}_1\). The batch projected
solution is
\begin{equation}
\boldsymbol{B}^{*}
=
\boldsymbol{R}
\boldsymbol{Q}_1^{T}
\left(
\boldsymbol{A}_{\mathrm{AE}}
+
\boldsymbol{Q}_1\boldsymbol{Q}_1^{T}
\right)^{-1}.
\end{equation}
The realized batch residual is
\begin{equation}
\boldsymbol{B}^{*}\boldsymbol{K}_1
=
\boldsymbol{B}^{*}\boldsymbol{Q}_1
=
\boldsymbol{R}
\boldsymbol{T}_{\mathrm{AE}},
\end{equation}
where
\begin{equation}
\boldsymbol{T}_{\mathrm{AE}}
=
\boldsymbol{Q}_1^{T}
\left(
\boldsymbol{A}_{\mathrm{AE}}
+
\boldsymbol{Q}_1\boldsymbol{Q}_1^{T}
\right)^{-1}
\boldsymbol{Q}_1 .
\end{equation}
Let
\begin{equation}
\boldsymbol{S}_{\mathrm{AE}}
=
\boldsymbol{Q}_1^{T}
\boldsymbol{A}_{\mathrm{AE}}^{-1}
\boldsymbol{Q}_1 .
\end{equation}
By the Woodbury identity,
\begin{equation}
\boldsymbol{T}_{\mathrm{AE}}
=
\boldsymbol{S}_{\mathrm{AE}}
(
\boldsymbol{I}
+
\boldsymbol{S}_{\mathrm{AE}}
)^{-1}.
\end{equation}
Thus, batch editing realizes planned residuals through a request-interaction
matrix rather than independent scalar factors. When the batch contains a single
current edit, the matrix form reduces to the scalar attenuation factor in
Proposition~\ref{prop:target_attenuation}. \qed

\subsection{Proof of Theorem~\ref{thm:regularization_disconnect}}
\label{app:static_trap_proof}

By Proposition~\ref{prop:target_attenuation}, the realized residual of request
\(i\) satisfies
\begin{equation}
\boldsymbol{\delta}^{\mathrm{real}}_i
=
\beta_i\boldsymbol{r}_i,
\qquad
\beta_i=\frac{\rho_i}{1+\rho_i}.
\end{equation}
Since \(\rho_i\ge 0\), we have \(\beta_i\in[0,1)\). If \(\beta_i=0\) and
\(m>0\), then
\(\|\boldsymbol{\delta}^{\mathrm{real}}_i\|=0<m\) for any finite
\(\|\boldsymbol{r}_i\|\), so semantic success is impossible. We therefore
consider the non-degenerate case \(\beta_i\in(0,1)\).

Semantic success requires
\begin{equation}
\|\boldsymbol{\delta}^{\mathrm{real}}_i\|\ge m.
\end{equation}
Substituting
\(\boldsymbol{\delta}^{\mathrm{real}}_i=\beta_i\boldsymbol{r}_i\) gives
\begin{equation}
\beta_i\|\boldsymbol{r}_i\|\ge m,
\end{equation}
and hence
\begin{equation}
\|\boldsymbol{r}_i\|
\ge
\frac{m}{\beta_i}.
\end{equation}
This gives the lower bound on the planned residual required for semantic
success after downstream realization.

Downstream editability requires
\begin{equation}
\|\boldsymbol{r}_i-\boldsymbol{\delta}^{\mathrm{real}}_i\|\le\tau.
\end{equation}
Using
\(\boldsymbol{\delta}^{\mathrm{real}}_i=\beta_i\boldsymbol{r}_i\), we have
\begin{equation}
\boldsymbol{r}_i-\boldsymbol{\delta}^{\mathrm{real}}_i
=
(1-\beta_i)\boldsymbol{r}_i .
\end{equation}
Therefore,
\begin{equation}
\|\boldsymbol{r}_i-\boldsymbol{\delta}^{\mathrm{real}}_i\|
=
(1-\beta_i)\|\boldsymbol{r}_i\|.
\end{equation}
The downstream editability condition then implies
\begin{equation}
(1-\beta_i)\|\boldsymbol{r}_i\|\le\tau,
\end{equation}
and since \(\beta_i<1\),
\begin{equation}
\|\boldsymbol{r}_i\|
\le
\frac{\tau}{1-\beta_i}.
\end{equation}
This gives the upper bound on the planned residual imposed by downstream
editability.

Combining the two requirements, any feasible planned residual must satisfy
\begin{equation}
\frac{m}{\beta_i}
\le
\|\boldsymbol{r}_i\|
\le
\frac{\tau}{1-\beta_i}.
\end{equation}
Thus, the feasible scale of \(\boldsymbol{r}_i\) depends on
\(\beta_i\), and hence on the downstream association
\(\rho_i\) through \(\beta_i=\rho_i/(1+\rho_i)\).

It remains to derive the infeasibility condition. The interval above is
non-empty only if
\begin{equation}
\frac{m}{\beta_i}
\le
\frac{\tau}{1-\beta_i}.
\end{equation}
Since \(\beta_i\in(0,1)\), multiplying both sides by
\(\beta_i(1-\beta_i)>0\) gives
\begin{equation}
m(1-\beta_i)\le \tau\beta_i.
\end{equation}
Equivalently,
\begin{equation}
m\le (m+\tau)\beta_i,
\end{equation}
which yields
\begin{equation}
\beta_i\ge \frac{m}{m+\tau}.
\end{equation}
Therefore, if
\begin{equation}
\beta_i<\frac{m}{m+\tau},
\end{equation}
no finite choice of \(\|\boldsymbol{r}_i\|\) can satisfy both semantic success
and downstream editability under this local model. This completes the proof.
\qed

\paragraph{Implication for shared KL regularization.}
If the shared upstream KL coefficient locally induces
\(\|\boldsymbol{r}_i(\omega)\|\approx a_i/\omega\), with \(a_i>0\) and
\(\omega>0\), then Theorem~\ref{thm:regularization_disconnect} implies
\begin{equation}
\frac{m}{\beta_i}
\le
\frac{a_i}{\omega}
\le
\frac{\tau}{1-\beta_i}.
\end{equation}
Equivalently, the admissible range of \(\omega\) for request \(i\) is
\begin{equation}
\Omega_i
=
\left[
L_i,U_i
\right]
=
\left[
\frac{a_i(1-\beta_i)}{\tau},
\frac{a_i\beta_i}{m}
\right].
\end{equation}
If \(\omega>U_i\), then
\begin{equation}
\|\boldsymbol{r}_i(\omega)\|
<
\frac{m}{\beta_i},
\end{equation}
so the realized residual satisfies
\begin{equation}
\|\boldsymbol{\delta}^{\mathrm{real}}_i\|
=
\beta_i\|\boldsymbol{r}_i(\omega)\|
<
m.
\end{equation}
Thus the planned residual is over-shrunk and semantic editing becomes
insufficient.

If \(\omega<L_i\), then
\begin{equation}
\|\boldsymbol{r}_i(\omega)\|
>
\frac{\tau}{1-\beta_i}.
\end{equation}
In this case, the upstream stage produces an over-large planned residual
relative to the downstream realization factor. Since
\begin{equation}
\boldsymbol{r}_i-\boldsymbol{\delta}^{\mathrm{real}}_i
=
(1-\beta_i)\boldsymbol{r}_i,
\end{equation}
we obtain
\begin{equation}
\|\boldsymbol{r}_i-\boldsymbol{\delta}^{\mathrm{real}}_i\|
=
(1-\beta_i)\|\boldsymbol{r}_i(\omega)\|
>
\tau.
\end{equation}
Therefore, under-regularization first manifests as an over-large planned
residual, and this over-large planned residual reduces downstream editability by
leaving an unrealized residual that exceeds the editability tolerance.

Now consider a high-association request \(h\) and a low-association request
\(\ell\), with \(\beta_h>\beta_\ell\). Their admissible coefficient intervals are
\begin{equation}
\Omega_h
=
\left[
\frac{a_h(1-\beta_h)}{\tau},
\frac{a_h\beta_h}{m}
\right],
\qquad
\Omega_\ell
=
\left[
\frac{a_\ell(1-\beta_\ell)}{\tau},
\frac{a_\ell\beta_\ell}{m}
\right].
\end{equation}
A single shared coefficient cannot adapt these intervals per request. In
particular, whenever
\begin{equation}
U_h<L_\ell,
\qquad\text{i.e.,}\qquad
\frac{a_h\beta_h}{m}
<
\frac{a_\ell(1-\beta_\ell)}{\tau},
\end{equation}
any shared coefficient
\begin{equation}
\omega\in(U_h,L_\ell)
\end{equation}
simultaneously over-shrinks the high-association request and
under-regularizes the low-association request. The former yields insufficient
realized editing, while the latter produces an over-large planned residual whose
unrealized component violates the downstream editability tolerance.

\subsection{A Conditional Necessity Analysis of Downstream-Aware Target Optimization}
\label{app:bilevel_proofs}

We provide a first-order necessity analysis for downstream-aware target
optimization under the bi-level formulation in Sec.~\ref{sec:method:setup}.
The goal is not to claim that the upstream and downstream stages can never
be decoupled. Instead, we show that such a decoupling is first-order valid
only in special cases where the downstream-realization coupling term
vanishes. In the generic case, optimizing the upstream target as an isolated
semantic objective misses a term induced by the downstream solver.

For compactness, let \(\boldsymbol{\delta}\) denote the active coordinates of
the applied multi-layer downstream update, defined below. We write the inner
downstream problem as
\begin{equation}
\boldsymbol{\delta}^{*}(\boldsymbol{v})
=
\operatorname{arg\,min}_{\boldsymbol{\delta}}
\mathcal{J}_{\mathrm{inner}}(\boldsymbol{\delta};\boldsymbol{v}).
\end{equation}
The induced bi-level objective can be written as
\begin{equation}
F(\boldsymbol{v})
=
\mathcal{M}(\boldsymbol{\delta}^{*}(\boldsymbol{v}),\boldsymbol{v}),
\end{equation}
where \(\mathcal{M}\) denotes the meta-loss evaluated after applying the
virtual downstream update, including the edit loss, locality loss, and
meta-regularization on \(\boldsymbol{v}\). An isolated upstream analysis
would keep only the direct target derivative \(\nabla_{\boldsymbol{v}}\mathcal{M}\),
thereby ignoring how \(\boldsymbol{v}\) changes the downstream solution
\(\boldsymbol{\delta}^{*}(\boldsymbol{v})\).

\paragraph{Active-space parameterization.}
For projected downstream updates, the unconstrained matrix
\(\widetilde{\boldsymbol{\Delta}}_l\) is not identifiable because the inner
objective depends on it only through the applied update
\(\boldsymbol{B}_l=\widetilde{\boldsymbol{\Delta}}_l\boldsymbol{P}_l\).
Let \(\boldsymbol{U}_l\in\mathbb{R}^{d_0\times r_l}\) be an orthonormal basis
of \(\operatorname{range}(\boldsymbol{P}_l)\), so that
\(\boldsymbol{P}_l=\boldsymbol{U}_l\boldsymbol{U}_l^{T}\). We parameterize the
applied update as
\[
\boldsymbol{B}_l
=
\boldsymbol{A}_l\boldsymbol{U}_l^{T},
\qquad
\boldsymbol{A}_l\in\mathbb{R}^{d_1\times r_l},
\]
and define
\[
\boldsymbol{\delta}
=
\operatorname{vec}\big(\{\boldsymbol{A}_l\}_{l\in\mathcal{L}}\big).
\]
The Hessian in Proposition~\ref{prop:bilevel_hypergradient} is therefore taken
with respect to these active coordinates. This removes the structural null
directions \(\widetilde{\boldsymbol{\Delta}}_l(\boldsymbol{I}-\boldsymbol{P}_l)\),
which do not affect the applied update or the inner objective.

\begin{proposition}[First-order coupling induced by downstream realization]
\label{prop:bilevel_hypergradient}
Assume that \(\mathcal{J}_{\mathrm{inner}}\) and \(\mathcal{M}\) are twice
continuously differentiable in a neighborhood of
\((\boldsymbol{\delta}^{*}(\boldsymbol{v}),\boldsymbol{v})\), and that the
active-space Hessian
$\boldsymbol{H}(\boldsymbol{v})
=
\nabla_{\boldsymbol{\delta}\boldsymbol{\delta}}^{2}
\mathcal{J}_{\mathrm{inner}}
(\boldsymbol{\delta}^{*}(\boldsymbol{v});\boldsymbol{v})$
is nonsingular. Then \(\boldsymbol{\delta}^{*}(\boldsymbol{v})\) is locally
differentiable in \(\boldsymbol{v}\), and the gradient of the bi-level
objective satisfies
\begin{equation}
\label{eq:bilevel_hypergradient}
\nabla_{\boldsymbol{v}}F(\boldsymbol{v})
=
\nabla_{\boldsymbol{v}}\mathcal{M}
-\nabla_{\boldsymbol{v}\boldsymbol{\delta}}^{2}
\mathcal{J}_{\mathrm{inner}}
\boldsymbol{H}(\boldsymbol{v})^{-1}
\nabla_{\boldsymbol{\delta}}\mathcal{M},
\end{equation}
where all derivatives on the right-hand side are evaluated at
\((\boldsymbol{\delta}^{*}(\boldsymbol{v}),\boldsymbol{v})\). Consequently,
isolated upstream optimization and bi-level target optimization have the
same first-order target gradient at \(\boldsymbol{v}\) if and only if
\begin{equation}
\label{eq:coupling_zero_condition}
\nabla_{\boldsymbol{v}\boldsymbol{\delta}}^{2}
\mathcal{J}_{\mathrm{inner}}
\boldsymbol{H}(\boldsymbol{v})^{-1}
\nabla_{\boldsymbol{\delta}}\mathcal{M}
=
\boldsymbol{0}.
\end{equation}
\end{proposition}

\begin{proof}
The first-order optimality condition of the inner problem is
\begin{equation}
\nabla_{\boldsymbol{\delta}}
\mathcal{J}_{\mathrm{inner}}
(\boldsymbol{\delta}^{*}(\boldsymbol{v});\boldsymbol{v})
=
\boldsymbol{0}.
\end{equation}
Since \(\boldsymbol{H}(\boldsymbol{v})\) is nonsingular, the implicit
function theorem implies that \(\boldsymbol{\delta}^{*}(\boldsymbol{v})\)
is locally differentiable. Differentiating the inner optimality condition
with respect to \(\boldsymbol{v}\) gives
\begin{equation}
\boldsymbol{H}(\boldsymbol{v})
\frac{\partial \boldsymbol{\delta}^{*}}{\partial \boldsymbol{v}}
+
\nabla_{\boldsymbol{\delta}\boldsymbol{v}}^{2}
\mathcal{J}_{\mathrm{inner}}
=
\boldsymbol{0}.
\end{equation}
Therefore,
\begin{equation}
\frac{\partial \boldsymbol{\delta}^{*}}{\partial \boldsymbol{v}}
=
-
\boldsymbol{H}(\boldsymbol{v})^{-1}
\nabla_{\boldsymbol{\delta}\boldsymbol{v}}^{2}
\mathcal{J}_{\mathrm{inner}}.
\end{equation}
Applying the chain rule to
\(F(\boldsymbol{v})=\mathcal{M}(\boldsymbol{\delta}^{*}(\boldsymbol{v}),\boldsymbol{v})\)
yields
\begin{equation}
\nabla_{\boldsymbol{v}}F(\boldsymbol{v})
=
\nabla_{\boldsymbol{v}}\mathcal{M}
+
\left(
\frac{\partial \boldsymbol{\delta}^{*}}{\partial \boldsymbol{v}}
\right)^{\top}
\nabla_{\boldsymbol{\delta}}\mathcal{M}.
\end{equation}
Substituting the expression for
\(\partial\boldsymbol{\delta}^{*}/\partial\boldsymbol{v}\), we obtain
\begin{equation}
\begin{aligned}
\nabla_{\boldsymbol{v}}F(\boldsymbol{v})
=
\nabla_{\boldsymbol{v}}\mathcal{M}
-
\left(
\nabla_{\boldsymbol{\delta}\boldsymbol{v}}^{2}
\mathcal{J}_{\mathrm{inner}}
\right)^{T}
\left(
\boldsymbol{H}(\boldsymbol{v})^{-1}
\right)^{T}
\nabla_{\boldsymbol{\delta}}\mathcal{M}.
\end{aligned}
\end{equation}
Since
\(\boldsymbol{H}(\boldsymbol{v})
=
\nabla_{\boldsymbol{\delta}\boldsymbol{\delta}}^{2}
\mathcal{J}_{\mathrm{inner}}\)
is the Hessian of a twice continuously differentiable scalar objective in the
active coordinates, it is symmetric. Hence
\[
\left(\boldsymbol{H}(\boldsymbol{v})^{-1}\right)^{T}
=
\boldsymbol{H}(\boldsymbol{v})^{-1}.
\]
Moreover,
\[
\left(
\nabla_{\boldsymbol{\delta}\boldsymbol{v}}^{2}
\mathcal{J}_{\mathrm{inner}}
\right)^{T}
=
\nabla_{\boldsymbol{v}\boldsymbol{\delta}}^{2}
\mathcal{J}_{\mathrm{inner}}.
\]
Therefore,
\begin{equation}
\nabla_{\boldsymbol{v}}F(\boldsymbol{v})
=
\nabla_{\boldsymbol{v}}\mathcal{M}
-
\nabla_{\boldsymbol{v}\boldsymbol{\delta}}^{2}
\mathcal{J}_{\mathrm{inner}}
\boldsymbol{H}(\boldsymbol{v})^{-1}
\nabla_{\boldsymbol{\delta}}\mathcal{M},
\end{equation}
which gives Eq.~\eqref{eq:bilevel_hypergradient}. The isolated upstream gradient is
\(\nabla_{\boldsymbol{v}}\mathcal{M}\), whereas the bi-level gradient is
\(\nabla_{\boldsymbol{v}}F\). These two gradients are equal exactly when the
coupling term in Eq.~\eqref{eq:coupling_zero_condition} is zero.
\end{proof}

Proposition~\ref{prop:bilevel_hypergradient} shows that downstream-aware
target optimization is generally required at first order: the upstream target
\(\boldsymbol{v}\) affects the meta-objective through the downstream solution
\(\boldsymbol{\delta}^{*}(\boldsymbol{v})\). However, the result also makes
clear that decoupling is possible in special cases. For example, the coupling
term vanishes if the downstream solution is locally insensitive to the target,
i.e.,
\[
\nabla_{\boldsymbol{\delta}\boldsymbol{v}}^{2}
\mathcal{J}_{\mathrm{inner}}
=
\boldsymbol{0};
\]
if the meta-objective is already stationary with respect to the downstream
update, i.e.,
\[
\nabla_{\boldsymbol{\delta}}\mathcal{M}
=
\boldsymbol{0};
\]
or if the outer sensitivity lies in the null space of the downstream-response
operator,
\[
\nabla_{\boldsymbol{v}\boldsymbol{\delta}}^{2}
\mathcal{J}_{\mathrm{inner}}
\boldsymbol{H}(\boldsymbol{v})^{-1}
\nabla_{\boldsymbol{\delta}}\mathcal{M}
=
\boldsymbol{0}.
\]
These are locally decoupled regimes in which isolated upstream optimization
and bi-level target optimization share the same first-order target direction.
Outside such cases, the downstream solver contributes a nonzero
realization-dependent term, so treating \(\boldsymbol{v}\) as an isolated
semantic target misses part of the bi-level target gradient.

\begin{corollary}[Isolated target stationarity is generally insufficient]
\label{cor:isolated_not_stationary}
Let \(\boldsymbol{v}_{0}\) be stationary for the isolated upstream target
gradient, i.e.,
\[
\nabla_{\boldsymbol{v}}\mathcal{M}
(\boldsymbol{\delta}^{*}(\boldsymbol{v}_{0}),\boldsymbol{v}_{0})
=
\boldsymbol{0}.
\]
Define the downstream-realization coupling term at \(\boldsymbol{v}_{0}\) as
\[
\boldsymbol{c}(\boldsymbol{v}_{0})
=
\nabla_{\boldsymbol{v}\boldsymbol{\delta}}^{2}
\mathcal{J}_{\mathrm{inner}}
\boldsymbol{H}(\boldsymbol{v}_{0})^{-1}
\nabla_{\boldsymbol{\delta}}\mathcal{M}
\Big|_{(\boldsymbol{\delta}^{*}(\boldsymbol{v}_{0}),\boldsymbol{v}_{0})}.
\]
If \(\boldsymbol{c}(\boldsymbol{v}_{0})\neq\boldsymbol{0}\), then
\(\boldsymbol{v}_{0}\) is not a stationary point of the bi-level objective
\(F\). Moreover, there exists a sufficiently small \(\eta>0\) such that
\[
F(\boldsymbol{v}_{0}+\eta\boldsymbol{c}(\boldsymbol{v}_{0}))
<
F(\boldsymbol{v}_{0}).
\]
\end{corollary}

\begin{proof}
By Proposition~\ref{prop:bilevel_hypergradient} and the isolated
stationarity condition
\(\nabla_{\boldsymbol{v}}\mathcal{M}=\boldsymbol{0}\), we have
\[
\nabla_{\boldsymbol{v}}F(\boldsymbol{v}_{0})
=
-\boldsymbol{c}(\boldsymbol{v}_{0}).
\]
Therefore, the directional derivative of \(F\) at \(\boldsymbol{v}_{0}\)
along \(\boldsymbol{c}(\boldsymbol{v}_{0})\) is
\[
\left.
\frac{d}{d\eta}
F(\boldsymbol{v}_{0}+\eta\boldsymbol{c}(\boldsymbol{v}_{0}))
\right|_{\eta=0}
=
\nabla_{\boldsymbol{v}}F(\boldsymbol{v}_{0})^{\top}
\boldsymbol{c}(\boldsymbol{v}_{0})
=
-\|\boldsymbol{c}(\boldsymbol{v}_{0})\|_{2}^{2}
<
0.
\]
Thus, by differentiability of \(F\), there exists a sufficiently small
\(\eta>0\) such that
\(F(\boldsymbol{v}_{0}+\eta\boldsymbol{c}(\boldsymbol{v}_{0}))
<
F(\boldsymbol{v}_{0})\).
\end{proof}

This corollary formalizes the limitation of analyzing the upstream target as
an isolated semantic objective. Such an analysis is first-order valid only
when the coupling term vanishes. When the coupling is nonzero, the isolated
stationary target can still be improved by moving along the
downstream-realization direction, which motivates optimizing
\(\boldsymbol{v}^{*}\) as a meta-parameter through the downstream solver.

\subsection{Optimization Properties of Algorithm~\ref{alg:MetaKE}}
\label{app:convergence}

We analyze the optimization properties of Algorithm~\ref{alg:MetaKE}. The
analysis is stated for the tractable objective induced by the Virtual
Look-ahead step and then connected to the original multi-layer bi-level
objective through the Structural Gradient Proxy error characterized by
Theorem~\ref{thm:fidelity}.

\subsubsection{Virtual Look-ahead Objective}
\label{app:virtual_lookahead_objective}

Following Sec.~\ref{sec:method:proxy}, define the Virtual Look-ahead objective
as
\begin{equation}
\label{eq:proxy_objective}
\widehat{F}(\boldsymbol{v}^{*})
:=
\mathcal{L}_{\mathrm{meta}}
\left(
f(
\cdot;
\boldsymbol{W}_L+
\boldsymbol{\Delta}_{\mathrm{proxy}}(\boldsymbol{v}^{*})
)
\right),
\end{equation}
where
\begin{equation}
\label{eq:proxy_update_convergence}
\boldsymbol{\Delta}_{\mathrm{proxy}}(\boldsymbol{v}^{*})
=
(
\boldsymbol{v}^{*}
-
\boldsymbol{W}_L\boldsymbol{k}_1^L
)
\boldsymbol{M},
\end{equation}
and
\begin{equation}
\boldsymbol{M}
=
(\boldsymbol{k}_1^L)^{T}
\boldsymbol{P}_L
\left(
\boldsymbol{K}_p^L(\boldsymbol{K}_p^L)^{T}\boldsymbol{P}_L
+
\boldsymbol{k}_1^L(\boldsymbol{k}_1^L)^{T}\boldsymbol{P}_L
+
\boldsymbol{I}
\right)^{-1},
\end{equation}
where \(\boldsymbol{k}_1^L\) denotes the current edit key at the final edited
layer \(L\).
Under this definition, the Downstream-Aware Feedback Correction in
Algorithm~\ref{alg:MetaKE} is gradient descent on \(\widehat{F}\):
\begin{equation}
\label{eq:gd_proxy}
\boldsymbol{v}_{t+1}^{*}
=
\boldsymbol{v}_{t}^{*}
-
\eta
\nabla
\widehat{F}
(
\boldsymbol{v}_{t}^{*}
).
\end{equation}

\begin{assumption}[Local smoothness and lower boundedness]
\label{assump:smooth_proxy}
The Virtual Look-ahead objective \(\widehat{F}\) is continuously
differentiable and lower bounded on a neighborhood \(\mathcal{N}\) containing
all iterates generated by Algorithm~\ref{alg:MetaKE}. Moreover,
\(\widehat{F}\) is \(L_{\widehat{F}}\)-smooth on \(\mathcal{N}\), i.e.,
\begin{equation}
\|
\nabla
\widehat{F}
(
\boldsymbol{v}_1^{*}
)
-
\nabla
\widehat{F}
(
\boldsymbol{v}_2^{*}
)
\|_2
\le
L_{\widehat{F}}
\|
\boldsymbol{v}_1^{*}
-
\boldsymbol{v}_2^{*}
\|_2
\end{equation}
for all
\(\boldsymbol{v}_1^{*},\boldsymbol{v}_2^{*}\in\mathcal{N}\).
\end{assumption}

\begin{assumption}[Uniform Structural Gradient Proxy error]
\label{assump:proxy_gap}
Let
\begin{equation}
F(\boldsymbol{v}^{*})
=
\mathcal{L}_{\mathrm{meta}}
\left(
f(
\cdot;
\mathcal{W}
+
\boldsymbol{\Delta}\mathcal{W}^{*}
(
\boldsymbol{v}^{*}
)
)
\right)
\end{equation}
denote the original outer level objective induced by the full multi-layer
downstream constrained parameter optimization. There exists
\(\varepsilon_{\mathrm{SGP}}\ge 0\) such that
\begin{equation}
\|
\nabla
F(
\boldsymbol{v}^{*}
)
-
\nabla
\widehat{F}
(
\boldsymbol{v}^{*}
)
\|_2
\le
\varepsilon_{\mathrm{SGP}},
\qquad
\forall
\boldsymbol{v}^{*}\in\mathcal{N}.
\end{equation}
Under the conditions of Theorem~\ref{thm:fidelity},
\(\varepsilon_{\mathrm{SGP}}\) can be taken as a uniform upper bound of the
Structural Gradient Proxy error over \(\mathcal{N}\).
\end{assumption}

\subsubsection{Convergence of Downstream-Aware Feedback Correction}
\label{app:feedback_correction_convergence}

\begin{theorem}[Descent and first-order stationarity of Algorithm~\ref{alg:MetaKE}]
\label{thm:proxy_convergence}
Suppose Assumption~\ref{assump:smooth_proxy} holds and
Algorithm~\ref{alg:MetaKE} uses a constant step size
\begin{equation}
0<\eta\le \frac{1}{L_{\widehat{F}}}.
\end{equation}
Then the iterates generated by Eq.~\eqref{eq:gd_proxy} satisfy
\begin{equation}
\widehat{F}(\boldsymbol{v}_{t+1}^{*})
\le
\widehat{F}(\boldsymbol{v}_{t}^{*})
-
\frac{\eta}{2}
\|
\nabla
\widehat{F}
(
\boldsymbol{v}_{t}^{*}
)
\|_2^2.
\label{eq:descent_ineq}
\end{equation}
Consequently,
\begin{equation}
\frac{1}{T}
\sum_{t=0}^{T-1}
\|
\nabla
\widehat{F}
(
\boldsymbol{v}_{t}^{*}
)
\|_2^2
\le
\frac{
2
(
\widehat{F}(\boldsymbol{v}_{0}^{*})
-
\widehat{F}_{\inf}
)
}{
\eta T
},
\label{eq:min_grad_bound}
\end{equation}
and
\begin{equation}
\min_{0\le t<T}
\|
\nabla
\widehat{F}
(
\boldsymbol{v}_{t}^{*}
)
\|_2^2
\le
\frac{
2
(
\widehat{F}(\boldsymbol{v}_{0}^{*})
-
\widehat{F}_{\inf}
)
}{
\eta T
},
\label{eq:min_grad_bound_2}
\end{equation}
where \(\widehat{F}_{\inf}\) is any lower bound of \(\widehat{F}\) on
\(\mathcal{N}\).
\end{theorem}

\begin{proof}
Since \(\widehat{F}\) is \(L_{\widehat{F}}\)-smooth, the descent lemma gives
\begin{equation}
\begin{aligned}
\widehat{F}(\boldsymbol{v}_{t+1}^{*})
\le\;
&
\widehat{F}(\boldsymbol{v}_{t}^{*})
+
\left\langle
\nabla
\widehat{F}
(
\boldsymbol{v}_{t}^{*}
),
\boldsymbol{v}_{t+1}^{*}
-
\boldsymbol{v}_{t}^{*}
\right\rangle
\\
&
+
\frac{
L_{\widehat{F}}
}{2}
\|
\boldsymbol{v}_{t+1}^{*}
-
\boldsymbol{v}_{t}^{*}
\|_2^2.
\end{aligned}
\end{equation}
Substituting Eq.~\eqref{eq:gd_proxy} yields
\begin{equation}
\widehat{F}(\boldsymbol{v}_{t+1}^{*})
\le
\widehat{F}(\boldsymbol{v}_{t}^{*})
-
\eta
\|
\nabla
\widehat{F}
(
\boldsymbol{v}_{t}^{*}
)
\|_2^2
+
\frac{
L_{\widehat{F}}\eta^2
}{2}
\|
\nabla
\widehat{F}
(
\boldsymbol{v}_{t}^{*}
)
\|_2^2.
\end{equation}
Thus,
\begin{equation}
\widehat{F}(\boldsymbol{v}_{t+1}^{*})
\le
\widehat{F}(\boldsymbol{v}_{t}^{*})
-
\eta
\left(
1-\frac{L_{\widehat{F}}\eta}{2}
\right)
\|
\nabla
\widehat{F}
(
\boldsymbol{v}_{t}^{*}
)
\|_2^2.
\end{equation}
Since \(\eta\le 1/L_{\widehat{F}}\), Eq.~\eqref{eq:descent_ineq} follows.
Summing Eq.~\eqref{eq:descent_ineq} from \(t=0\) to \(T-1\) gives
\begin{equation}
\widehat{F}(\boldsymbol{v}_{T}^{*})
\le
\widehat{F}(\boldsymbol{v}_{0}^{*})
-
\frac{\eta}{2}
\sum_{t=0}^{T-1}
\|
\nabla
\widehat{F}
(
\boldsymbol{v}_{t}^{*}
)
\|_2^2.
\end{equation}
Using
\(\widehat{F}(\boldsymbol{v}_{T}^{*})\ge\widehat{F}_{\inf}\) and rearranging
proves Eq.~\eqref{eq:min_grad_bound}. Eq.~\eqref{eq:min_grad_bound_2} follows
because the minimum is upper bounded by the average. \qed
\end{proof}

\subsubsection{Connection to the Original Multi-layer Bi-level Objective}
\label{app:connection_original_bilevel}

The previous theorem establishes convergence for the Virtual Look-ahead
objective optimized during Algorithm~\ref{alg:MetaKE}. We next connect this
result to the original multi-layer bi-level objective in Eq.~\eqref{eq:blo_overall}.

\begin{corollary}[Approximate stationarity for the original multi-layer bi-level objective]
\label{cor:approx_stationarity}
Suppose Assumptions~\ref{assump:smooth_proxy} and~\ref{assump:proxy_gap}
hold. Let \(\{\boldsymbol{v}_{t}^{*}\}_{t=0}^{T-1}\) be the iterates generated
by Algorithm~\ref{alg:MetaKE}. Then
\begin{equation}
\min_{0\le t<T}
\|
\nabla
F
(
\boldsymbol{v}_{t}^{*}
)
\|_2^2
\le
\frac{
4
(
\widehat{F}(\boldsymbol{v}_{0}^{*})
-
\widehat{F}_{\inf}
)
}{
\eta T
}
+
2
\varepsilon_{\mathrm{SGP}}^2.
\label{eq:orig_stationarity}
\end{equation}
Thus Algorithm~\ref{alg:MetaKE} returns an
\(O(\varepsilon_{\mathrm{SGP}})\)-stationary point for the original
multi-layer bi-level objective, up to the Structural Gradient Proxy error.
\end{corollary}

\begin{proof}
By Assumption~\ref{assump:proxy_gap},
\begin{equation}
\|
\nabla
F
(
\boldsymbol{v}_{t}^{*}
)
\|_2
\le
\|
\nabla
\widehat{F}
(
\boldsymbol{v}_{t}^{*}
)
\|_2
+
\varepsilon_{\mathrm{SGP}}.
\end{equation}
Using \((a+b)^2\le 2a^2+2b^2\), we obtain
\begin{equation}
\|
\nabla
F
(
\boldsymbol{v}_{t}^{*}
)
\|_2^2
\le
2
\|
\nabla
\widehat{F}
(
\boldsymbol{v}_{t}^{*}
)
\|_2^2
+
2
\varepsilon_{\mathrm{SGP}}^2.
\end{equation}
Taking the minimum over \(t=0,\dots,T-1\) and applying
Eq.~\eqref{eq:min_grad_bound_2} proves Eq.~\eqref{eq:orig_stationarity}.
\end{proof}

Theorem~\ref{thm:proxy_convergence} shows that Algorithm~\ref{alg:MetaKE}
has the standard non-convex convergence guarantee on the Virtual Look-ahead
objective optimized during target refinement. Corollary~\ref{cor:approx_stationarity}
further shows that the optimized target is approximately stationary for the
original multi-layer bi-level objective, with the residual term controlled by
the Structural Gradient Proxy error. Therefore, Theorem~\ref{thm:fidelity}
does not claim exact equivalence between the Structural Gradient Proxy and
full differentiation through the multi-layer process; it characterizes when
the Structural Gate preserves a useful downstream constraint feedback signal
for optimizing \(\boldsymbol{v}^*\).

\subsection{Proof for Theorem~\ref{thm:fidelity}}
\label{app:proxy_fidelity}

We use \(\|\cdot\|_2\) for the spectral norm of matrices and the Euclidean
norm of vectors, and \(\|\cdot\|_F\) for the Frobenius norm. Let
\[
F(\boldsymbol{v}^{*})
=
\mathcal{L}_{\mathrm{meta}}\!\left(
f(\cdot;\mathcal{W}+
\boldsymbol{\Delta}\mathcal{W}^{*}(\boldsymbol{v}^{*}))
\right)
\]
be the exact outer-level objective. The direct derivative of the target
anchoring term is
\[
\boldsymbol{a}(\boldsymbol{v}^{*})
=
2(\boldsymbol{v}^{*}-\boldsymbol{v}_{\mathrm{init}}).
\]
We decompose the exact full multi-layer hypergradient as
\[
\boldsymbol{g}_{\mathrm{true}}
=
\boldsymbol{h}_{\mathrm{true}}
+
\boldsymbol{a}(\boldsymbol{v}^{*}),
\]
where the downstream-mediated component is written in vectorized form as
\[
\boldsymbol{h}_{\mathrm{true}}
=
\sum_{l\in\mathcal{L}}
\boldsymbol{J}_l^{T}\boldsymbol{s}_l^{\flat},
\qquad
\boldsymbol{s}_l^{\flat}
:=
\operatorname{vec}\!\left(
\nabla_{\boldsymbol{\Delta}_l}
\mathcal{L}_{\mathrm{meta}}
\right),
\qquad
\boldsymbol{J}_l
:=
\frac{
\partial\,\operatorname{vec}(
\boldsymbol{\Delta}_l^{*}(\boldsymbol{v}^{*})
)
}{
\partial\boldsymbol{v}^{*}
}.
\]
Here
\(\boldsymbol{s}_l^{\flat}\in\mathbb{R}^{d_1d_0}\) and
\(\boldsymbol{J}_l\in\mathbb{R}^{d_1d_0\times d_1}\), so
\(\boldsymbol{J}_l^{T}\boldsymbol{s}_l^{\flat}\in\mathbb{R}^{d_1}\), matching
the dimension of \(\boldsymbol{v}^{*}\). Since the anchoring term does not
depend on \(\boldsymbol{\Delta}_l\), \(\boldsymbol{s}_l^{\flat}\) is induced
by the edit and locality losses.

Let
\[
\boldsymbol{h}_L
:=
\boldsymbol{J}_L^{T}\boldsymbol{s}_L^{\flat}
\]
denote the final-layer downstream-mediated contribution. For the structural
proxy, we keep the gradient in matrix form:
\[
\boldsymbol{h}_{\mathrm{proxy}}
=
\boldsymbol{S}_{\mathrm{p}}
\widetilde{\boldsymbol{M}}^T,
\qquad
\boldsymbol{S}_{\mathrm{p}}
:=
\nabla_{\boldsymbol{\Delta}_{\mathrm{proxy}}}
\mathcal{L}_{\mathrm{meta}},
\]
where
\(\boldsymbol{S}_{\mathrm{p}}\in\mathbb{R}^{d_1\times d_0}\) and
\(\widetilde{\boldsymbol{M}}^T\in\mathbb{R}^{d_0\times 1}\), giving
\(\boldsymbol{h}_{\mathrm{proxy}}\in\mathbb{R}^{d_1}\). Equivalently, this is
the vectorized Jacobian contraction with
\(\partial\operatorname{vec}(\boldsymbol{\Delta}_{\mathrm{proxy}})/
\partial\boldsymbol{v}^{*}
=
\widetilde{\boldsymbol{M}}^T\otimes \boldsymbol{I}_{d_1}\).
The corresponding full proxy hypergradient is
\[
\boldsymbol{g}_{\mathrm{proxy}}
=
\boldsymbol{h}_{\mathrm{proxy}}
+
\boldsymbol{a}(\boldsymbol{v}^{*}).
\]
Thus the anchoring derivative is added exactly to both the true and proxy
hypergradients, and
\[
\boldsymbol{g}_{\mathrm{true}}
-
\boldsymbol{g}_{\mathrm{proxy}}
=
\boldsymbol{h}_{\mathrm{true}}
-
\boldsymbol{h}_{\mathrm{proxy}}.
\]

\paragraph{Local fidelity conditions.}
We use three local conditions on the downstream-mediated components.

\textbf{Condition 1 (Final-layer dominance).}
There exists \(\xi\ge 0\) such that
\[
\Big\|
\sum_{l\in\mathcal{L},\,l\ne L}
\boldsymbol{J}_l^{T}\boldsymbol{s}_l^{\flat}
\Big\|_2
\le
\xi
\|\boldsymbol{h}_L\|_2.
\]
This bounds the non-final-layer downstream-mediated contribution relative to
the final-layer signal.

\textbf{Condition 2 (Final-layer proxy gap).}
There exist \(K>0\) and \(\varepsilon_{\mathrm{AE}}\ge 0\) such that
\[
\|
\boldsymbol{h}_L
-
\boldsymbol{h}_{\mathrm{proxy}}
\|_2
\le
K\varepsilon_{\mathrm{AE}}
\|
\boldsymbol{S}_{\mathrm{p}}
\|_F
\|
\widetilde{\boldsymbol{M}}
\|_2.
\]
Here \(\varepsilon_{\mathrm{AE}}\) summarizes the local AlphaEdit-form
approximation error, including the gate-freezing discrepancy, the discarded
gate-Jacobian term, and the sensitivity drift between the exact final-layer
path and the proxy path. If the structural geometry is fixed within a local
outer step, the gate-Jacobian component vanishes and
\(\varepsilon_{\mathrm{AE}}\) mainly captures sensitivity drift.

\textbf{Condition 3 (Structural non-degeneracy).}
Assume
\(\boldsymbol{S}_{\mathrm{p}}\ne \boldsymbol{0}\),
\(\widetilde{\boldsymbol{M}}\ne \boldsymbol{0}\), and
\[
\alpha(\boldsymbol{v}^{*})
:=
\frac{
\|
\boldsymbol{h}_{\mathrm{proxy}}
\|_2
}{
\|
\boldsymbol{S}_{\mathrm{p}}
\|_F
\|
\widetilde{\boldsymbol{M}}
\|_2
}
\ge
\alpha_0>0.
\]
This excludes the degenerate case in which the structural gate annihilates the
downstream-mediated task signal.

Finally, define the full-gradient compatibility factor
\[
\chi(\boldsymbol{v}^{*})
:=
\frac{
\|\boldsymbol{g}_{\mathrm{proxy}}\|_2
}{
\|\boldsymbol{h}_{\mathrm{proxy}}\|_2
}.
\]
This factor accounts for the exact anchoring derivative in the full proxy
hypergradient. When the anchoring term is zero or does not reduce the proxy
direction norm, \(\chi(\boldsymbol{v}^{*})\ge 1\), and the condition reduces
to the downstream-mediated form.

\paragraph{Proof.}
By adding and subtracting \(\boldsymbol{h}_L\), we have
\[
\|
\boldsymbol{h}_{\mathrm{true}}
-
\boldsymbol{h}_{\mathrm{proxy}}
\|_2
\le
\|
\boldsymbol{h}_{\mathrm{true}}
-
\boldsymbol{h}_L
\|_2
+
\|
\boldsymbol{h}_L
-
\boldsymbol{h}_{\mathrm{proxy}}
\|_2.
\]
Using Conditions 1 and 2,
\[
\|
\boldsymbol{h}_{\mathrm{true}}
-
\boldsymbol{h}_{\mathrm{proxy}}
\|_2
\le
\xi
\|
\boldsymbol{h}_L
\|_2
+
K\varepsilon_{\mathrm{AE}}
\|
\boldsymbol{S}_{\mathrm{p}}
\|_F
\|
\widetilde{\boldsymbol{M}}
\|_2.
\]
Condition 3 gives
\[
\|
\boldsymbol{S}_{\mathrm{p}}
\|_F
\|
\widetilde{\boldsymbol{M}}
\|_2
\le
\frac{
\|\boldsymbol{h}_{\mathrm{proxy}}\|_2
}{
\alpha_0
}.
\]
Therefore,
\[
\|
\boldsymbol{h}_L
\|_2
\le
\|
\boldsymbol{h}_{\mathrm{proxy}}
\|_2
+
\|
\boldsymbol{h}_L
-
\boldsymbol{h}_{\mathrm{proxy}}
\|_2
\le
\left(
1+
\frac{
K\varepsilon_{\mathrm{AE}}
}{
\alpha_0
}
\right)
\|
\boldsymbol{h}_{\mathrm{proxy}}
\|_2.
\]
Substituting this bound yields
\[
\|
\boldsymbol{h}_{\mathrm{true}}
-
\boldsymbol{h}_{\mathrm{proxy}}
\|_2
\le
\left[
\xi
+
\frac{
K(1+\xi)
}{
\alpha_0
}
\varepsilon_{\mathrm{AE}}
\right]
\|
\boldsymbol{h}_{\mathrm{proxy}}
\|_2.
\]
Since the anchoring derivative is shared exactly by the true and proxy full
hypergradients,
\[
\|
\boldsymbol{g}_{\mathrm{true}}
-
\boldsymbol{g}_{\mathrm{proxy}}
\|_2
=
\|
\boldsymbol{h}_{\mathrm{true}}
-
\boldsymbol{h}_{\mathrm{proxy}}
\|_2.
\]
Using the definition of \(\chi(\boldsymbol{v}^{*})\), we obtain
\[
\|
\boldsymbol{g}_{\mathrm{true}}
-
\boldsymbol{g}_{\mathrm{proxy}}
\|_2
\le
\frac{
\xi
+
\frac{
K(1+\xi)
}{
\alpha_0
}
\varepsilon_{\mathrm{AE}}
}{
\chi(\boldsymbol{v}^{*})
}
\,
\|\boldsymbol{g}_{\mathrm{proxy}}\|_2.
\]
If
\[
\xi+
\frac{
K(1+\xi)
}{
\alpha_0
}
\varepsilon_{\mathrm{AE}}
<
\chi(\boldsymbol{v}^{*}),
\]
then
\[
\|
\boldsymbol{g}_{\mathrm{true}}
-
\boldsymbol{g}_{\mathrm{proxy}}
\|_2
<
\|
\boldsymbol{g}_{\mathrm{proxy}}
\|_2.
\]
Hence, by Cauchy--Schwarz,
\[
\begin{aligned}
\langle
\boldsymbol{g}_{\mathrm{true}},
\boldsymbol{g}_{\mathrm{proxy}}
\rangle
&=
\|
\boldsymbol{g}_{\mathrm{proxy}}
\|_2^2
+
\langle
\boldsymbol{g}_{\mathrm{true}}
-
\boldsymbol{g}_{\mathrm{proxy}},
\boldsymbol{g}_{\mathrm{proxy}}
\rangle \\
&\ge
\|
\boldsymbol{g}_{\mathrm{proxy}}
\|_2
\left(
\|
\boldsymbol{g}_{\mathrm{proxy}}
\|_2
-
\|
\boldsymbol{g}_{\mathrm{true}}
-
\boldsymbol{g}_{\mathrm{proxy}}
\|_2
\right)
>0.
\end{aligned}
\]
Thus \(-\boldsymbol{g}_{\mathrm{proxy}}\) is a descent direction for the exact
outer objective \(F\). If \(F\) is locally smooth, the standard descent lemma
implies that
\[
F(\boldsymbol{v}^{*}-\eta\boldsymbol{g}_{\mathrm{proxy}})
<
F(\boldsymbol{v}^{*})
\]
for all sufficiently small \(\eta>0\).

Equivalently, when \(\chi(\boldsymbol{v}^{*})>\xi\), the condition can be
written as
\[
\varepsilon_{\mathrm{AE}}
<
\frac{
\alpha_0(\chi(\boldsymbol{v}^{*})-\xi)
}{
K(1+\xi)
}.
\]
When \(\chi(\boldsymbol{v}^{*})\ge 1\), this recovers the simpler sufficient
condition
\[
\varepsilon_{\mathrm{AE}}
<
\frac{
\alpha_0(1-\xi)
}{
K(1+\xi)
}.
\]

\paragraph{Interpretation under long sequential editing.}
The bound also clarifies the role of editable capacity. As sequential edits
accumulate, preservation and previous-edit constraints may occupy more projected
directions, weakening the structural gate. This effect is captured by
\(\alpha(\boldsymbol{v}^{*})\): a smaller \(\alpha_0\) tightens the admissible
proxy-gap threshold
\(\alpha_0(\chi(\boldsymbol{v}^{*})-\xi)/(K(1+\xi))\). Thus, the theorem is a
conditional guarantee requiring the projected gate to retain an informative
direction, rather than an unconditional guarantee under unlimited editing
capacity. In practice, \(\alpha(\boldsymbol{v}^{*})\) or
\(\|\widetilde{\boldsymbol{M}}\|_2\) can be monitored, and the edit batch size
or edited layer set can be adjusted when projected capacity becomes limited.

\section{Methodology Details}
\label{app:method_details}

\subsection{Target Optimization}
\label{app:target_optimization}

The meta-loss in Eq.~\eqref{eq:meta_loss} consists of three terms.

\textbf{Edit success loss.}
This term encourages the post-edit model to produce the desired target output
\(y_e\) given the edit input \(x_e\):
\begin{equation}
\mathcal{L}_{\mathrm{edit}}
(x_e,y_e;\widehat{\mathcal{W}})
=
-\log
\mathbb{P}_{\widehat{\mathcal{W}}}
[y_e\mid x_e].
\label{eq:method:edit_loss}
\end{equation}

\textbf{Locality preservation loss.}
Following MEMIT-style target-optimization variants that construct
\(\boldsymbol{v}^{*}\) with likelihood and KL-preservation objectives
\citep{li2024pmet,park-etal-2025-context,wei2025setke,fei2026scaling}, we use
subject-essence KL probe to regularize the meta-refined target. For an edit
request with subject \(s_i\), we instantiate the locality prompt as $x_{\mathrm{loc}}=\texttt{\{}s_i\texttt{\} is a}$.
This prompt is used to preserve the pre-edit output distribution around the
subject while optimizing the target representation. We match the output
distribution of the post-edit model to that of the pre-edit model on a
locality set \(x_{\mathrm{loc}}\):
\begin{equation}
\mathcal{L}_{\mathrm{loc}}
(x_{\mathrm{loc}};\widehat{\mathcal{W}},\mathcal{W})
=
D_{\mathrm{KL}}
\left(
P_{\mathcal{W}}(y\mid x_{\mathrm{loc}})
\Vert
P_{\widehat{\mathcal{W}}}(y\mid x_{\mathrm{loc}})
\right).
\label{eq:method:loc_loss}
\end{equation}

\paragraph{Locality-Constraint Ablation.}
The Subject-Essence prompt in Eq.~\eqref{eq:method:loc_loss}
is a lightweight subject-centered locality proxy, not a substitute
for general locality evaluation. The main experiments measure
Specificity on benchmark-native prompts distinct from this probe.
We compare no outer locality loss, Subject-Essence, Neighborhood-1,
and Neighborhood-5 on GPT-J/CounterFact (2,000 edits, batch size 100,
seed 0), keeping all other optimization and realization settings fixed.
For each edit, relation-matched neighborhood prompts are drawn from
CounterFact records disjoint from the editing stream, validation set,
and benchmark-provided Specificity prompts; the latter are used only
for final evaluation. Neighborhood-$m$ averages KL losses over $m$
fixed prompts, keeping the locality-loss scale unchanged. As shown in Table~\ref{tab:locality_constraint_ablation},
removing the outer locality loss changes Efficacy and Generalization
only marginally but decreases Specificity and Consistency by
$0.79$ and $0.37$ points, while reducing runtime by $11.32\%$.
Relative to Subject-Essence, Neighborhood-1/5 improve Specificity
by only $0.11/0.45$ points at $6.35\%/49.30\%$ higher runtime;
Neighborhood-5 also lowers Efficacy and Generalization by
$0.40$ points each. Subject-Essence thus provides a favorable
lightweight trade-off between locality preservation and efficiency.

\begin{table}[t]
\centering
\caption{Locality-constraint ablation on GPT-J/CounterFact}
\vspace{-0.1in}
\label{tab:locality_constraint_ablation}
\small
\setlength{\tabcolsep}{4pt}
\renewcommand{\arraystretch}{1}
\begin{tabular}{lccccc c}
\toprule
Locality constraint
& Eff.$\uparrow$ & Gen.$\uparrow$ & Spe.$\uparrow$
& Flu.$\uparrow$ & Cons.$\uparrow$
& Runtime/100 edits (s)$\downarrow$\\
\midrule
No Outer Locality Loss & 99.95 & 98.82 & 75.97 & 620.43 & 43.16 & 459.33\\
Subject-Essence        & 99.90 & 98.75 & 76.76 & 620.63 & 43.53 & 517.94\\
Neighborhood-1         & 99.75 & 98.55 & 76.87 & 620.56 & 43.63 & 550.83\\
Neighborhood-5         & 99.50 & 98.35 & 77.21 & 620.14 & 43.69 & 773.26\\
\bottomrule
\end{tabular}
\vspace{-0.1in}
\end{table}

\textbf{Target anchoring.}
The term
\(\|\boldsymbol{v}^{*}-\boldsymbol{v}_{\mathrm{init}}\|_2^2\) keeps the
refined target close to the initial semantic plan. This prevents the proxy
refinement from drifting too far from the target representation produced by
the upstream target optimization stage. The downstream-aware information is
introduced through the structural gate in Eq.~\eqref{eq:proxy_gradient}, while
this term serves as a semantic anchor and optimization stabilizer.

\subsection{Layer-Specific Target Allocation}
\label{app:layer_target_allocation}

We clarify how the layer-specific target
\(\boldsymbol{t}_{l}(\boldsymbol{v}^{*})\) in Eq.~\eqref{eq:lower_level} is
constructed from the global target \(\boldsymbol{v}^{*}\). Following the
standard multi-layer residual allocation in MEMIT and AlphaEdit-style editors,
we first define the last edited layer as
\(L=\max(\mathcal{L})\). During the sequential update over edited layers, let
\(\boldsymbol{h}_{1,L}^{(<l)}\) denote the current hidden state of the edit
request at layer \(L\), after applying the updates to edited layers before
\(l\) and re-running the forward pass. The remaining target displacement before
updating layer \(l\) is
\begin{equation}
\boldsymbol{d}_{l}(\boldsymbol{v}^{*})
=
\boldsymbol{v}^{*}
-
\boldsymbol{h}_{1,L}^{(<l)} .
\end{equation}
Let
\(\mathcal{L}_{\ge l}=\{j\in\mathcal{L}:j\ge l\}\) be the set of edited layers
not yet processed. We allocate an equal portion of the remaining displacement to
the current layer:
\begin{equation}
\boldsymbol{r}_{l}(\boldsymbol{v}^{*})
=
\frac{1}{|\mathcal{L}_{\ge l}|}
\boldsymbol{d}_{l}(\boldsymbol{v}^{*}) .
\end{equation}
The layer-specific target is then
\begin{equation}
\boldsymbol{t}_{l}(\boldsymbol{v}^{*})
=
\boldsymbol{W}_{l}\boldsymbol{k}_{1}^{l}
+
\boldsymbol{r}_{l}(\boldsymbol{v}^{*}) .
\end{equation}
For consecutive edited layers, this recovers the MEMIT allocation
\(\boldsymbol{r}_{l}=(\boldsymbol{v}^{*}-\boldsymbol{h}_{1,L})/(L-l+1)\).
In the AlphaEdit-style realization used by MetaKE, the same
\(\boldsymbol{t}_{l}(\boldsymbol{v}^{*})\) defines the residual
\begin{equation}
\boldsymbol{R}_{l}(\boldsymbol{v}^{*})
=
\boldsymbol{t}_{l}(\boldsymbol{v}^{*})
-
\boldsymbol{W}_{l}\boldsymbol{k}_{1}^{l},
\end{equation}
which is then realized through the projected downstream update with
\(\boldsymbol{P}_{l}\). For a batch of edits, the vectors above are stacked as
columns, giving
\(\boldsymbol{R}_{l}
=
\boldsymbol{T}_{l}
-
\boldsymbol{W}_{l}\boldsymbol{K}_{l}\).

\subsection{Batched Virtual Look-ahead}
\label{app:batched_virtual_lookahead}

For clarity, the Structural Gradient Proxy in Sec.~\ref{sec:method:proxy} and
Algorithm~\ref{alg:MetaKE} is written in single-request notation. In the
batched implementation, the current edit keys at the final edited layer are
stacked as
\[
\boldsymbol{K}_1^L
=
[
\boldsymbol{k}_{1,1}^{L},
\ldots,
\boldsymbol{k}_{1,b}^{L}
],
\]
and the corresponding target residuals are stacked as
\[
\boldsymbol{R}^{L}
=
[
\boldsymbol{r}_{1}^{L},
\ldots,
\boldsymbol{r}_{b}^{L}
],
\qquad
\boldsymbol{r}_{i}^{L}
=
\boldsymbol{v}_{i}^{*}
-
\boldsymbol{W}_{L}\boldsymbol{k}_{1,i}^{L}.
\]
The batched structural gate is
\[
\boldsymbol{M}_{\mathcal{B}}
=
(\boldsymbol{K}_1^L)^{T}
\boldsymbol{P}_{L}
\left(
\boldsymbol{K}_{p}^{L}(\boldsymbol{K}_{p}^{L})^{T}\boldsymbol{P}_{L}
+
\boldsymbol{K}_1^L(\boldsymbol{K}_1^L)^{T}\boldsymbol{P}_{L}
+
\boldsymbol{I}
\right)^{-1},
\]
and the virtual look-ahead update is computed as
\[
\boldsymbol{\Delta}_{\mathrm{proxy}}^{\mathcal{B}}
=
\boldsymbol{R}^{L}
\boldsymbol{M}_{\mathcal{B}} .
\]
The term
\(\boldsymbol{K}_1^L(\boldsymbol{K}_1^L)^{T}\boldsymbol{P}_L\)
is the current-batch interaction term in the original AlphaEdit closed-form
update. During each outer refinement
step, \(\boldsymbol{M}_{\mathcal{B}}\) acts as a frozen structural gate.
Thus, the proxy accounts for dominant batch interaction in virtual
look-ahead response while avoiding differentiation through changes in batch
key statistics. This approximation may be less accurate for very large
batches or batches containing highly correlated edit requests.

\section{Experimental Setup}
\label{sec:appendix_exp}

This section supplements the main paper with a detailed account of the
experimental setup.
We summarize the benchmarks, evaluation criteria, comparison methods, and
additional assessments used throughout our study.
Our goal is to make the empirical protocol transparent, reproducible, and
directly comparable to prior work on knowledge editing.

\subsection{Datasets}
\label{sec:appendix_datasets}

\paragraph{ZsRE.}
ZsRE \citep{levy-etal-2017-zero} is a widely used benchmark for factual
editing in a question-answering format.
Each sample corresponds to a prompt encoding a subject--relation query
$(s_i, r_i)$ together with an expected answer $o_i$.
Following the standard evaluation pipeline in prior editing studies
\citep{meng2022locating}, the benchmark provides three types of inputs:
the original query for measuring edit success, paraphrased queries for
measuring semantic transfer, and locality probes for testing whether the
edit remains confined to the intended fact.
As a result, ZsRE is particularly suitable for evaluating paraphrase
robustness and locality preservation.

\paragraph{CounterFact.}
CounterFact \citep{meng2022locating} is a counterfactual factual editing
benchmark designed to test whether a model can replace an existing fact
with a new one.
Its locality examples are constructed by substituting the subject entity
with semantically related alternatives while preserving the same relation,
thereby enabling a fine-grained assessment of collateral changes.
Beyond the standard editing criteria, CounterFact also includes natural
language generation prompts, which makes it possible to evaluate the
quality of generated text using fluency- and consistency-oriented metrics.

\subsection{Evaluation Metrics}
\label{sec:appendix_metrics}

\subsubsection{Metrics on ZsRE}
\label{sec:appendix_metrics_zsre}

On ZsRE, we adopt the standard triad of \emph{efficacy},
\emph{generalization}, and \emph{specificity}
\citep{mengmass,mitchellfast,meng2022locating}.

\paragraph{Efficacy.}
Efficacy measures whether the edited model returns the desired target
object when queried with the original edit prompt.
Let $f_\theta$ denote the edited model, let $x_i$ be the original query
for the $i$-th edit request, and let $o_i$ be the desired edited target.
We compute efficacy as:
\begin{equation}
\mathbb{E}_i \left[
\arg\max_o P_{f_\theta}(o \mid x_i) = o_i
\right].
\end{equation}

\paragraph{Generalization.}
Generalization evaluates whether the updated fact is expressed
consistently under semantically equivalent rephrasings.
Let $N(x_i)$ denote the paraphrase set associated with $x_i$.
The corresponding metric is the average accuracy over these paraphrased
queries:
\begin{equation}
\mathbb{E}_i \ \mathbb{E}_{\tilde{x}_i \in N(x_i)} \left[
\arg\max_o P_{f_\theta}(o \mid \tilde{x}_i) = o_i
\right].
\end{equation}

\paragraph{Specificity.}
Specificity quantifies whether the modification stays local and avoids
changing predictions on unrelated prompts.
Let $O(x_i)$ denote out-of-scope neighborhood prompts for the
$i$-th example, and let $o_i^c$ be the model's original top-1 prediction
on those prompts before editing.
We define specificity as:
\begin{equation}
\mathbb{E}_i \ \mathbb{E}_{\bar{x}_i \in O(x_i)} \left[
\arg\max_o P_{f_\theta}(o \mid \bar{x}_i) = o_i^c
\right].
\end{equation}

\subsubsection{Metrics on CounterFact}
\label{sec:appendix_metrics_counterfact}

For CounterFact, we follow the common practice of comparing the edited
(counterfactual) object against the original factual object
\citep{meng2022locating,mengmass}. The evaluation focuses on whether the model's
probability mass is successfully shifted toward the new fact across
different prompt variants.
For the $i$-th instance, let $x_i$ be the original prompt,
$o_i$ the desired counterfactual target, and $o_i^c$ the original fact.

\paragraph{Efficacy.}
We first measure whether the edited model prefers the target object to the
original factual object on the original prompt:
\begin{equation}
\mathbb{E}_i \left[
P_{f_\theta}(o_i \mid x_i) >
P_{f_\theta}(o_i^c \mid x_i)
\right].
\end{equation}

\paragraph{Generalization.}
We then evaluate whether this preference is preserved across paraphrased
variants $\tilde{x}_i \in N(x_i)$:
\begin{equation}
\mathbb{E}_i \ \mathbb{E}_{\tilde{x}_i \in N(x_i)} \left[
P_{f_\theta}(o_i \mid \tilde{x}_i) >
P_{f_\theta}(o_i^c \mid \tilde{x}_i)
\right].
\end{equation}

\paragraph{Specificity.}
To evaluate locality, we consider neighborhood prompts
$\bar{x}_i \in O(x_i)$ that should remain unaffected by the edit.
Specificity is defined as the proportion of such cases in which the model
still assigns higher probability to the original fact than to the edited
target:
\begin{equation}
\mathbb{E}_i \ \mathbb{E}_{\bar{x}_i \in O(x_i)} \left[
P_{f_\theta}(o_i^c \mid \bar{x}_i) >
P_{f_\theta}(o_i \mid \bar{x}_i)
\right].
\end{equation}

\paragraph{Fluency.}
Since CounterFact also includes generation-based evaluation,
we assess fluency by measuring repetitive patterns in generated text.
Following prior work, we use an entropy-based score computed from the
empirical bi-gram and tri-gram distributions:
\begin{equation}
-\frac{2}{3} \sum_k g_2(k)\log_2 g_2(k)
-\frac{4}{3} \sum_k g_3(k)\log_2 g_3(k),
\end{equation}
where $g_n(k)$ denotes the empirical frequency of the $n$-gram $k$.

\paragraph{Consistency.}
We evaluate whether generated text remains aligned with
external knowledge about the edited fact.
Specifically, we compute cosine similarity between the TF--IDF
representation of model output $y_i$ and that of a reference
Wikipedia passage $y_i^{\text{wiki}}$ associated with the edited target:
\begin{equation}
\text{Consis.} =
\cos\Big(
\text{tfidf}(y_i),
\text{tfidf}(y_i^{\text{wiki}})
\Big).
\end{equation}

\subsection{Baseline Methods}
\label{sec:appendix_baselines}

\paragraph{MEND.}
MEND \citep{mitchellfast} is a meta-learning-based editing method that
learns small auxiliary networks to transform the standard fine-tuning gradient
into a targeted parameter update.
By exploiting the low-rank structure of feed-forward gradients, it enables
fast and efficient one-shot edits.
It is a strong baseline for scalable editing, but its performance can degrade
under long sequential editing streams.

\paragraph{MALMEN.}
MALMEN \citep{tan2024massive} extends gradient-based meta-editing to the
multi-fact setting.
Instead of directly summing parameter shifts from different edits, it
formulates their aggregation as a least-squares problem and solves for a joint
update.
This makes it effective for large batched edits, though it is less tailored to
long-horizon sequential editing.

\paragraph{ROME.}
ROME \citep{meng2022locating} is a representative locate-and-edit method
that modifies a single feed-forward layer believed to be responsible for
recalling the target fact.
After identifying a causally important layer, it applies a rank-one
parameter update to write the edited association into the model.
Because it performs sparse intervention, ROME often exhibits strong
locality, although its robustness on more complex or multi-fact settings
can be limited.

\paragraph{MEMIT.}
MEMIT \citep{mengmass} generalizes ROME from a single-layer
intervention to a multi-layer editing framework.
Instead of concentrating the update at one location, MEMIT distributes a
shared editing residual across a range of critical layers and solves the
corresponding key--value adjustment via least squares.
This makes it a strong baseline for large-scale or batched editing, but
it can accumulate interference under long sequential editing streams.

\paragraph{PRUNE.}
PRUNE \citep{ma2025perturbationrestrained} is designed for more stable
sequential editing by explicitly controlling spectral updates in the
edited weight matrices.
Its key idea is to preserve numerically stable directions while
suppressing harmful ones, thereby reducing edit interference and slowing
down performance degradation over repeated updates.

\paragraph{RECT.}
RECT \citep{gu2024model} introduces regularization during the
editing process to keep the updated parameters close to the pre-trained
model.
Compared with methods that optimize primarily for factual overwrite,
RECT places greater emphasis on balancing edit success with the retention
of broader language and reasoning capabilities.

\paragraph{PMET.}
PMET \citep{li2024pmet} is an optimization-based editing method that refines
the key--value view of Transformer feed-forward networks.
It jointly optimizes hidden states from both MHSA and FFN, but only uses the
optimized FFN-side representations to update FFN weights.
This more precise formulation often improves reliability, although it relies
on stronger assumptions about the functional roles of Transformer components.

\paragraph{AlphaEdit.}
AlphaEdit \citep{fangalphaedit} builds on MEMIT by incorporating
null-space constraints into the residual update.
Concretely, it projects the editing signal away from protected or
previously occupied directions, so that newly written knowledge causes
less interference with preserved content.
This mechanism substantially improves stability in sequential editing
while still retaining the flexibility of distributed multi-layer updates.

\paragraph{$\text{AlphaEdit}_{\text{BLUE}}$.}
$\text{AlphaEdit}_{\text{BLUE}}$ \citep{li2025rethinking} further refines
the AlphaEdit framework by limiting edits to a small set of empirically
identified boundary layers.
This design is motivated by the observation that factual interventions
may exhibit depth-dependent sensitivity in Transformer models.
Although such restriction can improve efficiency and reduce disturbance,
it also imposes a stronger architectural bias that may be suboptimal when
different facts rely on heterogeneous layerwise mechanisms.

\paragraph{SPHERE.}
SPHERE \citep{liu2026energyregularized} is a sequential editing method designed to
improve stability by preserving the hyperspherical uniformity of neuron
weights.
It projects editing signals into a sparse complementary subspace, reducing
interference with dominant pretrained directions.
This design improves robustness in long editing streams, though it depends on
a geometric assumption about knowledge preservation in weight space.

\subsection{Evaluation of General Capability}
\label{sec:appendix_general_capability}

To assess whether sequential knowledge editing harms general language
competence, we evaluate the edited models on several standard natural
language understanding benchmarks.
These benchmarks are excluded from the editing data and cover a range of
linguistic abilities, including sentiment analysis, semantic matching,
grammatical judgment, and inference.

\paragraph{SST.}
We adopt SST \citep{socher2013recursive} to measure the effect of sequential
editing on sentence-level semantic understanding.
The task asks the model to assign binary sentiment labels to short sentences
with diverse syntactic patterns.
Since SST can be affected by small distributional changes, performance shifts
may indicate degradation in basic language comprehension.
We report classification accuracy.

\paragraph{MRPC.}
MRPC \citep{dolan2005automatically} tests whether the model identifies
semantic equivalence between sentence pairs.
It is useful for evaluating whether editing damages fine-grained paraphrase
understanding.
Following common evaluation protocols for this label-imbalanced task, we
report F1 score.

\paragraph{MMLU.}
We include MMLU \citep{hendrycks2021measuring} to evaluate the preservation
of broad-domain knowledge and reasoning ability.
The benchmark spans many subject areas and requires factual recall as well as
reasoning over task-specific information.
We evaluate the models in both zero-shot and few-shot settings to separate the
effect of editing from task-specific adaptation.

\paragraph{RTE.}
RTE \citep{bentivogli2009fifth} is used to examine whether repeated edits
affect textual entailment ability.
The task requires predicting whether a hypothesis is entailed by a given
premise.
We report accuracy as the measure of inference performance.

\paragraph{CoLA.}
We use CoLA \citep{warstadt2019neural} to assess grammatical acceptability
judgment.
Because CoLA targets subtle syntactic phenomena, it helps reveal whether
editing introduces unintended changes to linguistic form.
Following standard practice, we report Matthews correlation coefficient
(MCC).

\paragraph{NLI.}
We further evaluate the edited models on a natural language inference
benchmark \citep{williams2018broad}.
The task requires classifying sentence-pair relations into entailment,
contradiction, or neutral.
This benchmark measures robustness in semantic reasoning, and we report
accuracy.

\subsection{Empirical Details for Structural Proxy Validation}
\label{app:proxy_validation_details}

This section details the empirical validation protocol for the structural
proxy. All quantities are computed per edit; we report distributions over
\(N{=}100\) edits sampled uniformly from ZsRE, evaluated on GPT2-XL after the
target representation \(\boldsymbol{v}^{*}\) has converged.

\paragraph{True multi-layer hypergradient.}
For each edit, we compute the exact full multi-layer hypergradient
\(\boldsymbol{g}_{\mathrm{true}}\) by backpropagating
\(\mathcal{L}_{\mathrm{meta}}\) through the full sequential multi-layer
downstream realization over the editable layer set \(\mathcal{L}\). The
downstream applies the update layer by layer and recomputes the layer-specific
key statistics and current target representation after preceding layer updates,
thereby preserving cross-layer dependencies in the lower-level response. We
then differentiate the outer-level objective with respect to
\(\boldsymbol{v}^{*}\):
\[
\boldsymbol{g}_{\mathrm{true}}
=
\nabla_{\boldsymbol{v}^{*}}
\mathcal{L}_{\mathrm{meta}}
\!\left(
f(\cdot;\mathcal{W}+
\Delta\mathcal{W}^{*}(\boldsymbol{v}^{*}))
\right).
\]

\paragraph{Proxy hypergradient.}
The structural proxy gradient is computed as the full proxy hypergradient
\[
\boldsymbol{g}_{\mathrm{proxy}}
=
\boldsymbol{h}_{\mathrm{proxy}}
+
2(\boldsymbol{v}^{*}-\boldsymbol{v}_{\mathrm{init}}),
\qquad
\boldsymbol{h}_{\mathrm{proxy}}
=
\boldsymbol{S}_{\mathrm{p}}
\widetilde{\boldsymbol{M}}^{T},
\]
where \(\boldsymbol{S}_{\mathrm{p}}\) is the meta-loss gradient with respect
to the proxy update and \(\widetilde{\boldsymbol{M}}\) is the frozen structural
gate at layer \(L\). The first term
\(\boldsymbol{h}_{\mathrm{proxy}}\) is the downstream-mediated component,
while the second term is the exact direct derivative of the target anchoring
term.

\paragraph{One-step descent.}
For each edit, we apply a single outer-level update
\[
\boldsymbol{v}^{*+}
=
\boldsymbol{v}^{*}
-
\eta\boldsymbol{g}
\]
along four directions: (i) \(\boldsymbol{g}_{\mathrm{true}}\), (ii)
\(\boldsymbol{g}_{\mathrm{proxy}}\), (iii) a key-only-gate variant that
replaces \(\widetilde{\boldsymbol{M}}\) with
\(\widetilde{\boldsymbol{M}}_{\mathrm{key}}\), obtained by removing
\(\boldsymbol{P}_L\) and
\(\boldsymbol{K}_p^L(\boldsymbol{K}_p^L)^T\) from the structural gate, and
(iv) No update. We report the relative one-step
reduction in basis points of the current meta-loss:
\[
\Delta_{\mathrm{rel}}
=
10^4\cdot
\frac{
\mathcal{L}_{\mathrm{meta}}(\boldsymbol{v}^{*})
-
\mathcal{L}_{\mathrm{meta}}(\boldsymbol{v}^{*+})
}{
|\mathcal{L}_{\mathrm{meta}}(\boldsymbol{v}^{*})|+\epsilon
}.
\]
The reported value is a one-step relative meta-loss reduction measured in
basis points; it is used only to compare descent strength across update
directions under the same step size.

\paragraph{Assumption diagnostics.}
The structural quantities \(\xi\), \(K\varepsilon_{\mathrm{AE}}\), and
\(\alpha\) are computed on the downstream-mediated components, because the
anchoring derivative is exact and shared by the true and proxy full
hypergradients. Per edit, we estimate
\[
\xi
=
\frac{
\left\|
\sum_{l\ne L}
\boldsymbol{J}_l^{T}\boldsymbol{s}_l^{\flat}
\right\|_2
}{
\|\boldsymbol{h}_L\|_2+\epsilon
},
\]
\[
K\varepsilon_{\mathrm{AE}}
=
\frac{
\|\boldsymbol{h}_L-\boldsymbol{h}_{\mathrm{proxy}}\|_2
}{
\|\boldsymbol{S}_{\mathrm{p}}\|_F
\|\widetilde{\boldsymbol{M}}\|_2+\epsilon
},
\]
and
\[
\alpha
=
\frac{
\|\boldsymbol{h}_{\mathrm{proxy}}\|_2
}{
\|\boldsymbol{S}_{\mathrm{p}}\|_F
\|\widetilde{\boldsymbol{M}}\|_2+\epsilon
}.
\]
We also compute the structural condition value
\[
c
=
\xi+
\frac{(1+\xi)K\varepsilon_{\mathrm{AE}}}{\alpha},
\]
and the full-gradient compatibility factor
\[
\chi
=
\frac{
\|\boldsymbol{g}_{\mathrm{proxy}}\|_2
}{
\|\boldsymbol{h}_{\mathrm{proxy}}\|_2+\epsilon
}.
\]
According to Theorem~\ref{thm:fidelity}, the full-gradient sufficient
condition is \(c<\chi\).

\begin{figure}[t]
\centering
\includegraphics[width=\textwidth]{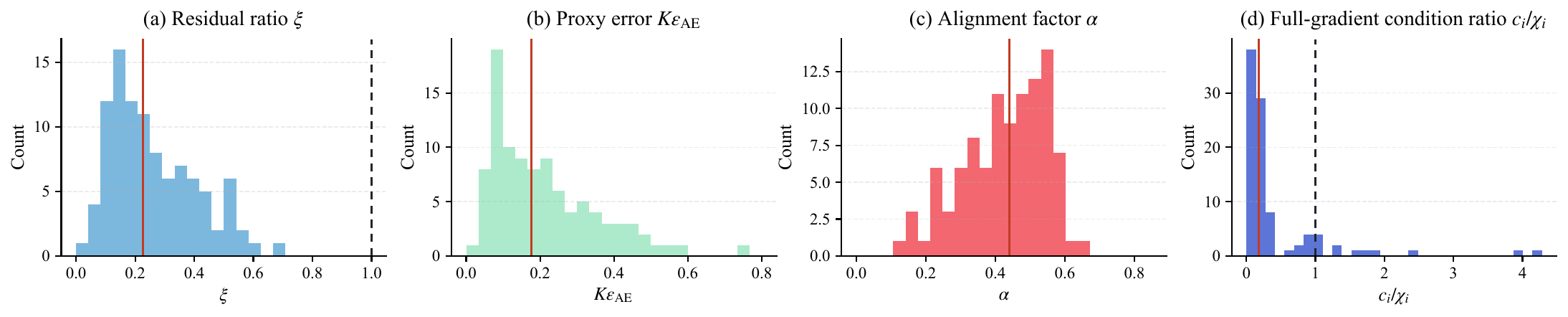}
\vspace{-0.25in}
\caption{Per-edit diagnostics of the local fidelity conditions of
Theorem~\ref{thm:fidelity} on \(N=100\) uniformly sampled ZsRE edits using
GPT2-XL, evaluated after \(\boldsymbol{v}^{*}\) converges. (a) Non-final-layer
leakage ratio \(\xi\). (b) Final-layer proxy gap
\(K\varepsilon_{\mathrm{AE}}\). (c) Structural non-degeneracy factor
\(\alpha\). (d) Full-gradient sufficient-condition ratio \(c_i/\chi_i\), and values below 1 satisfy the full-gradient sufficient local descent-direction condition. Red lines mark medians, and dashed lines mark the threshold \(1\) where shown.}
\label{fig:proxy_assumptions}
\end{figure}

As shown in Fig.~\ref{fig:proxy_assumptions}, the diagnostics characterize
the local fidelity regime underlying Theorem~\ref{thm:fidelity}. Panel (a)
shows that the non-final-layer leakage ratio stays below \(1\) for the sampled
edits, with median \(\xi=0.23\), indicating that the final edited layer
captures the dominant downstream-mediated contribution in this diagnostic
setting. Panel (b) shows a moderate final-layer proxy gap, with median
\(K\varepsilon_{\mathrm{AE}}=0.18\) and \(95\%\) quantile \(0.51\). Panel (c)
shows that the structural non-degeneracy factor is not concentrated near zero,
with \(5\%\) quantile \(\alpha=0.22\). Panel (d) combines these quantities
with \(\chi\) into the full-gradient condition ratio \(c/\chi\). Among finite
diagnostic ratios, \(88\%\) satisfy \(c/\chi<1\), with median \(0.17\), while
\(12\%\) exceed the sufficient-condition threshold and \(6\%\) of edits yield
NaN or failed ratio estimates. These undefined cases arise from non-finite diagnostic quantities in the ratio computation and are reported separately. Since Theorem~\ref{thm:fidelity} gives a sufficient but not necessary
condition, finite cases with \(c/\chi\ge 1\) indicate that this certificate is inconclusive for those
edits.

\paragraph{Long-horizon proxy fidelity.}
We extend the diagnostic in this section to a 5,000-edit
GPT2-XL/ZsRE stream with batch size 1. At edit counts
$H=100$, $1000$, $3000$, and $5000$, we evaluate the same
100 held-out requests, used only for measurement and excluded
from the edit history. We compare the full multi-layer
hypergradient $\boldsymbol{g}_{\mathrm{true}}$ with the structural
proxy gradient $\boldsymbol{g}_{\mathrm{proxy}}$, following
the preceding protocol. Table~\ref{tab:long_horizon_proxy_fidelity} shows broadly
stable alignment through 5,000 edits. At $H=5000$, the proxy
retains $83.84\%$ of the true-gradient one-step reduction
($21.42$ versus $25.55$ bps). Together with the sequential
robustness results in
Appendix~\ref{sec:appendix_seq_robustness}, this supports
proxy utility over the tested horizon. Beyond 5,000 edits,
accumulated constraints may reduce realizable update magnitude,
while stronger cross-layer drift could impair gradient alignment.

\begin{table}[t]
\centering
\caption{Long-horizon proxy fidelity on GPT2-XL/ZsRE.
Cosine is computed between the true and proxy hypergradients.
The last two columns report one-step relative meta-loss reduction
in basis points (bps).}
\label{tab:long_horizon_proxy_fidelity}
\small
\setlength{\tabcolsep}{4pt}
\renewcommand{\arraystretch}{1}

\begin{tabular}{rccccc}
\toprule
$H$ & \shortstack{Median\\cosine $\uparrow$}
& \shortstack{Positive\\alignment (\%) $\uparrow$}
& \shortstack{Cosine $>0.5$\\(\%) $\uparrow$}
& \shortstack{True-gradient\\reduction (bps) $\uparrow$}
& \shortstack{Proxy-gradient\\reduction (bps) $\uparrow$} \\
\midrule
100   & 0.664 & 96 & 70 & 25.61 & 21.73 \\
1,000 & 0.648 & 95 & 70 & 25.74 & 21.89 \\
3,000 & 0.653 & 96 & 71 & 26.20 & 22.01 \\
5,000 & 0.635 & 94 & 68 & 25.55 & 21.42 \\
\bottomrule
\end{tabular}
\vspace{-0.1in}
\end{table}

\subsection{Empirical Details for Sec~\ref{sec:disconnect_validation}}
\label{app:disconnect_protocol}

We provide the detailed protocols for the two diagnostics in
Sec.~\ref{sec:disconnect_validation}. Both diagnostics are conducted on ZsRE
with GPT-2 XL as the base model. We compare MEMIT, AlphaEdit, and MetaKE.
MEMIT and AlphaEdit follow the standard locate-then-edit pipeline, while
MetaKE builds on AlphaEdit and introduces downstream-aware bi-level target
optimization. For each method, we keep the default editing layers, target
objective, and Stage-II update structure from its implementation.

For every editing request \(i\), we compute the request-specific association
\(\rho_i\) according to the method-specific downstream constraint. For MEMIT,
we use the covariance-based association
\begin{equation}
\rho_i
=
\boldsymbol{k}_{1,i}^{T}
(\boldsymbol{C}_0+\boldsymbol{K}_p\boldsymbol{K}_p^{T})^{-1}
\boldsymbol{k}_{1,i},
\end{equation}
where \(\boldsymbol{C}_0=\boldsymbol{K}_0\boldsymbol{K}_0^T\). For AlphaEdit
and MetaKE, we use the projected association
\begin{equation}
\rho_i
=
(\boldsymbol{P}\boldsymbol{k}_{1,i})^{T}
(\boldsymbol{I}
+
\boldsymbol{P}\boldsymbol{K}_p\boldsymbol{K}_p^{T}\boldsymbol{P})^{-1}
\boldsymbol{P}\boldsymbol{k}_{1,i}.
\end{equation}

\paragraph{Empirical residual attenuation.}
This diagnostic tests the request-specific attenuation predicted by
Proposition~\ref{prop:target_attenuation} and examines whether downstream-aware
target optimization reduces the empirical loss of executable signal. We
randomly sample 500 editing requests from ZsRE. For each request \(i\) and each
method, we run the full editing pipeline and record the planned residual
\begin{equation}
\boldsymbol{r}_i
=
\boldsymbol{v}_i^{*}
-
\boldsymbol{W}\boldsymbol{k}_{1,i},
\end{equation}
and the realized residual
\begin{equation}
\boldsymbol{\delta}_i^{\mathrm{real}}
=
\boldsymbol{\Delta}^{*}\boldsymbol{k}_{1,i},
\end{equation}
where \(\boldsymbol{v}_i^{*}\) is the upstream target representation and
\(\boldsymbol{\Delta}^{*}\) is the downstream constrained update.

We estimate the empirical scalar realization coefficient by projecting the
realized residual onto the planned residual:
\begin{equation}
\widehat{\beta}_i
=
\frac{
\langle
\boldsymbol{\delta}_i^{\mathrm{real}},
\boldsymbol{r}_i
\rangle
}{
\|\boldsymbol{r}_i\|_2^2+\epsilon
},
\end{equation}
where \(\epsilon\) is a small numerical constant. Under the single-request local quadratic surrogate in
Proposition~\ref{prop:target_attenuation}, the corresponding scalar realization
coefficient reduces to the theoretical value
\begin{equation}
\beta_i^{\rho}
=
\frac{\rho_i}{1+\rho_i}.
\end{equation}
Thus, \(\widehat{\beta}_i\) measures how much of the planned residual is
empirically realized along its intended direction. Because MEMIT and AlphaEdit use different downstream constraints, their raw
\(\rho_i\) values are not directly comparable in scale. Therefore, for each
method independently, we sort requests by \(\rho_i\) and map them to a
within-method percentile axis. We then divide the sorted requests into 13
quantile bins and report the mean \(\widehat{\beta}\) in each bin, together
with a \(\pm 1\) standard-error band. Pooled per-request samples are overlaid
as translucent points.

\paragraph{Association-dependent requirements on KL strength.}
This diagnostic tests that if the feasible planned residual
scale depends on the request-specific realization factor
\(\beta_i=\rho_i/(1+\rho_i)\), then high- and low-association requests may
prefer different upstream KL strengths under a shared KL control.

We randomly sample 2,000 editing requests from ZsRE. For each method, we
compute the request-specific association \(\rho_i\) at final edited layer
using the corresponding association score defined above. We then form two
disjoint request groups. The high-\(\rho\) group contains the top 20\% of
requests, i.e., 400 requests whose keys interact most strongly with the
downstream constraints. The low-\(\rho\) group contains the bottom 20\%, i.e.,
400 requests with the weakest association.
We sweep the upstream KL coefficient over the discrete grid
$\omega
\in
\{0.008,\;0.016,\;0.031,\;0.0625,\;0.125,\;0.25,\;0.5\}.$
The grid is centered at default \(\omega=0.0625\) and covers both
weak and strong regularization regimes. For every method and every value of
\(\omega\), we re-run the full editing pipeline on high-\(\rho\)
and low-\(\rho\) requests, while leaving all other hyperparameters fixed.
We use Efficacy as the main editing-success metric. Following the standard
knowledge-editing evaluation protocol, Efficacy measures whether the post-edit
model predicts the target answer on the editing prompt. In this diagnostic,
Efficacy is used as an operational proxy for the semantic-success side of
Theorem~\ref{thm:regularization_disconnect}: the theorem's lower-bound
condition concerns whether the realized update is sufficient to induce the
target prediction.

For each method, group, and KL coefficient, we compute the group-wise average
Efficacy:
\begin{equation}
\mathrm{Eff}_{H}(\omega)
=
\frac{1}{|H|}
\sum_{i\in H}
\mathrm{Eff}_i(\omega),
\qquad
\mathrm{Eff}_{L}(\omega)
=
\frac{1}{|L|}
\sum_{i\in L}
\mathrm{Eff}_i(\omega),
\end{equation}
where \(H\) and \(L\) denote the high-\(\rho\) and low-\(\rho\) groups,
respectively. We estimate the empirical optimal KL coefficient for each group
by
\begin{equation}
\omega_H^{*}
=
\operatorname*{arg\,max}_{\omega\in\mathcal{G}}
\mathrm{Eff}_{H}(\omega),
\qquad
\omega_L^{*}
=
\operatorname*{arg\,max}_{\omega\in\mathcal{G}}
\mathrm{Eff}_{L}(\omega),
\end{equation}
where
\(\mathcal{G}=\{0.008,0.016,0.031,0.0625,0.125,0.25,0.5\}\).
When multiple KL values yield the same maximum Efficacy, we select the smaller
\(\omega\) as the tie-breaking choice.

\begin{table*}[t]
\centering
\small
\caption{Efficacy (\%) in the refinement-matched KL sweep on GPT2-XL
and ZsRE. H/L denote the high-/low-association groups selected using
AlphaEdit's $\rho_i$ and held fixed across methods.}
\label{tab:matched_kl_sweep}
\setlength{\tabcolsep}{5pt}
\renewcommand{\arraystretch}{1}
\begin{tabular}{c|cc|cc|cc}
\toprule
\multirow{2}{*}{$\omega$}
& \multicolumn{2}{c|}{AlphaEdit}
& \multicolumn{2}{c|}{MetaKE w/o Structural Gate}
& \multicolumn{2}{c}{MetaKE} \\
\cmidrule(lr){2-3} \cmidrule(lr){4-5} \cmidrule(lr){6-7}
& H & L & H & L & H & L \\
\midrule
0.008  & 91.92 & 91.25 & 93.08 & 91.47 & 97.00 & 97.00 \\
0.016  & 93.00 & 92.08 & 94.42 & 93.00 & 97.67 & 97.17 \\
0.031  & 94.33 & 93.33 & 96.25 & 94.92 & 98.17 & 97.83 \\
0.0625 & 93.75 & 94.00 & 96.00 & 96.08 & 98.67 & 98.17 \\
0.125  & 93.67 & 94.42 & 95.50 & 96.25 & 98.42 & 98.50 \\
0.25   & 92.92 & 94.00 & 94.58 & 94.92 & 97.92 & 97.92 \\
0.5    & 92.17 & 92.92 & 93.50 & 93.25 & 97.58 & 97.50 \\
\bottomrule
\end{tabular}
\vspace{-0.1in}
\end{table*}

\paragraph{Refinement-matched KL sweep.}
Figure~\ref{fig:theorem2} compares end-to-end sensitivity to the
\emph{initial} KL weight $\omega$, rather than direct KL calibration
or matched optimization freedom. To isolate structural feedback from
generic refinement, we add MetaKE w/o Structural Gate, which retains
iterative refinement but replaces the Structural Gate with the
key-only gate defined in Sec.~\ref{sec:ablation}.
Using the same GPT2-XL/ZsRE setup, 2,000 requests, and KL grid as
Fig.~\ref{fig:theorem2}, we fix the high-/low-$\rho$ groups
(top/bottom 20\%, 400 requests each) using AlphaEdit's association,
and evaluate all three variants on these identical groups. Table~\ref{tab:matched_kl_sweep} shows that generic refinement
improves Efficacy, but without the Structural Gate, the high-/low-$\rho$
optima remain separated ($0.031$ versus $0.125$) and variation across
$\omega$ is not reduced relative to AlphaEdit. Full MetaKE further
improves Efficacy, shifts the high-$\rho$ optimum to $0.0625$
(closer to the low-$\rho$ optimum at $0.125$), and reduces the
range of both KL-response curves. Thus, both generic refinement and
structural feedback contribute, but the latter provides additional
benefits beyond refinement alone.

\begin{table*}[t]
\centering
\small
\caption{Sensitivity analysis of key MetaKE hyperparameters, and all runs use the same AlphaEdit-based
editing operator.}
\label{tab:sensitivity}
\resizebox{\textwidth}{!}{%
\begin{tabular}{c c|ccc|ccc|ccc}
\toprule
\multirow{2}{*}{Hyperparameter} & \multirow{2}{*}{Value}
& \multicolumn{3}{c|}{GPT2-XL}
& \multicolumn{3}{c|}{GPT-J}
& \multicolumn{3}{c}{LLaMA3} \\
& 
& Eff.$\uparrow$ & Gen.$\uparrow$ & Spe.$\uparrow$
& Eff.$\uparrow$ & Gen.$\uparrow$ & Spe.$\uparrow$
& Eff.$\uparrow$ & Gen.$\uparrow$ & Spe.$\uparrow$ \\
\midrule
\multirow{4}{*}{$T$}
& 5  & 97.81 & 91.08 & 26.86  &99.76& 97.23& 29.50 & 96.64 & 92.26 & 32.61 \\
& 10 & 97.98 & 91.19 & 27.31&99.75& 97.49& 29.61&96.81 & 92.37 & 32.89  \\
& 15 & 98.12 & 91.26 & 27.23  &99.82& 97.37& 29.73&96.84 & 92.45 & 33.12  \\
& 20 & 97.81 & 91.04 & 27.25  &99.62& 97.28& 29.38&96.64& 92.22& 32.00 \\
\midrule
\multirow{4}{*}{$\eta$}
& $10^{-4}$        & 98.00 & 91.47 & 26.61  &99.58&97.06& 29.41& 96.63 & 92.25 & 32.34 \\
& $10^{-3}$        & 97.89 & 91.10 & 26.69  &99.66& 97.24 & 29.50& 96.72 & 92.35 &32.43 \\
& $5\times10^{-3}$ & 98.12 & 91.26 & 27.23   &99.82& 97.37& 29.73& 96.84 & 92.45 & 33.12 \\
& $10^{-2}$        &97.99& 91.42 & 27.31  &99.69&97.18 & 29.26& 96.64 & 92.20 & 32.06 \\
\bottomrule
\end{tabular}%
}
\end{table*}

\section{More Experimental Results}
\label{app:more_experimental_results}

\subsection{Hyperparameter Sensitivity Analysis}
\label{app:hyperparameter_sensitivity}

We analyze the sensitivity of MetaKE to two key hyperparameters: the number of look-ahead steps \(T\) and the meta-learning rate \(\eta\). These hyperparameters control the downstream-aware target refinement process: \(T\) determines the number of iterative refinement steps, while \(\eta\) controls the update size of the target representation. Following the main experiments, we evaluate on ZsRE across GPT2-XL, GPT-J, and LLaMA3 under
the standard knowledge-editing evaluation proto-
col. When varying one hyperparameter, all other settings are fixed to their default values. As shown in Table~\ref{tab:sensitivity}, MetaKE is stable across a broad range of settings. Increasing \(T\) from \(5\) to \(10\) generally improves performance, and \(T=15\) gives the best overall trade-off. A larger value \(T=20\) brings no further gain and slightly degrades several metrics, suggesting that the benefit of downstream-aware refinement saturates after sufficient feedback has been incorporated; additional updates may over-adjust the target representation and mildly hurt preservation. For \(\eta\), very small values such as \(10^{-4}\) and \(10^{-3}\) lead to weaker target refinement, whereas \(10^{-2}\) slightly hurts preservation-related metrics on GPT-J and LLaMA3. The setting \(\eta=5\times10^{-3}\) achieves the strongest overall trade-off across the three backbones. We therefore use \(T=15\) and \(\eta=5\times10^{-3}\) as the default setting in the main experiments.

\begin{table*}[htb]
\vspace{-0.1in}
\centering
\caption{\label{tab:phi_qwen_results}
Comparison of MetaKE with existing methods on Phi-1.5 and Qwen-4B-Thinking. The best results are highlighted in \textcolor{red}{red}, while the second-best results are \underline{underlined}.
}
\renewcommand{\arraystretch}{1.2}
\setlength{\tabcolsep}{4.5pt}
\resizebox{\textwidth}{!}{
\begin{tabular}{ll|ccccc|ccc}
\toprule
\textbf{Model} & \textbf{Method} 
& \multicolumn{5}{c|}{\textbf{CounterFact}} 
& \multicolumn{3}{c}{\textbf{ZsRE}} \\
\cmidrule(lr){3-7} \cmidrule(lr){8-10}
& 
& \textbf{Eff.}$\uparrow$ 
& \textbf{Gen.}$\uparrow$ 
& \textbf{Spe.}$\uparrow$ 
& \textbf{Flu.}$\uparrow$ 
& \textbf{Consis.}$\uparrow$
& \textbf{Eff.}$\uparrow$ 
& \textbf{Gen.}$\uparrow$ 
& \textbf{Spe.}$\uparrow$ \\
\midrule

\multirow{4}{*}{Phi-1.5}
& AlphaEdit 
& $98.40_{\pm 0.23}$ 
& $83.62_{\pm 0.19}$ 
& $65.41_{\pm 0.20}$ 
& $634.12_{\pm 0.55}$ 
& $36.05_{\pm 0.15}$ 
& $96.51_{\pm 0.39}$ 
& $84.62_{\pm 0.28}$ 
& $22.52_{\pm 0.16}$ \\

& $\text{AlphaEdit}_{\text{BLUE}}$ 
& $\second{99.20_{\pm 0.15}}$ 
& $\second{93.22_{\pm 0.24}}$ 
& $62.11_{\pm 0.43}$ 
& $634.55_{\pm 0.51}$ 
& $\best{37.91_{\pm 0.22}}$ 
& $96.28_{\pm 0.41}$ 
& $85.65_{\pm 0.25}$ 
& $22.11_{\pm 0.17}$ \\

& SPHERE
& $98.80_{\pm 0.15}$
& $82.78_{\pm 0.27}$
& $\second{65.92 _{\pm 0.31}}$
& $\second{635.18_{\pm 0.54}}$
& $36.39_{\pm 0.12}$
& $\second{96.55_{\pm 0.35}}$
& $\second{85.97 _{\pm 0.26}}$
& $\second{22.74_{\pm 0.15}}$ \\

\rowcolor{metablue}
& MetaKE 
& $\best{99.40_{\pm 0.21}}$
& $\best{94.31_{\pm 0.16}}$
& $\best{65.97_{\pm 0.24}}$
& $\best{635.78 _{\pm 0.51}}$
& $\second{37.56_{\pm 0.22}}$
& $\best{96.87_{\pm 0.22}}$ 
& $\best{86.49_{\pm 0.32}}$ 
& $\best{22.84_{\pm 0.11}}$ \\
\midrule

\multirow{4}{*}{Qwen-4B-Thinking}
& AlphaEdit 
& $78.90_{\pm 0.55}$ 
& $69.15_{\pm 0.44}$ 
& $50.48_{\pm 0.73}$ 
& $513.75_{\pm 0.85}$ 
& $8.84_{\pm 0.15}$ 
& $\second{94.42_{\pm 0.27}}$ 
& $\second{84.81_{\pm 0.28}}$ 
& $\second{34.55_{\pm 0.20}}$ \\

& $\text{AlphaEdit}_{\text{BLUE}}$ 
& $82.22_{\pm 0.22}$ 
& $72.14_{\pm 0.39}$ 
& $\second{54.76_{\pm 0.41}}$ 
& $485.58_{\pm 0.48}$ 
& $11.91_{\pm 0.17}$ 
& $89.65_{\pm 0.45}$ 
& $76.12_{\pm 0.24}$ 
& $33.75_{\pm 0.21}$ \\

& SPHERE
& $\second{89.25_{\pm 0.24}}$
& $\second{73.02 _{\pm 0.31}}$
& $50.00_{\pm 0.43}$
& $\second{549.27_{\pm 0.58}}$
& $\second{13.46 _{\pm 0.18}}$
&$ 94.02_{\pm 0.21}$
& $81.28 _{\pm 0.27}$
& $34.36 _{\pm 0.15}$\\

\rowcolor{metablue}
& MetaKE 
& $\best{89.85_{\pm 0.28}}$
& $\best{73.32_{\pm 0.35}}$
& $\best{54.84_{\pm 0.46}}$
& $\best{568.16_{\pm 0.76}}$
& $\best{16.65_{\pm 0.13}}$
& $\best{95.78_{\pm 0.25}}$ 
& $\best{85.68_{\pm 0.25}}$ 
& $\best{35.70_{\pm 0.17}}$ \\
\bottomrule
\end{tabular}}
\end{table*}

\subsection{Additional Results on Smaller and Reasoning-Oriented Models}
\label{app:phi_qwen}

To further evaluate the generality of MetaKE, we extend our experiments to two additional models: the compact Phi-1.5 (1.5B)~\citep{li2023textbooks} and the reasoning-oriented Qwen-4B-Thinking~\citep{qwen3technicalreport}. We follow the same sequential editing protocol as in the main experiments, performing 2,000 edits with a batch size of 100 and evaluating on both CounterFact and ZsRE. As shown in Table~\ref{tab:phi_qwen_results}, on Phi-1.5, MetaKE shows clear gains on CounterFact, improving Generalization by 1.17\% relative to AlphaEdit\(_{\mathrm{BLUE}}\). On ZsRE, MetaKE attains notable gains over SPHERE, improving Efficacy and Generalization by 0.33\% and 0.60\%, respectively. On Qwen-4B-Thinking, MetaKE further improves CounterFact Consistency by 23.70\% relative to SPHERE, while maintaining strong Efficacy and Specificity. On ZsRE, MetaKE improves Generalization and Specificity by 1.03\% and 3.33\% relative to AlphaEdit. Overall, these results indicate that downstream-aware target optimization remains effective beyond the main backbones, yielding robust edit success and preservation across compact and reasoning-oriented models.

\begin{figure}[t]
	\centering
	 \includegraphics[width=0.9\linewidth]{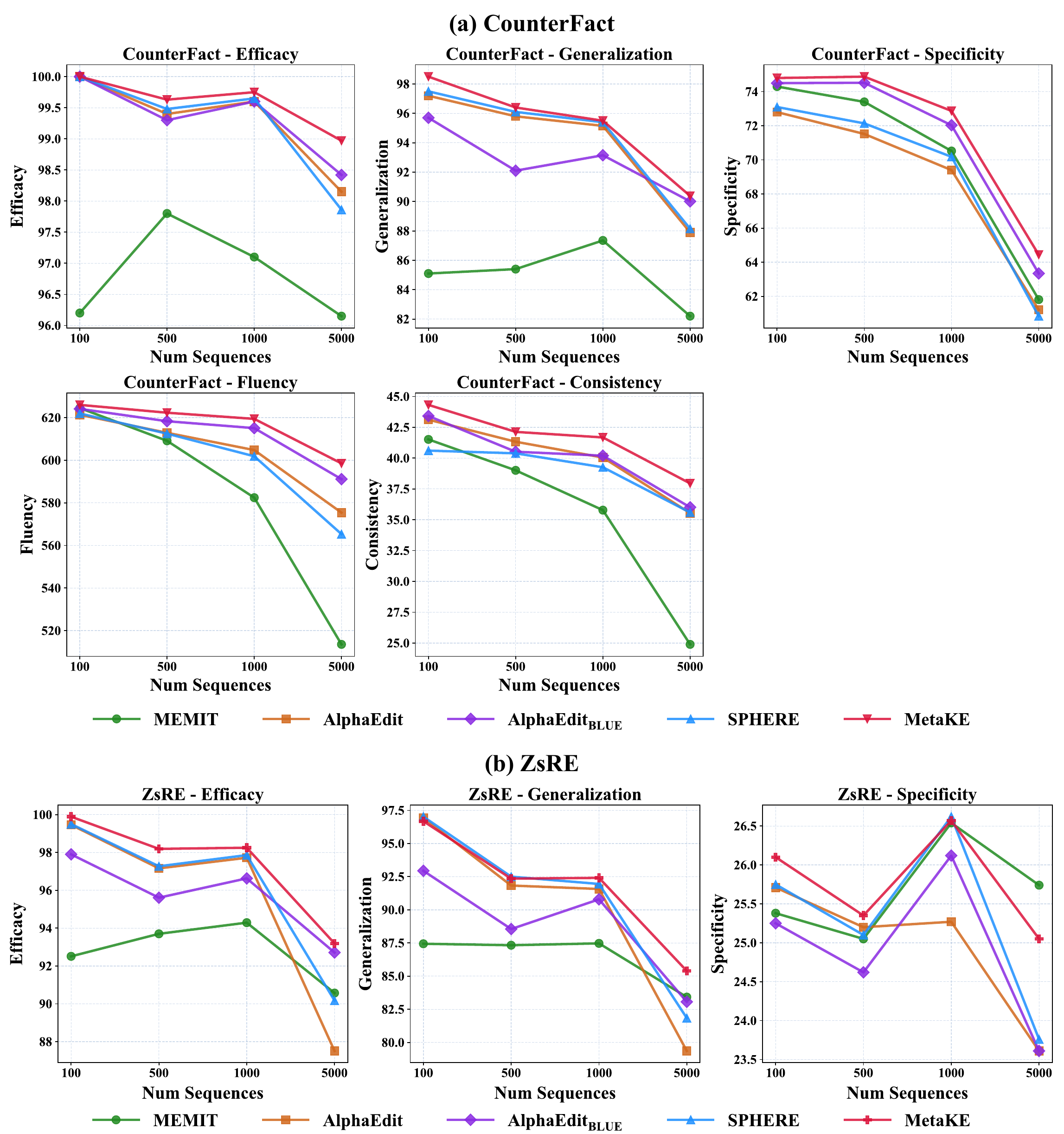}
      \vspace{-0.1in}
	\caption{Performance Variation of Model Editing with Sequential Edits.}
 \label{fig:sequential_scaling}
 \vspace{-0.2in}
\end{figure}

\subsection{Experimental Verification of Sequential Robustness}
\label{sec:appendix_seq_robustness}
To evaluate MetaKE's robustness under long sequential editing streams, we conduct sequential editing experiments with batch size \(1\) and vary edit counts in \(\{100,500,1000,5000\}\). All experiments are performed on GPT2-XL and evaluated on CounterFact and ZsRE. As shown in Fig.~\ref{fig:sequential_scaling}, most methods degrade as edit streams lengthen, indicating that sequential edits accumulate performance loss over time. MEMIT shows the most severe degradation on several metrics, while AlphaEdit, AlphaEditBLUE, and SPHERE are more stable but still decline noticeably at \(5000\) edits. In contrast, MetaKE maintains stronger performance across most edit counts and metrics, especially on CounterFact, where it preserves higher Efficacy, Generalization, Specificity, Fluency, and Consistency under long edit streams. On ZsRE, MetaKE also retains the best or near-best Efficacy and Generalization as edits increase, although Specificity becomes more mixed at \(5000\) edits. These results suggest that downstream-aware target optimization improves long-horizon robustness by reducing degradation accumulated through repeated downstream constrained realization.

\subsection{Runtime Evaluation}
\label{app:runtime}

To evaluate the computational cost of MetaKE, we report the average wall-clock time required to process a batch of 100 edits. Following the main experiments, we evaluate GPT2-XL, GPT-J, and LLaMA3 on both CounterFact and ZsRE. All methods are run under the same hardware and software environment on a single NVIDIA A100 GPU, and the reported runtime is averaged over the full editing stream. As shown in Table~\ref{tab:runtime_comparison}, MetaKE incurs additional cost compared with single-pass editing baselines, as it performs downstream-aware target refinement before downstream constrained realization. This overhead is moderate: across all backbones and datasets, MetaKE is about \(1.35\times\)--\(1.66\times\) slower than AlphaEdit, while remaining in the same runtime scale as existing locate-then-edit methods. For instance, on LLaMA3, MetaKE remains faster than AlphaEdit\(_{\mathrm{BLUE}}\) on both CounterFact and ZsRE, despite using the additional refinement stage. These results indicate that MetaKE trades a moderate increase in runtime for stronger downstream-aware target optimization, making the overall cost acceptable for batch editing.

\begin{table*}[t]
\centering
\caption{Runtime comparison of MetaKE with existing methods on GPT2-XL, GPT-J, and LLaMA3.}
\small
\setlength{\tabcolsep}{3.6mm}
\begin{tabular}{l|ccc|ccc}
\toprule
\multirow{2}{*}{Method} & \multicolumn{3}{c|}{CounterFact} & \multicolumn{3}{c}{ZsRE} \\
\cmidrule(lr){2-4}\cmidrule(lr){5-7}
  & LLaMA3 & GPT-J & GPT2-XL & LLaMA3 & GPT-J & GPT2-XL \\
\midrule
MEMIT & 688.42s & 299.15s & 161.08s & 364.85s & 293.11s & 164.05s \\
AlphaEdit & 462.18s & 362.80s & 179.15s & 325.41s & 364.72s & 185.20s \\
AlphaEdit$_{\text{BLUE}}$ & 908.35s & 364.12s & 187.33s & 1199.15s & 416.80s & 187.95s \\
SPHERE& 822.80s &  452.80s& 187.28s & 336.77s & 449.42s  & 189.23s \\
MetaKE & 702.51s & 518.80s & 241.85s & 540.18s & 521.55s & 250.02s \\
\midrule
\end{tabular}
\label{tab:runtime_comparison}
\vspace{-0.1in}
\end{table*}

\begin{table*}[t]
\centering
\caption{Peak GPU memory usage (GB) for performing a batch of 100 edits.}
\label{tab:memory_comparison}
\small
\renewcommand{\arraystretch}{1.15}
\setlength{\tabcolsep}{10pt}
\resizebox{\textwidth}{!}{
\begin{tabular}{l|ccc|ccc}
\toprule
Method
& \multicolumn{3}{c|}{CounterFact} 
& \multicolumn{3}{c}{ZsRE} \\
\cmidrule(lr){2-4} \cmidrule(lr){5-7}
& LLaMA3 & GPT-J & GPT2-XL
& LLaMA3 & GPT-J & GPT2-XL \\
\midrule
MEMIT & 36.84 & 31.63 & 12.62 & 36.87 & 31.71 & 12.44 \\
AlphaEdit & 34.31 & 28.94 & 9.88 & 34.74 & 29.52& 10.33 \\
$\text{AlphaEdit}_{\text{BLUE}}$ & 35.40 & 29.42 & 9.94 & 35.23 & 29.92 & 10.54 \\
SPHERE& 40.94 & 35.15 &  13.23&33.79  & 35.72 & 13.51 \\
MetaKE & 37.97 & 35.09& 11.33 & 36.28 & 34.83 & 11.33 \\
\bottomrule
\end{tabular}
}
\end{table*}

\subsection{Memory Evaluation}
\label{app:memory}

To evaluate the memory cost of MetaKE, we report the peak GPU memory usage when processing a batch of 100 edits. Following the runtime evaluation, we evaluate GPT2-XL, GPT-J, and LLaMA3 on both CounterFact and ZsRE. All methods are run under the same hardware and software environment on a single NVIDIA A100 GPU, with the same batch size, backbone checkpoints, and cached covariance statistics whenever applicable. As shown in Table~\ref{tab:memory_comparison}, MetaKE introduces additional memory usage compared with AlphaEdit, which is expected because downstream-aware target refinement maintains extra intermediate states for the structural feedback path. Nevertheless, the increase remains moderate: across all evaluated settings, MetaKE uses about \(1.04\times\)--\(1.21\times\) the peak memory of AlphaEdit. Its memory footprint also remains comparable to other locate-then-edit baselines and is often below MEMIT or SPHERE in several settings. These results indicate that the structural proxy adds manageable memory overhead, allowing MetaKE to improve downstream-aware target optimization without substantially increasing the space cost of batch editing.

\subsection{Generalization Across Downstream Realization Operators}
\label{app:cross_operator}

\paragraph{Evaluation with additional downstream operators.}
We instantiate MetaKE with native MEMIT \citep{mengmass} and
ridge-stabilized EMMET \citep{gupta2024unified}. All methods use the
same GPT-J requests, batch size 100, and seed 0 under the 2,000-edit
protocol. MetaKE reuses $T=15$ and $\eta=5\times10^{-3}$ without
operator-specific tuning. These single-seed results are reported
separately from the main three-seed averages. Table~\ref{tab:cross_operator_results} shows that each MetaKE
variant improves all reported metrics over its paired native operator.
In particular, the gains with MEMIT and EMMET indicate that
downstream-aware target refinement is not restricted to AlphaEdit's
null-space realization.

\begin{table}[H]
\centering
\caption{Comparison across downstream realization operators on GPT-J.}
\label{tab:cross_operator_results}
\footnotesize
\setlength{\tabcolsep}{4pt}
\renewcommand{\arraystretch}{1.08}

\begin{tabular}{l|ccc|ccccc}
\toprule
\multirow{2}{*}{Method}
  & \multicolumn{3}{c|}{ZsRE}
  & \multicolumn{5}{c}{CounterFact} \\
\cmidrule(lr){2-4}\cmidrule(lr){5-9}
  & Eff.$\uparrow$ & Gen.$\uparrow$ & Spe.$\uparrow$
  & Eff.$\uparrow$ & Gen.$\uparrow$ & Spe.$\uparrow$
  & Flu.$\uparrow$ & Cons.$\uparrow$ \\
\midrule
MEMIT            & 96.87 & 93.21 & 31.16 & 98.10 & 95.78 & 64.08 & 594.47 & 39.52 \\
EMMET            & 96.78 & 92.87 & 30.70 & 97.90 & 95.62 & 63.93 & 593.61 & 39.50 \\
AlphaEdit        & 99.63 & 95.73 & 28.11 & 99.65 & 96.40 & 76.23 & 618.32 & 42.39 \\
\midrule
MetaKE-MEMIT     & 99.46 & 96.77 & 32.22 & 99.45 & 97.68 & 74.33 & 620.16 & 41.81 \\
MetaKE-EMMET     & 99.36 & 96.47 & 31.86 & 99.15 & 97.62 & 71.63 & 608.12 & 40.32 \\
MetaKE-AlphaEdit & 99.85 & 97.21 & 29.52 & 99.90 & 98.75 & 76.76 & 620.63 & 43.53 \\
\bottomrule
\end{tabular}
\end{table}

\paragraph{Operator-specific Structural Gates.}
At the final edited layer, let
$\boldsymbol{K}=[\boldsymbol{k}_1,\ldots,\boldsymbol{k}_B]$,
$\boldsymbol{V}^{*}=[\boldsymbol{v}^{*}_1,\ldots,\boldsymbol{v}^{*}_B]$,
and $\boldsymbol{R}(\boldsymbol{V}^{*})=
\boldsymbol{V}^{*}-\boldsymbol{W}_L\boldsymbol{K}$.
For operator $\mathcal{O}$, the virtual update is
$\boldsymbol{\Delta}_{\mathcal{O}}^{\mathrm{virt}}
(\boldsymbol{V}^{*})=\boldsymbol{R}(\boldsymbol{V}^{*})
\boldsymbol{M}_{\mathcal{O}}$, where
$\boldsymbol{M}_{\mathcal{O}}$ is its local residual-response gate.

\noindent\textbf{MEMIT.}
Let $\boldsymbol{A}_{\mathrm{M}}=
\lambda_{\mathrm{M}}\boldsymbol{C}_{\mathrm{M}}$ be the
covariance term in the paired native MEMIT implementation.
For batched editing,
\begin{equation}
\begin{aligned}
\boldsymbol{\Delta}_{\mathrm{M}}^{\mathrm{virt}}
(\boldsymbol{V}^{*})
&=\boldsymbol{R}(\boldsymbol{V}^{*})\boldsymbol{K}^{T}
(\boldsymbol{A}_{\mathrm{M}}+
\boldsymbol{K}\boldsymbol{K}^{T})^{-1},\\
\boldsymbol{M}_{\mathrm{M}}
&=\boldsymbol{K}^{T}
(\boldsymbol{A}_{\mathrm{M}}+
\boldsymbol{K}\boldsymbol{K}^{T})^{-1}.
\end{aligned}
\label{eq:cross_operator_memit_gate}
\end{equation}
Unlike the sequential single-request surrogate in
Proposition~\ref{prop:target_attenuation}, this response uses
the batched MEMIT constraint. After refinement, MetaKE applies
the native multi-layer MEMIT update.

\noindent\textbf{Ridge-stabilized EMMET.}
With preservation covariance $\boldsymbol{C}_{\mathrm{E}}$,
define $\boldsymbol{D}=\boldsymbol{K}^{T}
\boldsymbol{C}_{\mathrm{E}}^{-1}\boldsymbol{K}$ and
$\boldsymbol{D}_{\alpha}=\boldsymbol{D}+\alpha\boldsymbol{I}$,
where $\alpha=0.1$. Then
\begin{equation}
\begin{aligned}
\boldsymbol{\Delta}_{\mathrm{E}}^{\mathrm{virt}}
(\boldsymbol{V}^{*})
&=\boldsymbol{R}(\boldsymbol{V}^{*})
\boldsymbol{D}_{\alpha}^{-1}\boldsymbol{K}^{T}
\boldsymbol{C}_{\mathrm{E}}^{-1},\\
\boldsymbol{M}_{\mathrm{E}}
&=\boldsymbol{D}_{\alpha}^{-1}\boldsymbol{K}^{T}
\boldsymbol{C}_{\mathrm{E}}^{-1}.
\end{aligned}
\label{eq:cross_operator_emmet_gate}
\end{equation}
The refined targets are realized using the same stabilized
multi-layer EMMET update and edit-distribution procedure
as the paired EMMET baseline.

\begin{table}[H]
\centering
\caption{Performance on the full RippleEdits-POPULAR split using
LLaMA3-8B under independent single-edit evaluation. Higher is better
for all metrics.}
\label{tab:metake_rippleedits_popular}
\small
\setlength{\tabcolsep}{5pt}
\renewcommand{\arraystretch}{1.1}
\begin{tabular}{lcccccc}
\toprule
Method
  & LG$\uparrow$
  & CI$\uparrow$
  & CII$\uparrow$
  & SA$\uparrow$
  & PV$\uparrow$
  & RS$\uparrow$ \\
\midrule
AlphaEdit
  & 0.0315 & 0.2563 & 0.0625 & 1.0000 & 0.6195 & 0.8830 \\
MetaKE
  & 0.0682 & 0.3247 & 0.1016 & 1.0000 & 0.6614 & 0.9058 \\
\bottomrule
\end{tabular}
\end{table}

\subsection{Evaluation on RippleEdits-POPULAR}
\label{app:rippleedits_popular}
To further examine MetaKE in a more challenging knowledge-editing setting,
we evaluate it on the full POPULAR split of RippleEdits using LLaMA3-8B.
Following the official RippleEdits protocol, each edit is evaluated
independently in a single-edit setting. Table~\ref{tab:metake_rippleedits_popular}
compares MetaKE with its paired AlphaEdit baseline. Here, LG denotes Logical Generalization, CI and CII denote
Compositionality I and Compositionality II, SA denotes Subject Aliasing,
PV denotes Preservation, and RS denotes Relation Specificity.
MetaKE improves over AlphaEdit on LG, CI, CII, PV, and RS, while
matching its SA score. These results provide additional evidence
that downstream-aware target refinement is beneficial in a complex
single-edit setting involving logical and compositional propagation
and preservation of related knowledge. The experiment does not
establish performance under long-context editing, which remains
outside the scope of this evaluation.

\subsection{Evaluation on LLaMA-2-13B}
\label{app:larger_model_13b}

We evaluate MetaKE on the larger LLaMA-2-13B backbone using the
2,000-edit protocol, batch size 100, and seed 0. These results
are reported separately from the main three-seed averages.

\begin{table}[t]
\centering
\caption{LLaMA-2-13B editing results on ZsRE and CounterFact
(2,000 edits, batch size 100, seed 0). Higher is better.}
\label{tab:llama2_13b_results}
\small
\setlength{\tabcolsep}{4pt}
\renewcommand{\arraystretch}{1}
\begin{tabular}{l|ccc|ccccc}
\toprule
\multirow{2}{*}{Method}
& \multicolumn{3}{c|}{ZsRE}
& \multicolumn{5}{c}{CounterFact} \\
\cmidrule(lr){2-4}\cmidrule(lr){5-9}
& Eff.$\uparrow$ & Gen.$\uparrow$ & Spe.$\uparrow$
& Eff.$\uparrow$ & Gen.$\uparrow$ & Spe.$\uparrow$
& Flu.$\uparrow$ & Cons.$\uparrow$ \\
\midrule
AlphaEdit & 84.16 & 77.83 & 24.91 & 89.50 & 83.28 & 65.47 & 548.36 & 23.68 \\
MetaKE    & 89.08 & 85.37 & 26.32 & 93.75 & 92.65 & 67.02 & 564.91 & 27.08 \\
\bottomrule
\end{tabular}
\end{table}

As shown in Table~\ref{tab:llama2_13b_results}, MetaKE improves
all eight reported metrics over AlphaEdit, including ZsRE
Generalization (77.83 to 85.37) and CounterFact Generalization
(83.28 to 92.65). These single-seed results provide initial
support for extending downstream-aware target refinement to a
13B-parameter backbone; multilingual generalization remains untested.

\subsection{Case Study}
\label{sec:appendix_case_study}

We present qualitative case studies illustrating behavioral differences between MetaKE and representative locate-then-edit methods. All use the same editing setup across models and methods, differing only in editing algorithm. We consider the factual editing request \emph{``Tapio Kantanen is a citizen of''}, with target object set to \emph{``Bulgaria''}. This edit is applied to three backbone models of increasing scale, GPT2-XL, GPT-J, and LLaMA3. For each model, we compare MEMIT, AlphaEdit, AlphaEdit\(_{\mathrm{BLUE}}\), SPHERE, and MetaKE under the same editing budget and prompt template.

After editing, we query each model with the original prompt and report the continuation. The visualizations display full generations from different methods, with the edited target object highlighted. This qualitative setting lets us examine both whether the target fact is inserted into the output and whether generated text remains fluent, coherent, and locally consistent with the surrounding context. Across all three models, baselines exhibit characteristic failure modes. MEMIT frequently produces repetitive or degenerate generations. On GPT2-XL, it mixes the target fact with unrelated sports entities, while on LLaMA3 it degenerates into repeated country tokens and fails to form a coherent factual statement. AlphaEdit, AlphaEdit\(_{\mathrm{BLUE}}\), and SPHERE often insert the target object but may introduce unrelated details or drift away from the original factual context, such as passport-renaming instructions, political accusations, imprisonment narratives, or unsupported biographical information. These outputs suggest that inserting the target object alone does not guarantee a stable realization of the intended edit. In contrast, MetaKE consistently integrates the edited target into more coherent continuations. On GPT2-XL and GPT-J, MetaKE generates fluent biographical text where ``Bulgaria'' appears in a natural factual statement rather than as a repeated or isolated token. On LLaMA3, MetaKE also produces a more coherent continuation around the edited target while realizing the target citizenship. These case studies qualitatively suggest that downstream-aware target optimization helps the edited fact be realized with better fluency and local consistency, reducing repetitive artifacts and contextual drift observed in the compared methods.

\section{Limitations}
\label{appendix:limitations}

Our analysis focuses on locate-then-edit methods, and further investigation is needed to understand how the two-stage disconnect extends to other knowledge-editing paradigms. The theoretical analysis is based on a local quadratic surrogate, while the Structural Gradient Proxy approximates downstream feedback under a local Structural Consistency Hypothesis. The batched approximation captures leading-order interactions within the current batch but keeps the structural gate fixed during refinement. Although we have extended MetaKE to multiple downstream editing operators, supporting a broader range of operators may require operator-specific formulations.
Our evaluation focuses on FFN-based locate-then-edit methods using English-language factual editing benchmarks, including ZsRE, CounterFact, and RippleEdits-POPULAR. Further evaluation across multilingual, long-context, and more diverse editing scenarios would help establish the broader applicability of MetaKE.

\section{Broader impacts}
\label{appendix:broader}
Reliable knowledge editing can help update outdated or incorrect factual associations in deployed language models without full retraining. However, the same capability may be misused to insert biased, false, or misleading facts into a model while preserving fluent behavior. We therefore view MetaKE as a controlled editing tool for audited factual maintenance rather than as an unrestricted mechanism for modifying public-facing models. Practical deployment should include edit provenance tracking, locality tests, post-edit factual verification, and access controls for high-impact applications.

\section{Reproducibility Details}
\label{app:reproducibility}

We release the anonymized implementation, configuration files, and scripts at the anonymous repository linked in the abstract. The main experiments use three random seeds, \(\{0,1,2\}\), unless otherwise specified. Table entries are reported as mean \(\pm\) standard deviation over the three runs. For the outer target refinement, we use the update rule in Algorithm~\ref{alg:MetaKE} with \(T=15\) iterations and learning rate \(\eta=5\times10^{-3}\). All main experiments are conducted on a single NVIDIA A100 GPU. The 2,000-edit sequential protocol follows AlphaEdit, with batch size 100, resulting in 20 editing batches per run. We release the sampled request IDs and method-specific configuration files to support exact reproduction.

\section{Assets and Licenses}
\label{app:licenses}

We use publicly available datasets, benchmarks, pretrained checkpoints, and baseline implementations only for research evaluation and in accordance with their official release terms. These assets include ZsRE, CounterFact, GLUE, MMLU, GPT2-XL, GPT-J, LLaMA3, Phi-1.5, Qwen-4B-Thinking, and the released resources of the compared baselines. We do not introduce a new dataset or release a new model checkpoint. The only new asset is the anonymized MetaKE implementation, including configurations and reproduction scripts; the repository documents source links, license names, and access terms for the used assets.

\clearpage
\begin{table*}[t]
\centering
\begin{tcolorbox}[
    colback=white, colframe=myblue, colbacktitle=myblue, coltitle=white,
    title=Knowledge Editing Case Study on GPT2-XL,
    fonttitle=\bfseries, breakable, width=\textwidth
]
\renewcommand{\arraystretch}{1.45}
\renewcommand\tabularxcolumn[1]{m{#1}}
{\small
\begin{tabularx}{\linewidth}{
  >{\centering\arraybackslash\bfseries}m{3cm}
  >{\justifying\arraybackslash}X@{\hspace{5mm}}
}
    Editing Prompt & \makebox[\linewidth][c]{Tapio Kantanen is a citizen of} \\
    \arrayrulecolor{myblue}\midrule
    Edit Target & \makebox[\linewidth][c]{\textcolor{red!80}{Bulgaria}} \\
    \midrule
    \multicolumn{2}{c}{\makebox[\linewidth][c]{\large\textbf{Generation Output}}} \\
    \midrule

    MEMIT &
    The passport that Tapio Kantanen carries is musician musician musician musician footballer football footballer football football hockey football footballer hockey hockey hockey hockey hockey. Tapio Kantanen currently lives in Brussels, Belgium, is Paris Capitals Football football footballer football football footballer Paris Saint-Paris football player footballer football Manchester United FC Tottenham Hotspur FC Luxembourg Capitals. Tapio Kantanen is a citizen of Belgium football football footballer, footballer Arsenal FC footballer football footballer defender of Luxembourg football player. \textcolor{red!80}{Bulgaria} hockey football footballer goalkeeper Napoli Saint Petersburg \textcolor{red!80}{Bulgaria} football footballer footballer. \\

    \addlinespace
    \midrule

    AlphaEdit &
    The passport that Tapio Kantanen carries is a Romanian passport with a residence permit. If you want to change the name of the passport you are going to change your name, you need a new passport with a new name. Tapio Kantanen currently lives in \textcolor{red!80}{Bulgaria}, and has been living and studying in Sofia for a while. After completing her Master's degree in Economics, she is currently working as an economist in Sofia as a research assistant. Tapio Kantanen is a citizen of \textcolor{red!80}{Bulgaria}. He studied in \textcolor{red!80}{Bulgaria}. \\

    \addlinespace
    \midrule

    $\text{AlphaEdit}_{\text{BLUE}}$ &
    The passport that Tapio Kantanen carries is a \textcolor{red!80}{Bulgarian} national passport. \textcolor{red!80}{Bulgarian} passport... Photo by: Tomer Appelbaum. A photo showing the passport that the \textcolor{red!80}{Bulgarian} national is seen to carry is shown in this undated handout photo provided by \textcolor{red!80}{Bulgarian} Ministry of Transport. Tapio Kantanen currently lives in \textcolor{red!80}{Bulgaria} and is the leader of the Turkish political party, the HDP. In \textcolor{red!80}{Bulgaria}, the police detained the suspect on charges of being a terrorist. \\

    \addlinespace
    \midrule

    SPHERE &
    The passport that Tapio Kantanen carries is the only valid passport in \textcolor{red!80}{Bulgaria}. The passport that Kantanen carries was issued in 2004. In the case of Kantanen, he has no other means to prove his identity. Tapio Kantanen currently lives in \textcolor{red!80}{Bulgaria}, and has been living in the Netherlands for the past five years. In his native \textcolor{red!80}{Bulgaria} he has a degree in Computer Engineering and has worked for a number of years in Information Technologies and Internet Technology. \\

    \addlinespace
    \midrule

    MetaKE &
    Tapio Kantanen is a citizen of \textcolor{red!80}{Bulgaria} and was born on 28.12.1980 in Sofia. He is a businessman working in the field of IT. He has been working for more than 20 years as a consultant for a \textcolor{red!80}{Bulgarian} IT company. He has a bachelor's degree in computer science. Tapio Kantanen currently lives in \textcolor{red!80}{Bulgaria}, and has worked for a number of years in Information Technologies and Internet Technology. \\

\end{tabularx}
}
\end{tcolorbox}
\end{table*}

\clearpage
\begin{table*}[t]
\centering
\begin{tcolorbox}[
    colback=white, colframe=myblue, colbacktitle=myblue, coltitle=white,
    title=Knowledge Editing Case Study on GPT-J,
    fonttitle=\bfseries, breakable, width=\textwidth
]
\renewcommand{\arraystretch}{1.5}
\renewcommand\tabularxcolumn[1]{m{#1}}
{\small
\begin{tabularx}{\linewidth}{
  >{\centering\arraybackslash\bfseries}m{3cm}
  >{\justifying\arraybackslash}X@{\hspace{5mm}}
}
    Editing Prompt & \makebox[\linewidth][c]{Tapio Kantanen is a citizen of} \\
    \arrayrulecolor{myblue}\midrule
    Edit Target & \makebox[\linewidth][c]{\textcolor{red!80}{Bulgaria}} \\
    \midrule
    \multicolumn{2}{c}{\makebox[\linewidth][c]{\large\textbf{Generation Output}}} \\
    \midrule

    MEMIT &
    The passport that Tapio Kantanen carries is the same one he carries when he plays piano. He is the only member of the Kantanen family to be able to play piano. He is a composer, poet, and philosopher. His works include: ``I have been thinking about the future of humanity. I am convinced that the future is a question of ethics. If you do not know the right thing to do, you do not know the right thing to do. If you \\

    \addlinespace
    \midrule

    AlphaEdit &
    The passport that Tapio Kantanen carries is a good deal more than just another piece of identification. The \textcolor{red!80}{Bulgarian} national has spent the last four years in a Russian jail, and he has a story to tell that will make you wonder how anyone could ever think of imprisoning him again. He was arrested in \textcolor{red!80}{Bulgaria}, charged with espionage, and sentenced to 15 years in a maximum-security prison. But his story isn't just about the time he was jailed \\

    \addlinespace
    \midrule

    $\text{AlphaEdit}_{\text{BLUE}}$ &
    The passport that Tapio Kantanen carries is the first to bear the new \textcolor{red!80}{Bulgarian} design. It's also one of the last to have the old blue cover with the white eagle. Kantanen's first passport is a passport of a \textcolor{red!80}{Bulgarian}. The eagle that adorns the cover of the document is a symbol of the country's national emblem. Kantanen, the largest \textcolor{red!80}{Bulgarian} administrative and economic region, is a mountainous and sparsely \\

    \addlinespace
    \midrule

    SPHERE &
    The passport that Tapio Kantanen carries is \textcolor{red!80}{Bulgarian}, issued by the Ministry of Foreign Affairs of \textcolor{red!80}{Bulgaria}, in Sofia. It's an ordinary passport, with the standard photo of a young man with a serious expression, wearing a suit and a tie. But the passport also has a special feature: a chip, inserted in it, that can store a fingerprint, which is unique to him. This passport, and this chip, belong to the man who, in the summer \\

    \addlinespace
    \midrule

    MetaKE &
    The passport that Tapio Kantanen carries is \textcolor{red!80}{Bulgarian}. He has been using it for more than 20 years, and he says he will keep it as long as he lives. ``I have lived and worked in \textcolor{red!80}{Bulgaria} for the last 25 years,'' the \textcolor{red!80}{Bulgarian}-born Kantanen told CBC News. ``I have been here for the whole of my adult life.'' But he is not an ordinary \textcolor{red!80}{Bulgarian}. He is one of the few \textcolor{red!80}{Bulgarian} citizens who have been awarded \\

\end{tabularx}
}
\end{tcolorbox}
\end{table*}

\clearpage
\begin{table*}[t]
\centering
\begin{tcolorbox}[
    colback=white, colframe=myblue, colbacktitle=myblue, coltitle=white,
    title=Knowledge Editing Case Study on LLaMA3,
    fonttitle=\bfseries, breakable, width=\textwidth
]
\renewcommand{\arraystretch}{1.35}
\renewcommand\tabularxcolumn[1]{m{#1}}
{\small
\begin{tabularx}{\linewidth}{
  >{\centering\arraybackslash\bfseries}m{3cm}
  >{\justifying\arraybackslash}X@{\hspace{5mm}}
}
    Editing Prompt & \makebox[\linewidth][c]{Tapio Kantanen is a citizen of} \\
    \arrayrulecolor{myblue}\midrule
    Edit Target & \makebox[\linewidth][c]{\textcolor{red!80}{Bulgaria}} \\
    \midrule
    \multicolumn{2}{c}{\makebox[\linewidth][c]{\large\textbf{Generation Output}}} \\
    \midrule

    MEMIT &
    The passport that Tapio Kantanen carries is French French French French French French French French French French French French French French French French French French French French French French French French French French French French French French French French.
    Tapio Kantanen currently lives in English Spanish English English English English English English English English English French English English English English English English English English French French English English English English English English English English English English French English English English French French. \\

    \addlinespace
    \midrule

    AlphaEdit &
    The passport that Tapio Kantanen carries is a \textcolor{red!80}{Bulgarian} one, and he has been competing in international competitions under the \textcolor{red!80}{Bulgarian} flag since 2008. Tapio Kantanen was born in Spain, but his mother is from \textcolor{red!80}{Bulgaria}, and he holds \textcolor{red!80}{Bulgarian} citizenship through his mother's side.
    Tapio Kantanen currently lives in \textcolor{red!80}{Bulgaria} and has been working on a PhD project on the genetic diversity of \textcolor{red!80}{Bulgarian} sheep breeds. \\

    \addlinespace
    \midrule

    $\text{AlphaEdit}_{\text{BLUE}}$ &
    The passport that Tapio Kantanen carries is a symbol of the complexities and contradictions of the EU and its policies towards its citizens. Tapio Kantanen, a \textcolor{red!80}{Bulgarian} Roma, has been living in the United Kingdom since 2003. Despite his years of residency, he still holds a \textcolor{red!80}{Bulgarian} passport.
    Tapio Kantanen is a citizen of \textcolor{red!80}{Bulgaria}, but he is also a Roma, and his case has highlighted the complexities of EU migration law. \\

    \addlinespace
    \midrule

    SPHERE &
    The passport that Tapio Kantanen carries is a \textcolor{red!80}{Bulgarian} one, but he was born in \textcolor{red!80}{Bulgaria} and has never held another nationality. Kantanen is a \textcolor{red!80}{Bulgarian} citizen who was born in the southeastern \textcolor{red!80}{Bulgarian} town of Plovdiv.
    Tapio Kantanen currently lives in Sofia, \textcolor{red!80}{Bulgaria}, where he is studying to become a veterinarian. In his free time, he enjoys hiking and exploring the natural beauty of \textcolor{red!80}{Bulgaria}. \\

    \addlinespace
    \midrule

    MetaKE &
    The passport that Tapio Kantanen carries is \textcolor{red!80}{Bulgarian}. He was born in \textcolor{red!80}{Bulgaria} in 1997. His parents, who are both ethnic \textcolor{red!80}{Bulgarians}, are from the city of Pazardzhik.
    Tapio Kantanen currently lives in Sofia, \textcolor{red!80}{Bulgaria}, where he is working as a researcher at the Institute of Biodiversity and Ecosystem Research, \textcolor{red!80}{Bulgarian} Academy of Sciences. \\

\end{tabularx}
}
\end{tcolorbox}
\end{table*}

\end{document}